%% file: main.tex
\documentclass[11pt]{article}

\input{0_config/init_load_plain}

\input{0_config/math_commands_plain}

\newcommand*\samethanks[1][\value{footnote}]{\footnotemark[#1]}

\title{Capturing Conditional Dependence via Auto-regressive Diffusion Models}

\author[$\dagger$]{\normalsize Xunpeng Huang\thanks{Equal contribution}}
\author[$\S$]{Yujin Han\samethanks}
\author[$\S$]{Difan Zou}
\author[$\P$]{Yian Ma}
\author[$\ddag$]{Tong Zhang}

\affil[$\dagger$]{Hong Kong University of Science and Technology}
\affil[$\S$]{The University of Hong Kong}
\affil[$\P$]{University of California San Diego}
\affil[$\ddag$]{University of Illinois Urbana-Champaign}
\begin{document}

\date{}
\maketitle

\begin{abstract}

\input{0_contents/000abstract}
\end{abstract}

\input{0_contents/010intro}
\input{0_contents/020preliminaries}

\input{0_contents/030MethodARD}

\input{0_contents/Arxiv_040AnalysisARD}

\input{0_contents/Arxiv_050Experiments}

\bibliographystyle{apalike}
\bibliography{0_contents/ref}  %%% Uncomment this line and comment out the ``thebibliography'' section below to use the external .bib file (using bibtex) .





% \end{thebibliography}
\newpage
\appendix
\onecolumn

\input{0_contents/0Xappendix/appendix_main}

\end{document}

%% file: 0_config/init_load_plain.tex
\usepackage[utf8]{inputenc} % allow utf-8 input

\pdfoutput=1

\usepackage{algorithm}
\usepackage{algorithmic}
\usepackage{amsmath}
\usepackage{amssymb}
\usepackage{amsthm,amsfonts}
\usepackage{array}
\usepackage{authblk}
\usepackage{babel}
\usepackage{booktabs}
\usepackage{caption}
\usepackage{colortbl}% http://ctan.org/pkg/xcolor
\usepackage{enumitem}
\usepackage[T1]{fontenc}    
\usepackage{framed}
\usepackage{fullpage}
\usepackage{graphicx}
\usepackage[colorlinks, linkcolor=red, anchorcolor=blue, citecolor=blue]{hyperref}
\usepackage{indentfirst}
\usepackage[utf8]{inputenc}
\usepackage{lipsum}        
\usepackage{makecell}
\usepackage{mathrsfs}
\usepackage{mathtools}
\usepackage{xparse}

\usepackage[framemethod=default]{mdframed}
\usepackage{microtype}     
\usepackage{multicol}
\usepackage{multirow}
\usepackage[numbers]{natbib}
\usepackage{nicefrac}       % compact symbols for 1/2, etc.
\usepackage{tablefootnote}
\usepackage{url}            % simple URL typesetting
\usepackage{wrapfig,lipsum}
\usepackage{xcolor}         % colors
\usepackage{xspace}

\usepackage{doi}
\usepackage{tikz}
\usetikzlibrary{decorations.pathreplacing}
\usepackage{subfigure}
\usepackage[capitalize,noabbrev]{cleveref}
\usepackage{caption}

%% file: 0_config/math_commands_plain.tex
\usepackage{amsmath,amsfonts,bm}

\def\1{\bm{1}}

\def\ry{{\textnormal{y}}}

\def\rvv{{\mathbf{v}}}
\def\rvw{{\mathbf{w}}}
\def\rvx{{\mathbf{x}}}
\def\rvy{{\mathbf{y}}}
\def\rvz{{\mathbf{z}}}

\def\vzero{{\bm{0}}}

\def\vtheta{{\bm{\theta}}}

\def\vg{{\bm{g}}}

\def\vs{{\bm{s}}}

\def\vu{{\bm{u}}}
\def\vv{{\bm{v}}}

\def\vx{{\bm{x}}}
\def\vy{{\bm{y}}}
\def\vz{{\bm{z}}}

\def\mA{{\bm{A}}}
\def\mB{{\bm{B}}}

\def\mI{{\bm{I}}}

\DeclareMathAlphabet{\mathsfit}{\encodingdefault}{\sfdefault}{m}{sl}
\SetMathAlphabet{\mathsfit}{bold}{\encodingdefault}{\sfdefault}{bx}{n}

\def\gN{{\mathcal{N}}}

\newcommand{\grad}{\ensuremath{\nabla}}

\newcommand{\E}{\mathbb{E}}

\newcommand{\R}{\mathbb{R}}

\newcommand{\KL}[2]{\mathrm{KL}\left(#1 \big\| #2\right)}
\newcommand{\FI}[2]{\mathrm{FI}\left(#1 \big\| #2\right)}
\newcommand{\TVD}[2]{\mathrm{TV}\left(#1, #2\right)}
\newcommand{\Var}{\mathrm{Var}}

\newcommand{\der}{\mathrm{d}}

\makeatletter
\def\thickhline{%
  \noalign{\ifnum0=`}\fi\hrule \@height \thickarrayrulewidth \futurelet
   \reserved@a\@xthickhline}
\def\@xthickhline{\ifx\reserved@a\thickhline
               \vskip\doublerulesep
               \vskip-\thickarrayrulewidth
             \fi
      \ifnum0=`{\fi}}
\makeatother

\newlength{\thickarrayrulewidth}
\newtheorem{lemma}{Lemma}[section]
\newtheorem{theorem}[lemma]{Theorem}

\newtheorem{remark}{Remark}

%% file: 0_contents/000abstract.tex
Diffusion models have demonstrated appealing performance in both image and video generation. However, many works discover that they struggle to capture important, high-level relationships that are present in the real world. For example, they fail to learn physical laws from data, and even fail to understand that the objects in the world exist in a stable fashion. This is due to the fact that important conditional dependence structures are not adequately captured in the vanilla diffusion models. 
% In this work, we initiate a study on using auto-regressive (AR) diffusion models to capture conditional dependence in data. 
In this work, we initiate an in-depth study on strengthening the diffusion model to capture the conditional dependence structures in the data.
In particular, we examine the efficacy of the auto-regressive (AR) diffusion models for such purpose and develop the first theoretical results on the sampling error of AR diffusion models under (possibly) the mildest data assumption.
% To learn the AR-diffusion model, we construct a global objective that jointly trains the score estimators to match the conditional scores. 
Our theoretical findings indicate that, compared with typical diffusion models, the AR variant produces samples with a reduced gap in approximating the data conditional distribution. On the other hand, the overall inference time of the AR-diffusion models is only moderately larger than that for the vanilla diffusion models, making them still practical for large scale applications. We also provide empirical results showing that when there is clear conditional dependence structure in the data, the AR diffusion models captures such structure, whereas vanilla DDPM fails to do so. 
On the other hand, when there is no obvious conditional dependence across patches of the data, AR diffusion does not outperform DDPM.
% This finding is consistent with our theoretical analysis. 

%% file: 0_contents/010intro.tex
\section{Introduction}
% \textcolor{red}{Rewrite the intro. Stress the fact that vanilla diffusion models do not understand the world, because they do not capture conditional dependence. Auto-regression tackles the problem.
% Shrink the theoretical contributions into one paragraph.}

Diffusion models transform the data---oftentimes in the form of pixels---\emph{jointly} towards a simple, Gaussian distribution. 
After learning with score matching, they incrementally denoise and transform the normal random variable back to one that follows the data distribution.
They scale well with high dimensional data during inference time and have achieved strong performance in various domains including text-to-image generation \citep{dhariwal2021diffusion, austin2021structured, ramesh2022hierarchical, saharia2022photorealistic}, video generation~\citep{gupta2024photorealistic,luo2023videofusion}, etc.
One salient shortcoming, however, is that they struggle to capture important, high-level relationships that are present in the real world. 
For example, they fail to learn physical laws from the data~\citep{kang2024far}, do not capture the stability of the objects in the world, and in general have difficulty capturing the causal structures that exist in the data.
This is due to the fact that important \emph{conditional dependence} structures are not adequately captured in the vanilla diffusion models such as DDPM \citep{ho2020denoising, song2019generative}. 
% \textcolor{blue}{[Maybe cite Yang Song and Stephan Ermon's paper]}. 
% Especially for data whose features exhibit strong conditional dependencies, the added noise to the data samples dilutes and degrades such dependencies.
% As a result, those conditional dependencies 
% even if the scores are well-trained in expectation, typical diffusion models tend to capture conditional dependencies only weakly, leading to limited performance in fine-grained and physical-laws generation.

% The core idea behind typical diffusion models is to incrementally add noise, thereby gradually transforming the data distribution into a prior that is easier to sample from (e.g., a Gaussian distribution).
% Subsequently, these models parameterize and learn the scores of the noised distributions, enabling the progressive denoising of samples drawn from such priors and ultimately recovering the data distribution~\citep{vincent2011connection,song2019generative,ho2020denoising}.
% Under these conditions, training and inference processes in typical diffusion models heavily depend on the scores associated with the noised samples.
% However, for data whose features exhibit strong conditional dependencies, these noised samples can degrade such dependencies compared with the original ones, making both training and generation more challenging.
% As a result, even if joint scores are well-trained in expectation, typical diffusion models tend to capture conditional dependencies only weakly, leading to limited performance in fine-grained and physical-laws generation.

In this work, we initiate a study and ask: would simple modifications to the model structure help diffusion models capture the conditional dependence and have theoretical guarantees of that?
In particular, we consider the recently proposed auto-regressive (AR) diffusion models~\citep{li2024autoregressive, zhang2024var}.
Although such an approach is oftentimes applied for the purpose of scalability~\citep{li2024autoregressive,meng2024autoregressive,liu2024mardini}, the AR structure is intuitively fitting to reflect the conditional dependencies among features, especially when there is a sequential nature to them.
We therefore hypothesize that AR diffusion can provably capture the conditional dependence in the data and at the same time, enjoy appealing training and inference performances.

To verify the hypothesis, we present the first theoretical analysis of the generation quality and efficiency of AR diffusion models.  
We begin by rigorously formulating the training process of AR diffusion.
In particular, we express the global training objective that is in accordance with the existing work~\cite{li2024autoregressive}.
The training error in the global objective upper bounds the expected score estimation error over the conditional distributions.
We then consider the convergence of the reverse diffusion process in the conditional distribution associated with each patch of the data.
We establish that AR diffusion drives each conditional distribution to converge towards the ground-truth one in terms of the Kullback–Leibler (KL) divergence.
We also establish that vanilla diffusion models, on the contrary, can experience a blow up in the KL divergence between the approximate and the ground-truth conditional distributions, even if the joint distribution converges up to arbitrarily high accuracy.
We then study the efficiency of the training and inference of the AR diffusion.
Our theory matches and guides the experiments in terms of the generation quality and training efficiency.
The experiments corroborate that when there is a clear conditional dependence structure in the data, the AR diffusion models capture such structure, whereas vanilla DDPM fails to do so. 
On the other hand, when there is no obvious conditional dependence across patches of the data, AR diffusion does not outperform DDPM.
We summarize the technical contributions of this paper below.

% We begin by formulating the global objective in alignment with the settings presented in~\citet{li2024autoregressive}.
% Assuming the global objective is minimized, its relationship to the expected score estimation error implies a small upper bound for the latter.
% Then, we propose a stage-wise OU process to parallel the noise-addition procedure in AR diffusions, showing that different stages in this OU process effectively transform various data conditional distributions into a standard Gaussian.
% Leveraging this insight, we establish KL convergence for AR diffusions by deriving an adaptive upper bound that controls the gap in conditional dependence between generated and ground-truth samples.
% In contrast to this controllable bound, we also show that typical diffusion models can experience a blow-up in conditional dependence under certain conditions.
% We then conduct experiments on synthetic data, demonstrating that our theoretical convergence framework reflects the realistic relationship between training error and generation quality.
% Finally, the main contributions are summarized as follows. 
% % \dz{each contribution seems to be a bit long, we may consider shorten them a bit.}
% % \textcolor{blue}{Maybe let's call $K$ the number of segments or something, instead of the number of tokens?}
\begin{itemize}[leftmargin=*]
    \item This paper is the first to provide rigorous theoretical analysis for AR diffusion models.
    To achieve KL convergence of the generated samples to the data distribution, we consider the inference process as the reverse of a stage-wise Ornstein–Uhlenbeck (OU) process.
    We show that AR diffusion requires a gradient complexity of $\tilde{O}(KL^2d\epsilon^{-2})$, which introduces only an additional factor $K$, representing patches of data, relative to typical diffusion models.
    Furthermore, under an $O(\epsilon/\sqrt{K})$ level score estimation error, we argue that the assumptions necessary for AR diffusion models to attain this convergence are nearly as mild as those required by typical diffusions~\citep{benton2024nearly}.
    \item We argue that, compared with typical diffusion models, AR diffusions capture data conditional dependencies more effectively, reflected by a smaller KL divergence between the generated token distribution and the ground-truth ones when the same conditioning tokens are given.
    Under certain settings, we can further establish an adaptive upper bound for this conditional KL divergence in AR diffusions; in contrast, this divergence may even blow up for typical diffusion models.
    
    \item We conduct experiments on two distinct synthetic datasets to illustrate that the score estimation error, as characterized in our theoretical findings, can guide the comparison of inference performance between typical and AR diffusion models.
    Practically, AR diffusion often exhibits a smaller training error, yet its inference performance may not necessarily exceed that of typical diffusion.
    To explain this mismatch, our theory indicates that to achieve the same convergence, the ratio of score estimation errors between AR and typical diffusion should be on the order of $O(\sqrt{K})$, a relationship that our experiments empirically validate.
\end{itemize}

% \dz{the logic now looks better, but do not need to include the entire roadmap, just pick some of them, and put some reference together if they have something similar. This paragraph should be around 12-13 lines.}

%% file: 0_contents/020preliminaries.tex
\section{Preliminaries}
\label{sec:pre}
In this section, we will first explain notations used in subsequent sections and then revisit the framework of \citet{li2024autoregressive} briefly.

\paragraph{Conditional Distribution Decomposition.}
Suppose we divide the vector $\vx\in\R^d$ into $K$ patches $(\vx_1, \vx_2,\ldots, \vx_K)$ following some rule; thus, a general distribution on $\R^d$ can be viewed as a joint distribution $p(\vx_1,\vx_2,\ldots,\vx_K)$ where $\vx_k\in\R^{d_k}$ and $\sum_{k=1}^K d_k = d$. 
Given an index set $S$, it can deduce a vector $\vx_{S} \coloneqq (\vx_{i})_{i\in S}$ from the joint one $\vx$. 
Similarly, the index set can also deduce a joint distribution:
\begin{equation*}
    \begin{aligned}
        &p_S(\vx_S)\coloneqq  \int_{(\vx_i)_{i\not\in S}} p(\vx_1,\vx_2,\ldots, \vx_K)\;\der ((\vx_i)_{i\not\in S}).
    \end{aligned}
\end{equation*}
In what follows, we focus on sets with consecutive indexes. 
Therefore we denote $[l:r] = \{l, l+1,\ldots, r\}$, which has
\begin{equation}
    \label{eq:marginal_def_case}
    \begin{aligned}
        &p_{[l:r]}(\vx_{[l:r]}) \;=\; p_{[l:r]}(\vx_l,\vx_2,\ldots,\vx_r)\\
        &= \int_{\vx_{[1:l-1]},\,\vx_{[r+1, K]}} p(\vx_{[1:K]})\;\der (\vx_{[1:l-1]},\,\vx_{[r+1, K]}).
    \end{aligned}
\end{equation}
If $k\not\in [l:r]$ and $(\vx_l,\ldots,\vx_r)$ is given, the conditional probability on $\rvx_k$ is defined as
\begin{equation}
    \label{eq:conditional_def}
    p_{k|[l:r]}(\vx_k|\vx_l,\ldots,\vx_r) \;=\; \frac{p_{[l:r]\cup \{k\}}(\vx_{[l:r]},\vx_k)}{p_{[l:r]}(\vx_{[l:r]})}.
\end{equation}
In addition, we denote the density function of the Gaussian-type distribution $\mathcal{N}(\vzero,\sigma^2\mI)$ as $\varphi_{\sigma^2}$.

After the general definition, we consider some notations related to the diffusion models.
Specifically, the data density function is denoted as
\begin{equation*}
    p_* \propto \exp(-f_*) \colon \R^{d_1+ d_2+ \ldots + d_K}\rightarrow \R,
\end{equation*}
so the marginal and conditional distributions derived from any index set are $p_{*,S}$ and $p_{*,k|[l:r]}$.

\paragraph{AR Diffusion Models.} Following \citet{li2024autoregressive}, we briefly revisit AR diffusion models, which usually divide the generated data into several patches, e.g., $\hat{\rvx} = [\hat{\rvx}_1,\hat{\rvx}_2,\ldots,\hat{\rvx}_K]$ and then predict the next patch, e.g., $\hat{\rvx}_{k+1}$ in a sequence based on the previous ones e.g., $\hat{\vx}_{1:k}$ or a compressed representation, e.g., $\vz\coloneqq \vg_{\vtheta_{\text{ar}}}(\hat{\vx}_{1:k})$.
The prediction is usually to draw $\hat{\rvx}_{k+1}$ from a specific distribution $q_{0}(\cdot|\vx_{[1:k]})$ inspired by typical diffusion models.
Specifically, considering a OU process initialized by $q_{0}(\cdot|\vx_{[1:k]})$, i.e.,
\begin{equation}
    \small
    \label{sde:ideal_condi_forward}
    \begin{aligned}
        \der \rvy_t = -\rvy_t \der t + \sqrt{2}\der \mB_t,\; \rvy_t\sim q_t(\cdot|\vx_{[1:k]}),\; t\in(0,T],
    \end{aligned}
\end{equation}
drawing sample from $q_{0}(\cdot|\vx_{[1:k]})$ is equivalent to obtain $\rvy_0$ by reversing the OU process SDE.~\ref{sde:ideal_condi_forward} and run
\begin{equation}
    \label{sde:ideal_condi_reverse}
    \small
    \begin{aligned}
        \der \rvy^\gets_t = \left(\rvy^\gets_t + 2\grad\ln q_{T-t}(\rvy^\gets_t|\vx_{[1:k]})\right)\der t + \sqrt{2}\der\mB_t,
    \end{aligned}
\end{equation}
where the $\rvy^\gets_t$ follows the distribution $q^\gets_t = q_{T-t}$.
Previous work usually approximately solves the above SDE with
\begin{equation}
    \label{eq:inner_step_main}
    \small
    \begin{aligned}
        \der \hat{\rvy}_t = \Bigl(\hat{\rvy}_t + \vs_{\vtheta_{\text{dm},k+1}}\bigl(\hat{\rvy}_{t_r}\!\mid T - t_r,\vz \bigr)\Bigr)\,\der t 
                        \;+\; \sqrt{2}\,\der\mB_t.
    \end{aligned}
\end{equation}
% \dz{better to just using $d$ and $a$ rather than $df$ and $ar$ as the subscripts? }
In this formulation, the score estimator $\vs_{\vtheta_{\text{dm},k+1}}(\cdot|T - t_r,\vz)$ depending on $\vz$ is parameterized by $\vtheta_{\text{dm},k+1}$ and used to approximate $\grad\ln q_{T-t_r}(\cdot|\vx_{[1:k]})$. 
Besides, $t_r$ denotes the timestamps belonging to the set $\{t_r\}_{r=0}^R$, which partitions the mixing time $T$ of the forward process $\{\rvy_t\}_{t=0}^T$ into $R$ segments of lengths $\{\eta_r\}_{r=0}^{R-1}$.
Given these definitions, if we set 
\begin{equation*}
    \xi_{\vtheta}(\vy|t,\vz)\coloneqq - (1-e^{-2t})^{-1/2} \cdot \vs_{\vtheta}(\vy|t, \vz),
\end{equation*}
and define, at $T-t$, the loss
\begin{equation*}
    \begin{aligned}
        L_{k+1,t}(\vtheta_{\text{dm},k+1}, \vtheta_{\text{ar}}|\vx_{[1:k]})
        = \E_{\rvy_0\sim q_{0}(\cdot|\vx_{[1:k]}),\,\xi\sim \mathcal{N}(\vzero,\mI)} \\
        \Bigl[\bigl\|\xi \;-\; \xi_{\vtheta_{\text{df+1},k}}\bigl(\rvy_{T-t}\!\mid T-t,\; g_{\vtheta_{\text{ar}}}\bigl(\vx_{[1:k]}\bigr)\bigr)\bigr\|^2\Bigr],
    \end{aligned}
\end{equation*}
then \citet{li2024autoregressive} proposes the following objective for training $\vtheta_{\text{dm},k+1}$ and $\vtheta_{\text{ar}}$:
\begin{equation}
    \label{eq:obj_ard_p}
    \begin{aligned}
        L(\vtheta_{\text{dm},k+1}, \vtheta_{\text{ar}} |\vx_{[1:k]})
        = \E_{t}\Bigl[L_{k+1,t}\bigl(\vtheta_{\text{dm},k+1}, \vtheta_{\text{ar}}\mid \vx_{[1:k]}\bigr)\Bigr].
    \end{aligned}
\end{equation}
In the subsequent discussion, for simplicity, we do not explicitly distinguish different learnable parameters (e.g., $\vtheta_{\text{dm},1},\ldots, \vtheta_{\text{dm},K}, \vtheta_{\text{ar}}$) and use $\vtheta$ to represent all relevant parameters instead.  
Moreover, in handling SDEs within Alg.~\ref{alg:ard_inner} for some fixed $k$ and $\vz = g_{\vtheta_{\text{ar}}}(\vx_{[1:k]})$, we abbreviate the underlying distribution $q_t(\cdot|\vx_{[1:k]})$ by $q_t(\cdot)$.  

\paragraph{General Assumptions.} To study convergence and the gradient complexity required for achieving small total variation (TV) distance or Kullback–Leibler (KL) divergence, we assume $p_*$ satisfies:
\begin{enumerate}[label=\textbf{[A{\arabic*}]}]
    \item \label{a2} The second moment of $p_*$ is bounded, i.e.,
    \begin{equation*}
        \E_{\rvx\sim p_*}[\|\rvx\|^2] 
        \;=\;\int p_*(\vx)\,\|\vx\|^2 \,\der \vx 
        \;\le\; m_0.
    \end{equation*}

    \item \label{a1} The energy function of $p_*$ has a bounded Hessian and bounded gradient, namely
    \begin{equation*}
        \|\grad^2 \ln p_*\|\le L
        \quad\text{and}\quad 
        \|\grad \ln p_*\|\le \sqrt{L}.
    \end{equation*}
\end{enumerate}
Assumption~\ref{a2} is prevalent across most works on sampling and diffusion analysis. 
In Assumption~\ref{a1}, we only impose the Hessian upper bound on the data distribution's energy function, often referred to as a minimal smoothness requirement~\citet{chen2023improved}, which is weaker than the smoothness assumption along the entire SDE trajectory in \citet{chen2022sampling}.  
Here, we provide the upper bound of $\grad\ln p_*$ with a formulation related to $L$ only to simplify the notation.
Although, compared with previous works, an additional gradient norm upper bound is only required in our paper, it does not have any constraint on the isoperimetric property, which means the data distribution is still allowed to be general non-log-concave.

%% file: 0_contents/030MethodARD.tex
\section{Inference and Training of AR Diffusion}
\label{sec:me_ARD}
In this section, we will formalize the algorithm for AR diffusion inference from a theoretical perspective, which is simplified from \citet{li2024autoregressive}.
Then, to implement Eq.~\ref{eq:inner_step_main} in the inference, we formally provide a global objective function to train all parameters, i.e., $\vtheta$s deduced from Eq.~\ref{eq:obj_ard_p}.
In the following analysis, we suppose the global objective is well optimized, which implies the neural score estimator is accurate, and the score estimation error is small.
With all of the above knowledge, we bridge the gap between AR and typical diffusion models by formulating the inference of AR diffusion as the reverse of a stage-wise OU process.

\begin{algorithm}[t]
    \caption{\sc Simplified Autoregressive Diffusion Generation}
    \label{alg:ard_outer}
    \begin{algorithmic}[1]
            \STATE {\bfseries Input:} Number of patches $K$, mixing time $T$ for each patch, number of iterations $R$ per patch, score estimator $\vs_\vtheta$, condition generator $\vg_{\vtheta_{\text{ar}}}$
            \STATE For the mixing time $T$, define two sequences:
                \begin{equation}
                    \label{def:time_eta_seq}    
                    \begin{aligned}
                        & \{t_r\}_{r=0}^R,\quad t_0 = 0,\quad  t_R = T,\quad t_i\le t_j\ \forall\ i\le j,\\
                        & \{\eta_r\}_{r=0}^{R-1},\quad \eta_{r} = t_{r+1}-t_r. 
                    \end{aligned}
                \end{equation}
            \STATE Generate $\hat{\vx}_1$ by calling Alg.~\ref{alg:ard_inner}($0$, $\vzero$, $R$, $\{t_r\}_{r=0}^R$, $\vs_\theta$).
            \FOR{$k=1$ to $K-1$} \label{step:token_loop_starts}
                \STATE Acquire the condition for the next patch:
                \begin{equation}
                    \label{eq:conditioning_gene}
                    \vz\coloneqq g_{\vtheta_{\text{ar}}}(\hat{\vx}_{[1:k]}) \;=\; g_{\vtheta_{\text{ar}}}\bigl(\hat{\vx}_1, \hat{\vx}_2,\ldots, \hat{\vx}_k\bigr).
                \end{equation}
                \STATE \label{step: inner_calls_step}Generate the next patch $\hat{\vx}_{k+1}$ conditioned on $\vz$ by calling Alg.~\ref{alg:ard_inner}($k$, $\vz$, $R$, $\{t_r\}_{r=0}^R$, $\vs_\theta$).
            \ENDFOR \label{step:token_loop_end}
            \STATE {\bfseries return} The concatenated patches $[\hat{\vx}_1, \hat{\vx}_2,\ldots,\hat{\vx}_K]$.
    \end{algorithmic}
\end{algorithm}

\paragraph{Implementation of the inference.} We summarize the inference implementation of AR diffusion in Alg.~\ref{alg:ard_outer} where
Step~\ref{step:token_loop_starts}-\ref{step:token_loop_end} show the most essential characteristics of AR diffusion.
These steps iteratively generates each patch, i.e., $\hat{\vx}_{k+1}$, conditioning on the previous, i.e., $\hat{\vx}_{[1:k]}$ or their compressions, i.e., $\vz= g_{\vtheta_{\text{ar}}}(\hat{\vx}_{[1:k]})$.
The specific process for generating $\hat{\vx}_{k+1}$ is shown in Alg.~\ref{alg:ard_inner} which is similar to typical diffusion inference for recovering $q_{0}(\cdot|\hat{\vx}_{[1:k]})$ from standard Gaussian.
The main difference is Alg.~\ref{alg:ard_inner} introduces a non-uniform division $ \{t_r\}_{r=0}^R$ in Eq.~\ref{def:time_eta_seq} for the reverse process, which is an important trick to remove the score smoothness of the SDE trajectories in \citet{chen2022sampling}.

\begin{algorithm}[t]
    \caption{\sc Patch Inference Process under Given Conditions}
    \label{alg:ard_inner}
    \begin{algorithmic}[1]
            \STATE {\bfseries Input:} Patch index $k$, latent vector $\vz$ (conditions), number of iterations $R$, time sequence $\{t_r\}_{r=0}^R$, score estimator $\vs_\vtheta$
            \STATE Draw an initial sample $\hat{\rvy}_0 \sim \mathcal{N}(\vzero, \mI_{d_{k+1}})$.
            \FOR{$r=R-1$ {\bfseries downto} $0$}
                \STATE \label{step:inner_sampler} Use an exponential integrator to simulate the SDE:
                \begin{equation}
                    \small
                    \label{sde:prac_condi_reverse}
                    \begin{aligned}
                        & \hat{\rvy}_{t_{r+1}} = e^{\eta_r}\hat{\rvy}_{t_r} + (e^{\eta_r}-1)\cdot 2\vs_{\vtheta_{\text{df},k+1}}(\hat{\rvy}_{t_r}\!\mid t_R - t_r,\vz ) \\
                        & + \sqrt{e^{2\eta_r}-1}\cdot \xi \quad \text{where}\quad  \xi\sim\mathcal{N}(\vzero,\mI)
                    \end{aligned}
                \end{equation}
                for $t\in(t_r,\,t_{r+1}]$.
            \ENDFOR
            \STATE {\bfseries return} $\hat{\rvy}_{t_R}$.
    \end{algorithmic}
\end{algorithm}

\textbf{Score estimation and training loss.} To implement Alg.~\ref{alg:ard_inner}, the core step Eq.~\ref{sde:prac_condi_reverse} is based on the well-trained neural score estimator, i.e., $\vs_{\vtheta_{\text{df},k+1}}(\cdot\!\mid t_R - t_r,\vz ) $ for any $k$ and $r$.
To integrate all trainable parameters, we deduce a global objective accounting for distributions over $k$, $\vz$ (or $\vx_{[1:k]}$), $t$, and $\rvy_0$ from Eq.~\ref{eq:obj_ard_p}.
% \dz{not quite clear why the previous paper cannot train all variables? If they do not train these parameters, how can they implement the entire AR diffusion model?}
\begin{itemize}[leftmargin=*,nosep]
    \item A convenient choice for the distribution of $k$ is uniform sampling from $\{1,2,\ldots,K\}$.
    \item To estimate the expectation of $ L(\vtheta_{\text{dm},k+1}, \vtheta_{\text{ar}} |\rvx_{[1:k]})$ with the random variable $\rvz$ (or $\rvx_{[1:k]}$), we can let $\rvx_{[1:k]}\sim p_{*,[1:k]}$. In practice, the underlying distribution can be approximated by 
    \begin{equation}
        \label{aprx:data_margin_training_set}
        p_{*,[1:k]}(\vx_{[1:k]}) \approx \frac{1}{U}\sum_{i=1}^U \delta_{\vu^{(i)}_{[1:k]}}(\vx_{[1:k]}),
    \end{equation}
    where $\vu^{(i)}$ is a ground-truth sample from the training set of size $U$, and $\delta$ denotes a Dirac measure.
    \item To build $L_{k+1,t}\bigl(\vtheta_{\text{dm},k+1}, \vtheta_{\text{ar}}\mid \vx_{[1:k]}\bigr)$ in Eq.~\ref{eq:obj_ard_p}, we need samples $\rvy_{t}$ for each $t\in[0,T]$. Since $\{\rvy_t\}_{t=0}^T$ follows an OU process, it suffices to draw $\rvy_0$ when we set $q_0(\cdot|\vx_{[1:k]}) = p_{*, k+1|[1:k]}(\cdot|\vx_{[1:k]})$ and approximate RHS of the equation with
    \begin{equation*}
        \small
        \begin{aligned}
            p_{*,k+1|[1:k]}(\vx_{k+1}|\vx_{[1:k]}) 
        \approx \frac{\sum_{i=1}^U \delta_{\vu^{(i)}_{[1:k+1]}}(\vx_{[1:k]},\vx_{k+1})}{\sum_{i=1}^U \delta_{\vu^{(i)}_{[1:k]}}(\vx_{[1:k]})}.
        \end{aligned}
    \end{equation*}
    \item We set the distribution of $t$ to be uniform over the set $\{T-t_0, T-t_1, \ldots, T-t_{R-1}\}$ for ease of implementation.
\end{itemize}
Under these settings, let the conditional denoising score-matching loss be written as
\begin{equation}
    \small
    \label{eq:ard_global_loss_each}
    \begin{aligned}
        &L^{\mathrm{DSM}}_{k+1,r}(\vtheta|\vx_{[1:k]})
        \coloneqq  \E_{\rvy_0\sim p_{*,k+1|[1:k]}(\cdot|\vx_{[1:k]}),\,\xi\sim \mathcal{N}(\vzero,\mI)} \\
        &\qquad \Bigl[\bigl\|\xi - \xi_{\vtheta}\bigl(\rvy_{T-t_r}\!\mid T-t_r,\; g_{\vtheta}\bigr(\vx_{[1:k]}\bigr)\bigr)\bigr\|^2\Bigr],
    \end{aligned}
\end{equation}
then the global objective becomes
\begin{equation}
    \small
    \label{eq:ard_global_loss}
    \begin{aligned}
        L^{\mathrm{DSM}}(\vtheta) \coloneqq \frac{1}{KR}\sum_{k=1}^K \sum_{r=0}^{R-1}\E_{\rvx_{[1:K]}\sim p_*}\Bigl[L^{\mathrm{DSM}}_{k, t_r}(\vtheta |\rvx_{[1:k-1]})\Bigr].
    \end{aligned}
\end{equation}
Here, we slightly abuse notation because $\vx_{[1:0]}$ is undefined. In fact, Alg.~\ref{alg:ard_outer} shows that generating $\rvx_1$ is unconditional, so we do not need $p_{[1:0]}$, and $p_{*,1|[1:0]}$ only needs to match $p_{*,1}$.
Actually, Eq.~\ref{eq:ard_global_loss_each} and Eq.~\ref{eq:ard_global_loss} are formulated to implement the training by fitting the noise.
While in analysis, the conditional score-matching loss formulated as follows 
\begin{equation*}
    \small
    \begin{aligned}
        & L^{\mathrm{SM}}_{k+1,r}(\vtheta|\vx_{[1:k]})
        = \E_{\rvy_0\sim p_{*,k+1|[1:k]}(\cdot|\vx_{[1:k]}),\,\xi\sim \mathcal{N}(\vzero,\mI)}\\
        &\qquad \Bigl[\bigl\|\vs_{\vtheta}\bigl(\rvy^\prime\!\mid T-t_r,g_{\vtheta}(\vx_{[1:k]})\bigr) - \grad\ln q_{T-t_r}\bigl(\rvy^\prime\!\mid\vx_{[1:k]}\bigr)\bigr\|^2\Bigr],
    \end{aligned}
\end{equation*}
will be more concerned about where $\rvy^\prime$ satisfies
\begin{equation*}
    \rvy^\prime = e^{-(T-t_r)}\cdot \rvy_0 + \sqrt{1-e^{-2(T-t_r)}}\cdot \xi.
\end{equation*}
Corresponding to Eq.~\ref{eq:ard_global_loss}, we consider a global score-matching loss formulated as
\begin{equation}
    \small
    \label{eq:ard_global_matching}
    \begin{aligned}
        L^{\mathrm{SM}}(\vtheta) &= 
    \frac{1}{KR}\sum_{k=1}^K \sum_{r=0}^{R-1}\E_{\rvx_{[1:k-1]}\sim p_{*,[1:k-1]}} \\
            &\quad\cdot\bigl[ (1-e^{-2(T-t_r)})\cdot {L}^{\mathrm{SM}}_{k,r}(\vtheta|\rvx_{[1:k-1]})\bigr].
    \end{aligned}
\end{equation}
Compared with with Eq.~\ref{eq:ard_global_loss}, we may note the weight of $L^{\mathrm{SM}}_{k,r}(\vtheta|\rvx_{[1:k-1]})$ is not uniformed.
While the additional factor $(1-e^{-2(T-t_r)})$ will be canceled by a different choice of $t$'s distribution.
For example, we can sample from $\{\hat{t}_r\}$ defined as
\begin{equation*}
    \small
    \begin{aligned}
        \{\hat{t}_r\}_{r=0}^R\coloneqq \{T-t_r\}_{r=0}^R\quad\text{with}\quad \mathrm{Pr}(r) = \frac{1-e^{-2\hat{t}_r}}{\sum_{\tau=0}^R (1-e^{-2\hat{t}_\tau})}.
    \end{aligned}
\end{equation*}
Under these settings, AR diffusions have the following lemma whose proof is deferred to Appendix~\ref{sec:app_probsetting}.
\begin{lemma}
    \label{lem:training_loss_equ_main}
    Following from the notations of Section~\ref{sec:me_ARD}, for any $\vtheta\in \mathrm{dom}(L^{\mathrm{DSM}})$, it holds that
    \begin{equation*}
        \begin{aligned}
            \grad_{\vtheta} L^{\mathrm{DSM}}(\vtheta) = \grad_{\vtheta} L^{\mathrm{SM}}(\vtheta).
        \end{aligned}
    \end{equation*}
\end{lemma}

% \dz{May I confirm whether the above lemma is saying that denoising score matching can be related to score matching?}
\begin{remark}
    This lemma explicitly shows that minimizing the global objective and the global score matching in AR diffusion with a gradient-based optimizer is equivalent.
    When
    the objective, i.e., Eq.~\ref{eq:ard_global_loss}, is well optimized, we can expect to have a highly accurate score estimation.
    Then, it is reasonable for us to propose the following assumption.
\end{remark}

% Consequently, if Eq.~\ref{eq:ard_global_loss} is optimized to an $O(\epsilon^2_{\text{score}})$ level, we adopt the following assumption:
\begin{enumerate}[label=\textbf{[A{\arabic*}]}]
    \setcounter{enumi}{2}
    \item \label{a3} The score training error satisfies
    \begin{equation*}
        \begin{aligned}
             \frac{1}{KR}\sum_{k=1}^K\sum_{r=0}^{R-1}  \E_{p_{*,[1:k-1]}} 
            \bigl[L^{\mathrm{SM}}_{k,r}(\vtheta|\rvx_{[1:k-1]})\bigr]\;\le\;\epsilon^2_{\text{score}}.
        \end{aligned}
    \end{equation*}
\end{enumerate}
Early studies often require $L^\infty$-accurate score estimation \citep{de2021diffusion, de2022convergence}, but yield exponentially dependent complexity bounds.
Although \citet{chen2022sampling} weakened the score estimation requirement to $L_2$-accuracy and achieved a polynomial complexity for convergence, they imposed a score smoothness assumption over the entire forward OU process.
Subsequently, \citet{chen2023improved,benton2024nearly} retained the minimal score estimation assumptions and replaced the forward process smoothness with the data distribution smoothness without any gradient complexity degrades.
After that, most theoretical works to accelerate the complexity, e.g.,~\citep{li2024d,li2024provable}, are based on these minimal assumptions.
However, unlike typical diffusion models, the forward and reverse processes of AR diffusions cannot be unified into a single SDE.
What score estimate error is sufficient to provide provable convergence remains unknown.

We answer this question with \ref{a3}, which is analogous or even milder than $L_2$ score estimation error required in non-AR diffusion analyses.
Specifically, each term in the summation over $k$ can be viewed as a $L^2$-accurate score estimation.
Rather than demanding that each component individually satisfy an $O(\epsilon_{\mathrm{score}})$ upper bound, AR diffusion only requires that the average of these components be small.
From this viewpoint, we regard \ref{a3} as one of the mildest requirements for score estimation, especially compared to the $L^2$-accurate score estimation typically assumed in non-AR analyses.

% \dz{add some discussion on the assumption and mention some comparison to other assumptions. Also we can also mention that detailed discussion will be left in further lemmas. Reviewers would like to see some direct discussion right after assumption.}

\paragraph{Forward and reverse processes.}
In typical diffusion models, the inference process is implemented by simulating a reverse OU SDE with an exponential integrator.
By contrast, AR diffusion (Alg.~\ref{alg:ard_outer}) introduces an inference paradigm that generates each patch separately and concatenates them to reconstruct the entire sample.
We now explain that this algorithm corresponds to reversing a stage-wise forward OU process.

Specifically, consider a forward process with $K$ (the number of patches) stages, each stage proceeds as follows:
\begin{enumerate}[leftmargin=*]
    \item For Stage $k$, we consider a random process initialized by the distribution
    \begin{equation*}
        p_{*,[1:K-k+1]} = p_{*,[1:K-k]}\cdot p_{*, K-k+1|[1:K-k]}.
    \end{equation*}
    \item Given the random variable $\rvy_t$ implemented as in Eq.~\ref{sde:ideal_condi_forward}, we have
    \begin{equation*}
        \{[\rvx_{[1:K-k]}, \rvy_t]\} \sim  p_{*,[1:K-k]}\cdot q_t,
    \end{equation*}
    where $q_0 = p_{*, K-k+1|[1:K-k]}$.
    \item Since $q_T \rightarrow \mathcal{N}(\vzero,\mI)$ as $T \rightarrow \infty$, we can approximate the underlying distribution of $\{[\rvx_{[1:K-k]}, \rvy_T]\}$ by
    \begin{equation*}
          p_{*,[1:K-k]}\cdot q_T \approx   p_{*,[1:K-k]}\cdot \varphi_1.
    \end{equation*}
\end{enumerate}
After $K$ recursions, the entire forward process converges to a product of standard Gaussians with different dimensions.

Owing to the stage-wise nature of this forward OU process, one cannot apply Doob's $h$-transform to derive a continuous reverse process. 
Nonetheless, Alg.~\ref{alg:ard_outer} effectively implements a stage-wise reverse and recovers the data distribution, whose underlying mechanism proceeds as follows.
\begin{enumerate}[leftmargin=*]
    \item For Stage $k$, suppose we have sample $\hat{\rvx}_{1:k-1}$ whose underlying distribution satisfies $\hat{p}_{[1:K-1]} \approx p_{*,[1:K-1]}$.
    \item Alg.~\ref{alg:ard_inner} can approximately reverse Eq.~\ref{sde:ideal_condi_forward} and obtain a random variable $\hat{\rvy}$ satisfying 
    \begin{equation*}
        \hat{\rvy}\sim \hat{p}_{*, K|K-1}(\cdot|\hat{\rvx}_{1:k-1})\approx p_{*,k|[1:k-1]}(\cdot|\hat{\vx}_{[1:k-1]})
    \end{equation*}
    \item Concat all generated variables, i.e., $\hat{\rvx}_{[1:k]} = [\hat{\rvx}_{1:k-1}, \hat{\rvy}]$ and use it as the conditioning of the next stage.
\end{enumerate}

Consider that AR diffusion models use \( p_{*,k+1|[1:k]}(\cdot|\vx_{[1:k]}) \) as the initial distribution at each stage of the forward process, we expect \( p_{*,k+1|[1:k]}(\cdot|\vx_{[1:k]}) \) to exhibit the same theoretical properties typically assumed for data distributions in standard diffusion models.
\begin{lemma}
    \label{lem:init_second_moment_bounded_main}
    Suppose Assumption~\ref{a2} holds. Then, for any $K>k^\prime \ge 1$, we have
    \begin{equation*}
        \small
        \begin{aligned}
            &\sum_{k=0}^{k^\prime} \E_{\rvx_{[1:k]}\sim p_{*,[1:k]}}\!\Bigl[\E_{\rvy \sim p_{*,k+1|[1:k]}(\cdot|\rvx_{[1:k]})}\!\bigl[\|\rvy\|^2\bigr]\Bigr]\;\le\; m_0.\\
            &\quad \text{and}\quad \E_{\rvx_{[1:k^\prime]}\sim p_{*,[1:k^\prime]}}\![\|\rvx_{1:k^\prime}\|^2]\;\le\; m_0.
        \end{aligned}
    \end{equation*}
\end{lemma}

\begin{lemma}
    \label{lem:init_smoothness_bound_main}
    Suppose Assumption~\ref{a1} holds. For any $k>1$, any $\vx,\vx^\prime\in \R^{d_k}$, and any $\vy\in \R^{d_1+d_2+\ldots+d_{k-1}}$, we have
    \begin{equation*}
        \Bigl\|\grad\ln \frac{p_{*, k|[1:k-1]}(\vx|\vy)}{p_{*, [k|[1:k-1]]}(\vx^\prime|\vy)} \Bigr\|
        \;\le\; 2L\,\|\vx - \vx^\prime\|.
    \end{equation*}
    Moreover, we have $\|\grad^2 \ln p_{*,[1:1]}(\cdot)\|\le 2L$.
\end{lemma}
\begin{remark}
    Compared with the second-moment bound assumed in typical diffusion analyses (e.g., \citet{chen2022sampling, chen2023improved}), there is no uniform second-moment bound on the initial distributions for all stages in the AR diffusion setting.
    Hence, we require adaptive convergence for different \( p_{*,k+1|[1:k]} \), then removing particle dependence by taking the expectation.
    Moreover, the score smoothness condition (Lemma~\ref{lem:init_smoothness_bound_main}) is satisfied by the initial distributions for all stages, which aligns with the score smoothness requirements on data distributions in \citet{chen2023improved}.
\end{remark}

%% file: 0_contents/Arxiv_040AnalysisARD.tex
\section{Theoretical Guarantees for AR Diffusion}

% In this section, we will first provide the main theorem on the convergence of AR diffusion inference and compare it with that in typical diffusion.
% Then, we explain that the AR diffusion model may have better performance because the gap between the conditional distributions of generated and ground truth data will blow up in typical DDPM when some tokens are drawn from a low-density area. \dz{this is confusing, tokens drawn from a low-density area? this looks like you are doing finite-sample training ..}
% After that, we will provide the proof sketch and core lemmas.

In this section we provide the theoretical analyses of the AR diffusion models.
We first analyze the efficacy of AR diffusion in capturing the conditional dependence structures in the data, and compare against the vanilla diffusion models.
We then analyze in general the inference and training performance of AR diffusion.
We demonstrate that, compared with typical DDPM, its gradient complexity increases by a factor of $K$ (the number of data patches) during the inference time, but it is practical for large-scale applications.

\subsection{AR diffusion captures conditional dependence}
% Besides, the gap between the conditional distributions of generated and ground truth data is given as follows.
We first provide a lemma on the convergence of the AR diffusion-generated conditional distributions towards the ground truth data distribution.
We then provide another lemma to contrast it against the distributions generated by vanilla diffusion models.
\begin{lemma}
    \label{lem:condi_rtk_error_eachk_main}
    For any $k\ge 1$, for any $k$-tuples $\vx_{[1:k]}\in\R^{d_1+d_2+\ldots+d_k}$, we consider the SDE.~\ref{sde:prac_condi_reverse} to simulate the reverse process of SDE.~\ref{sde:ideal_condi_forward}, with a proper design of the time sequence $\{t_r\}_{r=0}^R$,
    we have
    \begin{equation*}
        \begin{aligned}
            & \KL{p_{*,k+1|[1:k]}(\cdot|\vx_{[1:k]})}{\hat{p}_{*,k+1|[1:k]}(\cdot|\vx_{[1:k]})} \lesssim e^{-2T}\cdot \left(2Ld_{k+1} + \E_{p_{*,k+1|[1:k]}(\cdot|\vx_{[1:k]})}\left[\|\rvy\|^2\right] \right)\\
            & \quad + \eta\cdot \sum_{r=0}^{R-1} \tilde{L}_{k+1,r}(\vtheta|\vx_{[1:k]}) +  d_{k+1}L^2R\eta^2 + d_{k+1}T\eta+ \eta \E_{p_{*,k+1|[1:k]}(\cdot|\vx_{[1:k]})}\left[\|\rvy\|^2\right].
        \end{aligned}
    \end{equation*}
\end{lemma}
The above means that given the data patch $\vx_{[1:k]}$ that is being conditioned upon, we can always choose a small enough step size $\eta$ and a large enough convergence time $T$, so that the conditional distribution converges to any desired accuracy $\epsilon$: $\KL{p_{*,k+1|[1:k]}(\cdot|\vx_{[1:k]})}{\hat{p}_{*,k+1|[1:k]}(\cdot|\vx_{[1:k]})} \leq \epsilon$.

% \textcolor{red}{Put Lemma B.2 in the other file here to be the main result.
% Assume that $p(\mathbf{x}_{[1:k]})$ is $\sigma$-subGaussian. 
% Then we should get that $\E_{p_{*,k+1|[1:k]}(\cdot|\vx_{[1:k]})}\left[\|\rvy\|^2\right]$ in that upper bound scales as $\sigma^2$.
% On the other hand, the error in the conditional KL, achieved by the vanilla diffusion model would scales as $e^{\sigma^2} \sigma^d$.}

% \dz{we need a paragraph or subsection to give a complete and precise comparison with the diffusion model, i.e., the case of K=1}.

On the other hand, if one only performs score matching over the joint distributions $q_{T-t}(\vx_{[1:K]})$, the convergence of the diffusion model is in terms of $\KL{p_*(\vx_{[1:K]})}{\hat{p}_*(\vx_{[1:K]})}$.
In the Lemma below, we demonstrate that even for the Gaussian target, and even if the KL divergence between the joint distributions are arbitrarily small, the KL divergence between the conditional distributions can remain arbitrarily large.
\begin{lemma}
\label{lem:lower_bound}
Consider random vectors $\rvy\in\mathbb{R}^{d_{k+1}}$ and $\vx\in\mathbb{R}^{d_1+d_2+\dots+d_k}$.
For any error threshold $\varepsilon\in(0,1/2]$ and for any $M\in\mathbb{R}$, there exists a pair of Gaussian probability densities $\left( p_*(\rvy,\vx), \hat{p}_*(\rvy,\vx) \right)$, such that 
$\KL{p_*(\rvy,\vx)}{\hat{p}_*(\rvy,\vx)} \leq \varepsilon$, while $\KL{p_*(\rvy|\vx)}{\hat{p}_*(\rvy|\vx)} > M^2 \cdot \|\vx_{(1:d_{k+1})}\|^2$.
\end{lemma}

\subsection{Inference performances of AR diffusion}
Building upon Lemma~\ref{lem:condi_rtk_error_eachk_main}, we deliver the theoretical result on the sampling error of AR diffusion model on the joint distributions over all the random variables $\vx_{[1:K]}$.
\begin{theorem}
    \label{thm:main_}
    Suppose Assumption~\ref{a2}-\ref{a3} hold, 
    \begin{equation*}
        \delta\le \Big(0, \ln \sqrt{(4L)^{-2}+1}+(4L)^{-1}\Big]
    \end{equation*}
    if Alg.~\ref{alg:ard_outer} chooses the time sequence $\{\eta_r\}_{r=0}^{R-1}$ as
    \begin{equation*}
        \eta_r = \left\{
            \begin{aligned}
                & \eta && \text{when}\quad 0\le r<M\\
                & \eta/(1+\eta)^{r-M+1} && \text{when}\quad M\le r<N\\
                & \eta && \text{when}\quad N\le r\le R
            \end{aligned}
        \right.
    \end{equation*}
    where 
    \begin{equation*}
        M=\frac{T-1}{\eta},\quad N=M+\frac{2\ln(1/\delta)}{\eta},\quad \text{and}\quad R=N+\frac{\delta}{\eta},
    \end{equation*}
    then, the generated samples
    $[\hat{\rvx}_1, \hat{\rvx}_2, \ldots, \hat{\rvx}_K]$ follows the distribution $\hat{p}_*$, which satisfies
    \begin{equation*}
         \begin{aligned}
             &\KL{p_*}{\hat{p}_*}\lesssim  2e^{-2T}L\cdot \left(m_0+ d\right) + (L^2R\eta^2+T\eta)\cdot d  + \eta m_0 + \eta KR\cdot \epsilon_{\mathrm{score}}^2.
         \end{aligned}
    \end{equation*}
\end{theorem}
\begin{remark}
% \dz{we may mention why to consider such a special time sequences? any intuition?}
    To achieve the KL convergence, e.g., $\KL{p_*}{\hat{p}_*}\le \epsilon^2$ for the generated data, we only require the hyper-parameters to satisfy $T = \tilde{\Theta}(1)$,
    \begin{equation*}
        \eta = \tilde{\Theta}(L^{-2}d^{-1}\epsilon^{-2})\quad \text{and}\quad \epsilon_{\mathrm{score}}=\tilde{O}(K^{-1/2}\epsilon).
    \end{equation*}
    Under these conditions, the total gradient complexity of the inference process will be at an $\tilde{O}(KL^2d\epsilon^{-2})$ level.
    Compared with typical DDPM (corresponding to the special case $K=1$ in our setting), this complexity will have an additional factor of $K$, which means AR diffusion usually requires more inference steps to achieve the same generation quality and matches people's general perception in empirical studies.
    From the training perspective, the average dimension of the estimated score in AR diffusion is only $1/K$ of that in typical diffusion.
\end{remark}

\paragraph{Proof sketch of Theorem~\ref{thm:main_}.}
To obtain Theorem~\ref{thm:main_}, we start by checking the gap between the distribution of generated $\hat{p}_*$ and ground-truth data distribution $p_*$, which is written as
\begin{equation}
    \small
    \label{ineq:kl_gap_accu}
    \begin{aligned}
        \KL{p_*}{\hat{p}_*}\le \KL{p_{*,1}}{\hat{p}_{*,1}} + \sum_{k=1}^{K-1} \E_{\rvx_{[1:k]}\sim p_{*,[1:k]}}\left[\KL{p_{*,k+1|[1:k]}(\cdot|\rvx_{[1:k]})}{\hat{p}_{*,k+1|[1:k]}(\cdot|\rvx_{[1:k]})}\right]
    \end{aligned}
\end{equation}
following from the chain rule of $f$ divergence.
In Alg.~\ref{alg:ard_inner}, for any given $\vx_{[1:k]}$, $\hat{p}_{*,k+1|[1:k]}(\cdot|\vx_{[1:k]})$ and $p_{*,k+1|[1:k]}(\cdot|\vx_{[1:k]})$ are the target and output distributions, which is presented as $q^\gets_T$ and $\hat{q}_T$, respectively. 
Moreover, their KL divergence satisfies
\begin{equation*}
    \begin{aligned}
        &\KL{q^\gets_T}{\hat{q}_T}\le \KL{{q}^\gets_0}{\hat{q}_0} + \underbrace{ \sum_{r=0}^{R-1} \E_{{\rvy}^\gets\sim {q}^\gets_{t_r}}\left[\KL{{q}^\gets_{t_{r+1}|t_r}(\cdot|{\rvy}^\gets)}{\hat{q}_{t_{r+1}|t_r}(\cdot|{\rvy}^\gets)}\right]}_{\text{reverse transition error}}.
    \end{aligned}
\end{equation*}
Considering that both $\{{q}^\gets_t\}_{t=0}^T$ and $\{\hat{q}_t\}_{t=0}^T$ are simulating reverse OU processes, the initialization error, i.e., $\KL{{q}^\gets_0}{\hat{q}_0}$, can be adaptively controlled by
\begin{equation}
    \label{ineq:init_error_main}
    \begin{aligned}
        &\KL{{q}^\gets_0(\cdot|\vx_{[1:k]})}{\hat{q}_0(\cdot|\vx_{[1:k]})}\le e^{-2T}\cdot \left(2Ld_{k+1} + \E_{p_{*,k+1|[1:k]}(\cdot|\vx_{[1:k]})}\left[\|\rvy\|^2\right] \right)
    \end{aligned}
\end{equation}
following the convergence of the standard OU process.
Although the upper bound of initialization error in our settings contains the unbounded variables, i.e., $\vx_{[1:k]}$ rather than a uniform bound in typical DDPM analyses.
For the reverse transition error, the Girsanove Theorem allows us to decompose it to the sum of score estimation error, i.e.,
\begin{equation*}
    \small
    \begin{aligned}
        \sum_{r=0}^{R-1} \eta_r\cdot \E_{\tilde{\rvy}_{t_r}}\left[\left\|\vs_{\vtheta}(\tilde{\rvy}_{t_r}|T - t_r,\vz )-  \grad\ln q_{T-t_r}(\tilde{\rvy}_{t_r})\right\|^2\right]
    \end{aligned}
\end{equation*}
and the discretization error, i.e.,
\begin{equation*}
    \small
    \begin{aligned}
        \sum_{r=0}^{R-1} \E_{\tilde{\rvy}_{[0:T]}}\left[\int_{t_r}^{t_{r+1}} \left\|\grad\ln q_{T-t_r}(\tilde{\rvy}_{t_r})-  \grad\ln q_{T-t}(\tilde{\rvy}_t) \right\|^2 \der t \right].
    \end{aligned}
\end{equation*}
Suppose there is an upper bound $\eta$ for all of the dynamics step size $\eta_r$, then the score estimation error will be coincidentally controlled by a partial sum of the score training error provided in Assumption~\ref{a3}, i.e.,
\begin{equation*}
     \eta\cdot  \sum_{r=0}^{R-1} \tilde{L}_{k+1,r}(\vtheta|\vx_{[1:k]}).
\end{equation*}
For the discretization error, we divide the entire reverse process into two segments, i.e., $[0, T-\delta]$ and $(T-\delta, T]$. 
For the first segment, we check the dynamics of 
\begin{equation*}
    E_{s,t}:=\E_{\tilde{\rvy}_{[0:T]}\sim \tilde{Q}_{[0:T]}}\left[\left\|\grad\ln \tilde{q}_{t}(\tilde{\rvy}_{t})-  \grad\ln \tilde{q}_{s}(\tilde{\rvy}_s) \right\|^2 \right]
\end{equation*}
with It\^o calculus and bound the differential inequality coefficients with a similar approach to~\cite{benton2024nearly}. 
For the second segment, we choose a sufficient small  $\delta$ to satisfy the condition of Lemma \ref{lemma:c.9}, then the Lipschitz property of the score function $\grad \ln q_t$ ($t\in[0,\delta]$) can be guaranteed, thereby allowing the discretization error to be bounded by the difference between particles.
Under these conditions, Lemma~\ref{lem:condi_rtk_error_eachk_main} can summarize an adaptive convergence.
Although the adaptive convergence of Alg.~\ref{alg:ard_inner} contains the unbounded $\vx_{[1:k]}$, taking the expectation on $\vx_{[1:k]}$ will make the conditional second moment and the partial sum of the score training error have uniform upper bounds following from Lemma~\ref{lem:init_second_moment_bounded_main} and Assumption~\ref{a3}, respectively.
Under these conditions, the KL divergence accumulates linearly across the patches following Eq.~\ref{ineq:kl_gap_accu}
Theorem~\ref{thm:main_} is thus proved.

%% file: 0_contents/Arxiv_050Experiments.tex
\section{Experiment}
\label{sec:exp}
In this section, we demonstrate the advantages of AR Diffusion over DDPM and the tightness of the sampling performance bound of AR Diffusion provided in Theorem~\ref{thm:main_} in two synthetic experiments, where the image data contains feature dependencies.
\begin{figure}[]
\centering
    \hfill
    \subfigure[Task 1]{\label{fig:task1}\includegraphics[width=0.2\textwidth]{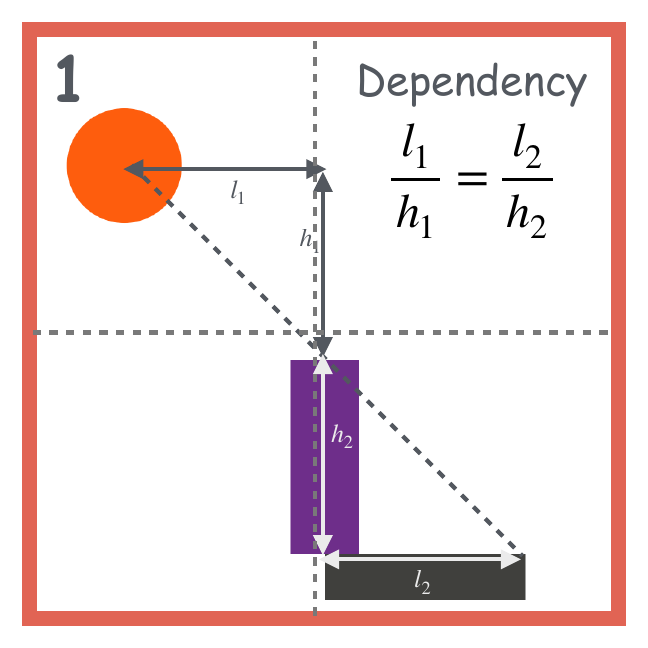}}
    \hfill
    \subfigure[Task 2]{\label{fig:task2}\includegraphics[width=0.2\textwidth]{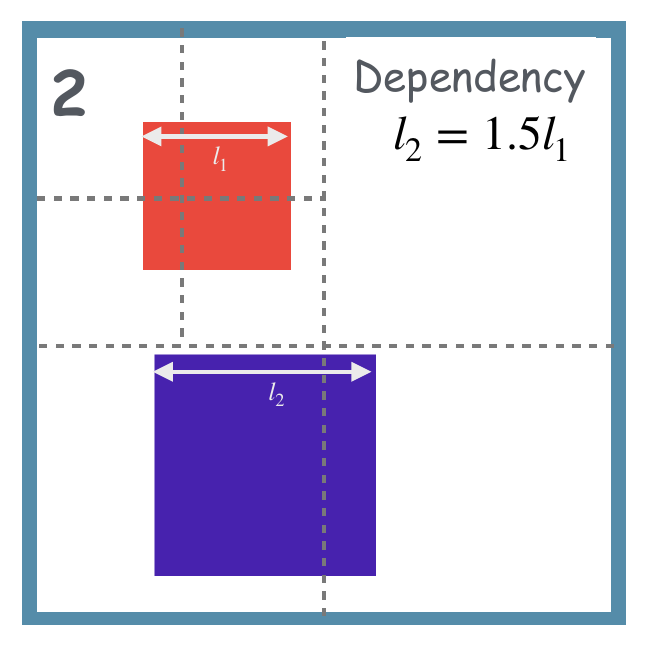}}
    \hfill
\vspace{-0.1in}
\caption{Visualization of Task 1 and Task 2. In Task 1, a patch size of 16 ensures that correlated features (sun and shadow) are located in different patches. In contrast, Task 2 uses a patch size of 8, which disrupts the dependency between the square's lengths by segmenting them into different patches.}
\vspace{-0.15in}
\label{fig:task_vis}
\end{figure}

\paragraph{Synthetic Tasks.}Specifically, following the setting \cite{han2025can}, we consider the following synthetic tasks in \cref{fig:task_vis}: (1) Task 1: Contains three elements representing the sun, the flagpole, and the shadow, where geometric features such as the height of the sun and the length of the shadow are correlated and satisfy the constraint \( \frac{l_1}{h_1} = \frac{l_2}{h_2} \); (2) Task 2: Contains two squares, where the side length \( l_1 \) of the square in the upper part and the side length \( l_2 \) of the square in the lower part satisfy the constraint \( l_2 = 1.5l_1 \). Task 1 is an abstraction of common light and shadow phenomena in the real world. In this setting, the length of an object's shadow is determined by multiple factors, such as the height of the sun and the height of the object. Task 2, on the other hand, attempts to minic certain physical rules, such as perspective effects, where objects appear larger when they are closer to the lower observation point.

\begin{figure*}[]
\centering
    \hfill
    \subfigure[Training Data]{\label{fig:task1_tr}\includegraphics[width=0.24\textwidth]{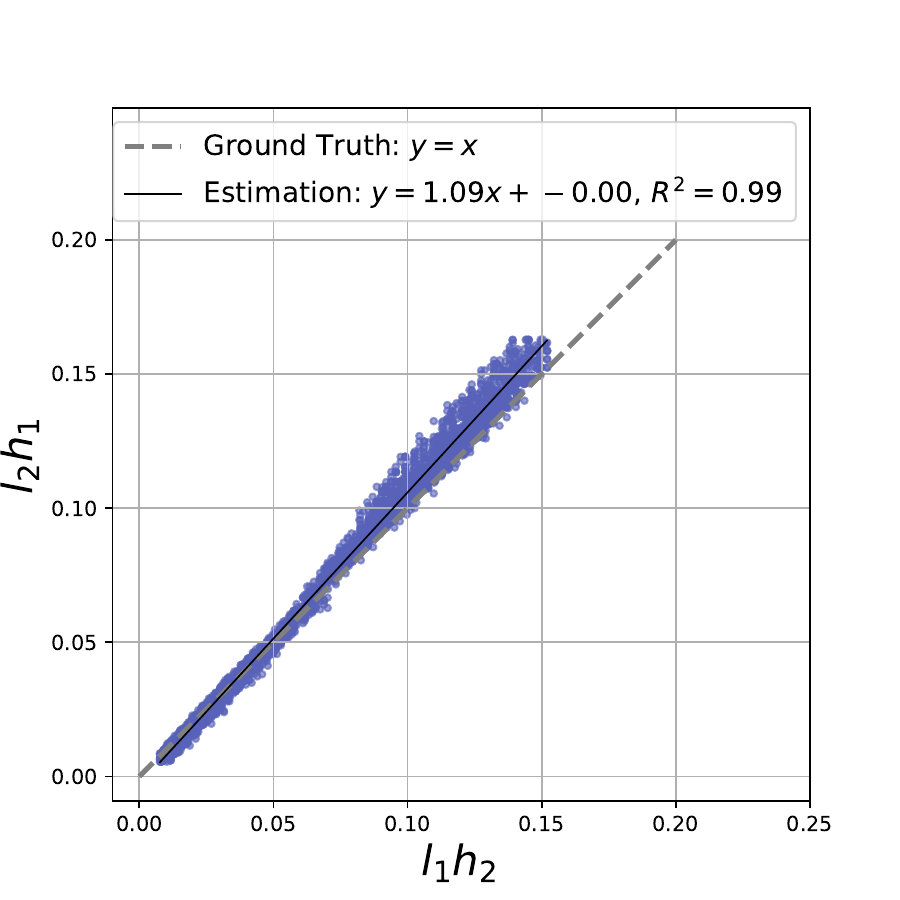}}
    \hfill
    \subfigure[Inference: AR Diffusion]{\label{fig:task1_ar}\includegraphics[width=0.24\textwidth]{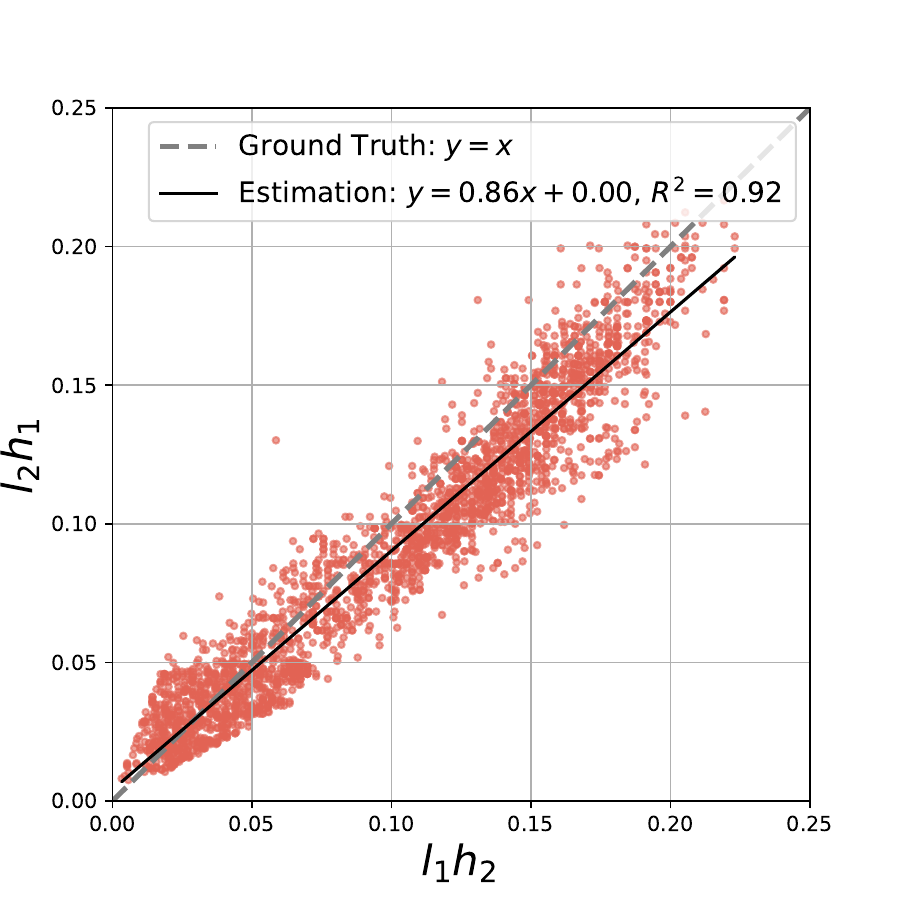}}
    \hfill
    \subfigure[Inference: DDPM]{\label{fig:task1_ddpm}\includegraphics[width=0.24\textwidth]{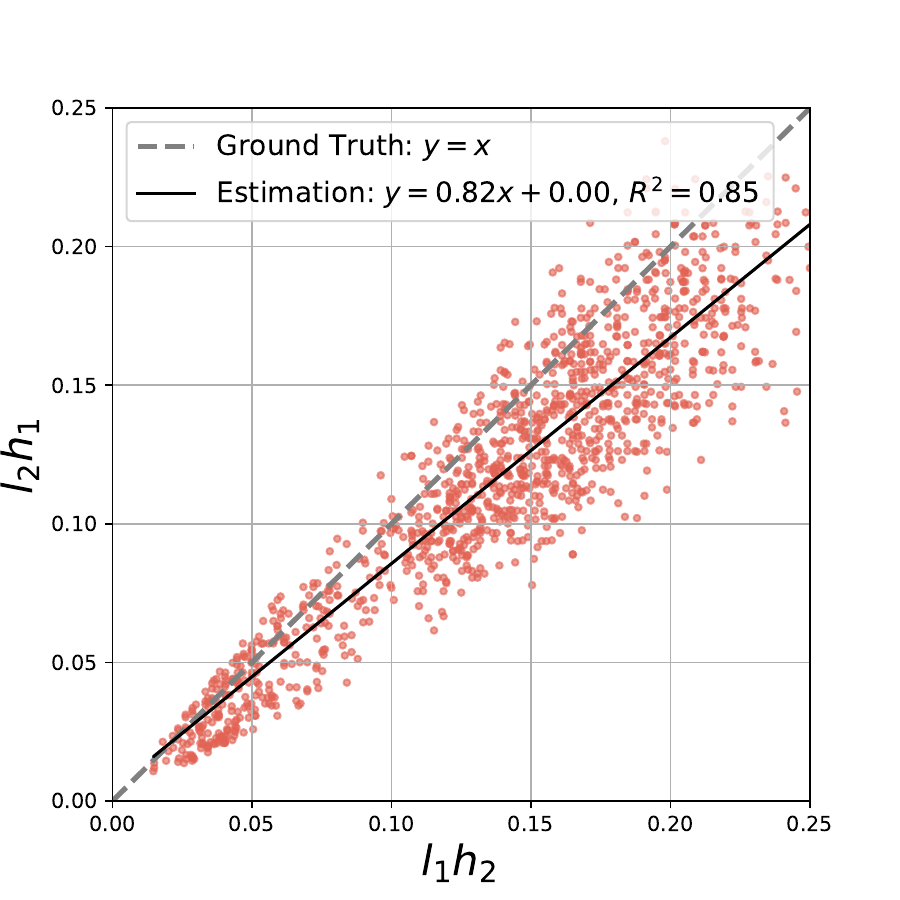}}
    \hfill
    \subfigure[Training: Diffusion Loss]{\label{fig:task1_loss}\includegraphics[width=0.22\textwidth]{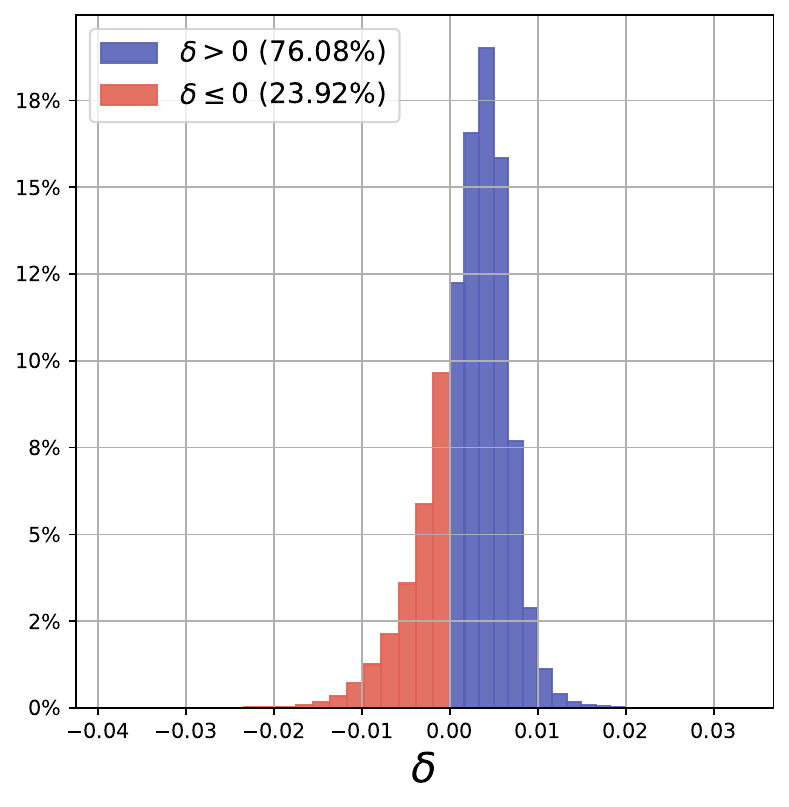}}
    \hfill
\vspace{-0.15in}
\caption{\textbf{Comparison of AR Diffusion and DDPM Performance on Task 1.} \Cref{fig:task1_tr} demonstrates the validity of the evaluation method using the training data as a baseline.  \Cref{fig:task1_ar} and \Cref{fig:task1_ddpm} illustrate the performance of AR Diffusion and DDPM during the inference phase, showing that AR Diffusion better captures inter-feature dependencies with a higher \( R^2 \). \Cref{fig:task1_loss} presents the difference in training loss between DDPM and AR Diffusion, denoted as \( \delta \). For most training steps, AR Diffusion's training loss is lower than that of DDPM, with \( \delta > 0 \).}
\vspace{-0.15in}
\label{fig:task_1_dis}
\end{figure*}

\paragraph{Setup and Evaluation.} 
Based on the designed Task 1 and Task 2, we train AR Diffusion and DDPM from scratch while ensuring that these two models have comparable parameters. Moreover, due to the synthetic data containing only sufficiently simple elements, we can directly extract relevant geometric features from the generated images, as the evaluation method proposed in previous work \cite{han2025can}. After extracting these features, we can further examine whether the generated images satisfy the predefined rules to evaluate the model's performance during inference.

Specifically, for AR Diffusion, we carefully control the patch size to facilitate or hinder the model’s ability to learn feature relationships in a raster scan order. In Task 1, given a $32\times32$ input, we set the patch size to 16, ensuring the model learns sun-related features before shadow features, aligning with underlying dependencies. In contrast, for Task 2, we use a patch size of 8, causing different parts of the same square to appear in separate tokens, disrupting feature learning. By adjusting the patch size, we make Task 1 easier and Task 2 more challenging for AR Diffusion, allowing a direct comparison of its training and inference performance against DDPM under varying difficulty levels. Further details on synthetic data construction and model training are provided in Appendix~\ref{sec:app_more_training_details}.

\begin{figure*}[]
\centering
    \hfill
    \subfigure[Training Data]{\label{fig:task2_tr}\includegraphics[width=0.24\textwidth]{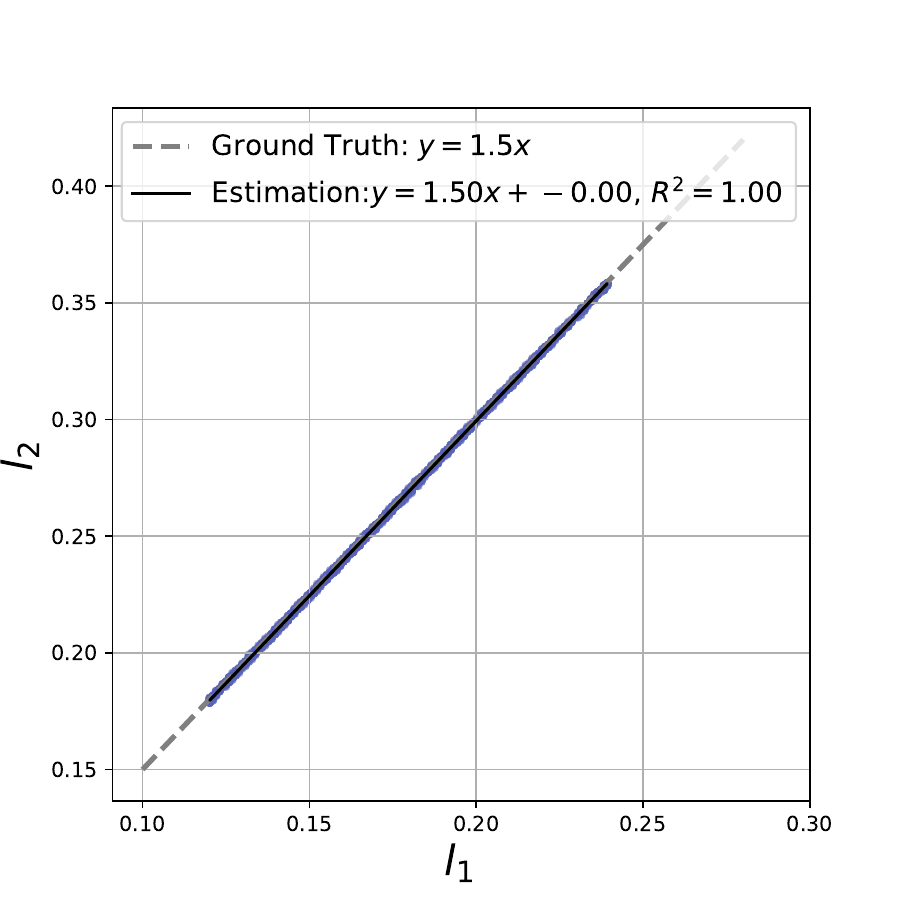}}
    \hfill
    \subfigure[Inference: AR Diffusion]{\label{fig:task2_ar}\includegraphics[width=0.24\textwidth]{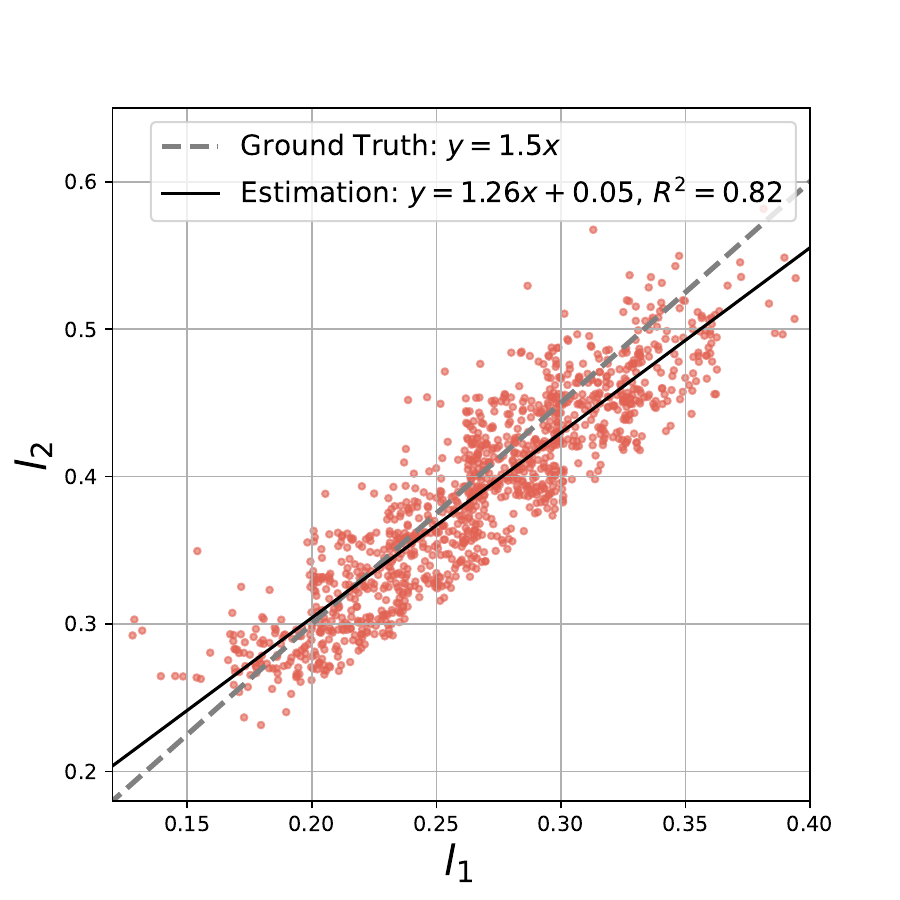}}
    \hfill
    \subfigure[Inference: DDPM]{\label{fig:task2_ddpm}\includegraphics[width=0.24\textwidth]{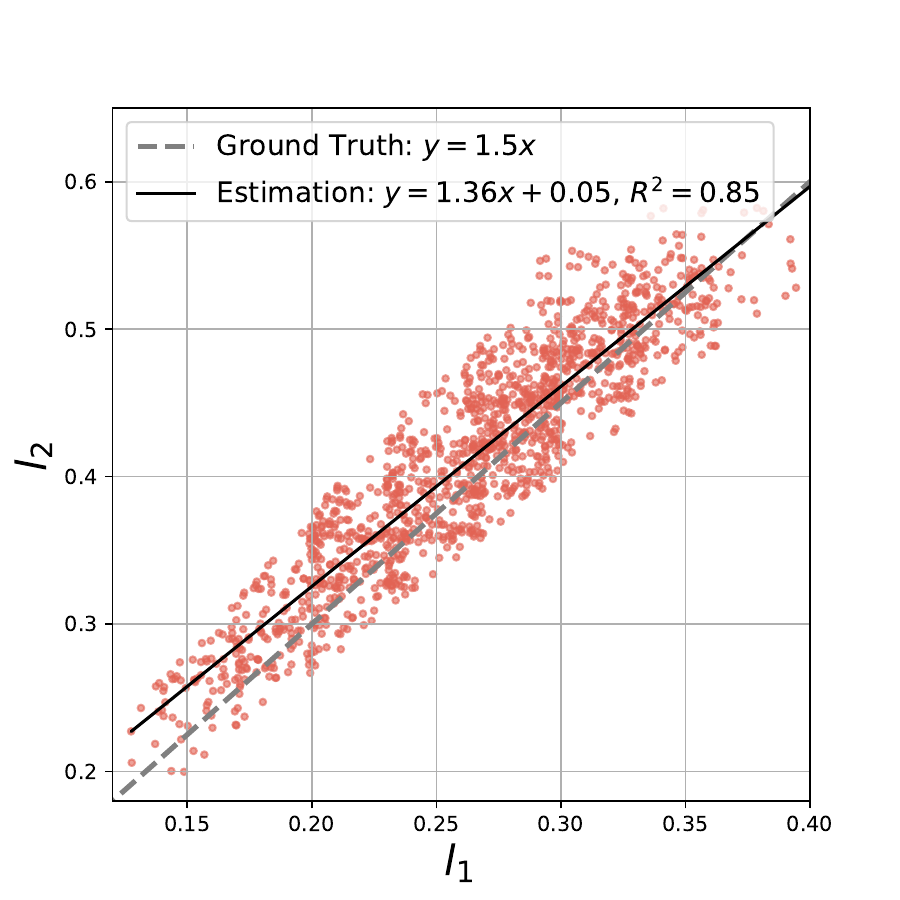}}
    \hfill
    \subfigure[Training: Diffusion Loss]{\label{fig:task2_loss}\includegraphics[width=0.22\textwidth]{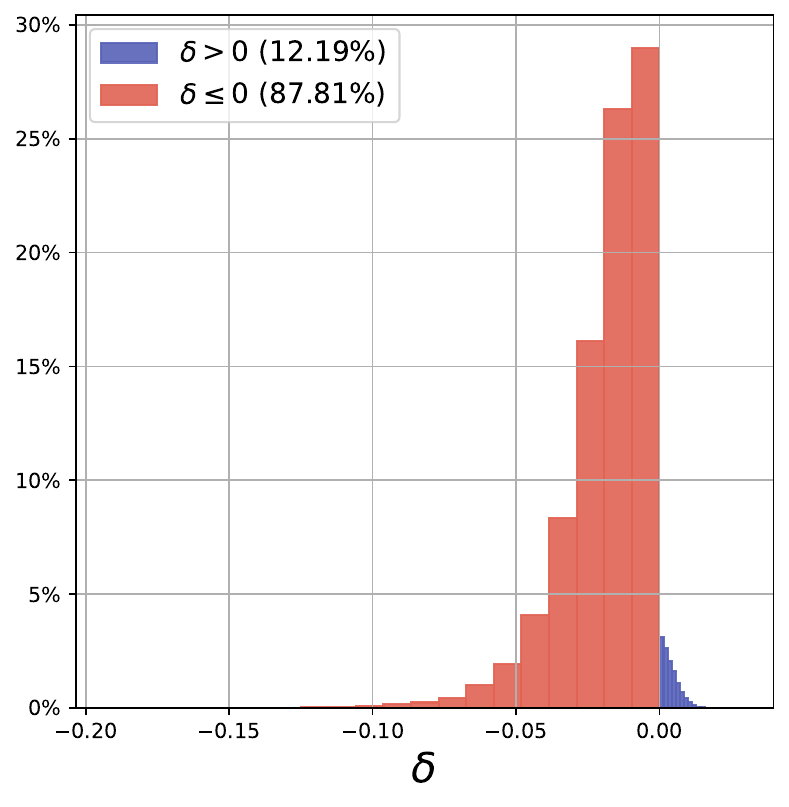}}
    \hfill
\vspace{-0.15in}
\caption{\textbf{Comparison of AR Diffusion and DDPM Performance on Task 2.} \Cref{fig:task2_tr} demonstrates the validity of the evaluation method using the training data as a baseline.  
\Cref{fig:task2_ar} and \Cref{fig:task2_ddpm} illustrate the performance of AR Diffusion and DDPM during the inference phase, showing that AR Diffusion captures inter-feature dependencies less effectively, with a lower \( R^2 \).  
\Cref{fig:task2_loss} presents the difference in training loss between DDPM and AR Diffusion, denoted as \( \delta \). For most training steps, AR Diffusion's training loss is higher than that of DDPM, with \( \delta \leq 0 \).
.}
\vspace{-0.15in}
\label{fig:task_2_dis}
\end{figure*}

\paragraph{Results.}\Cref{fig:task_1_dis} and \cref{fig:task_2_dis} illustrate the performance of AR Diffusion and DDPM on the synthetic tasks, Task 1 and Task 2. Specifically, we evaluate both models in terms of their performance during inference and training.

\textit{Inference Phase.} Following \citet{han2025can}, we first extract geometric features from the training dataset and observe that for both Task 1 and Task 2 in \cref{fig:task1_tr} and \cref{fig:task2_tr}, the extracted geometric features (almost) perfectly satisfy the predefined feature dependencies, with an \( R^2 \) value (almost) equal to 1. The (almost) perfect estimation line further validates the effectiveness of the feature extraction method. Next, \cref{fig:task1_ar}, \cref{fig:task1_ddpm}, \cref{fig:task2_ar}, and \cref{fig:task2_ddpm} illustrate the generation quality of AR Diffusion and DDPM in both tasks, specifically whether the generated samples satisfy the predefined feature dependencies. We observe that for Task 1, AR Diffusion's generations better satisfy the predefined dependencies compared to DDPM, as indicated by larger \( R^2 \), whereas the opposite holds for Task 2. This demonstrates that when feature dependencies exist, given an appropriate patch partitioning, AR Diffusion’s learning paradigm can better capture these dependencies than DDPM.

Moreover, to further confirm that the key factor behind AR Diffusion's superior performance over DDPM in Task 1 stems from the advantage of AR Diffusion's next-token prediction paradigm in emphasizing inter-feature dependencies, \Cref{sec:Ablation: Parallel order} provides additional ablation experiments. Specifically, while keeping other factors unchanged, we modify only the learning order of AR Diffusion. Instead of using a raster scan order that aligns with the feature relationships—learning the sun region first and then the shadow region—we change it to a parallel order, where the sun and shadow regions are learned simultaneously, preventing AR Diffusion from correctly capturing inter-feature dependencies. We observe that the parallel order severely degrades AR Diffusion's generation quality, with \( R^2 \) dropping to $0.68$.

\textit{Training Phase.} \Cref{fig:task1_loss} and \Cref{fig:task2_loss} illustrate the loss differences between AR Diffusion and DDPM during training, where \( \delta := \mathcal{L}_{DDPM} - \mathcal{L}_{AR} \). Notably, the AR Diffusion loss incorporates the correction factor \( K \) derived in Theorem~\ref{thm:main_}. A sufficiently accurate correction factor can precisely characterize the training loss of all tokens in AR Diffusion, and a reasonable training loss should provide a sufficiently tight upper bound in Theorem~\ref{thm:main_} to faithfully reflect the model's generation quality during inference.  

We observe that for Task 1, where AR Diffusion outperforms DDPM in the inference phase, AR Diffusion loss is also lower than DDPM loss in approximately $76\%$ of training steps. Conversely, for Task 2, where AR Diffusion underperforms compared to DDPM in inference, its training loss is only lower than DDPM in only about $12\%$ of training steps. The consistency between AR Diffusion’s performance in the training and inference phases highlights the tightness of the upper bound in Theorem~\ref{thm:main_}. This observation also suggests a potential method for comparing the performance of AR Diffusion and DDPM—by comparing the corrected AR Diffusion training loss with DDPM training loss, we may be able to predict the model's performance during the inference phase.

\paragraph{Additional Experiments on Real-world Data.}
In addition, \cref{app:More Architectures and More Data} considers more real-world data. Specifically, we construct $2\times2$ composite images by concatenating MNIST digits, where the four digits satisfy predefined feature dependencies, such as forming an arithmetic sequence (e.g., $1, 2, 3, 4$) or all being even or odd numbers (e.g., $2, 2, 4, 4$). These additional experiments based on real data further support our analysis that AR learns the dependencies between non-independent features better than DDPM. Furthermore, the impacts of model architecture is also considered in \cref{app:More Architectures and More Data}, where we compare AR and DDPM using different backbones such as MLP and U-Net, and observe consistent conclusions with the experiments mentioned above.

%% file: 0_contents/0Xappendix/appendix_main.tex
\input{0_contents/0Xappendix/0X1_problem_settings}

\input{0_contents/0Xappendix/0X2_inference_complexity}

\input{0_contents/0Xappendix/0XA_auxiliary_lemmas}

\input{0_contents/0Xappendix/0X4_training_details}

%% file: 0_contents/0Xappendix/0X1_problem_settings.tex
\section{Notations in Appendix}
\label{sec:app_probsetting}

\begin{remark}
    
    With the OU (Eq.~\ref{sde:ideal_condi_forward}) and reverse OU process (Eq.~\ref{sde:ideal_condi_reverse}), standard Gaussian and $p_{*,k+1|[1:k]}(\cdot|\vx_{[1:k]})$ can be transformed into each other.
    According to the closed form of Eq.~\ref{sde:ideal_condi_forward}, i.e.,
    \begin{equation*}
        \rvy_t = e^{-t}\cdot \rvy_0 + \sqrt{1-e^{-2t}}\xi\quad \text{where}\quad \xi\sim \gN(\vzero,\mI),
    \end{equation*}
\end{remark}

\paragraph{Training loss and conditional score estimation.}
Here, we note that the trainable parameters in the autoregressive model include
\begin{equation*}
    \vtheta = \left[\vtheta_{\text{ar}}, \vtheta_{\text{df},1},\ldots, \vtheta_{\text{df},K}\right].
\end{equation*}
To simplify the notations, we will not distinguish them strictly and only use $\vtheta$ to present the trainable parameters we are concerned with.
For simplicity, we take the training loss of the $k$-th random variable as an example.
We first denote that 
\begin{equation*}
    \vs_{\vtheta}(\vy|t, \vz)\coloneqq -\sqrt{1-e^{-2t}}\epsilon_{\vtheta}(\vy|t,\vz) 
\end{equation*}
well-trained in the learning stage, where $\epsilon_{\vtheta}$ is directly used to present the training loss provided in~\cite{li2024autoregressive}, i.e.,
\begin{equation*}
    \small
    \begin{aligned}
        \arg\min_{\vtheta}\ L(\vtheta) \coloneqq \frac{1}{K}\sum_{k=1}^K\left[\frac{1}{R}\sum_{r=0}^{R-1} \E_{\rvx_{[1:K]}\sim p_*, \xi\sim \mathcal{N}(\vzero, \mI_{d_k})}\left[\left\|\epsilon_{\vtheta}\left(e^{-(T-t_r)} \rvx_{k} - \sqrt{1-e^{-2(T-t_r)}}\xi \Big|T-t_r, g_{\vtheta}(\rvx_{[1:k-1]})\right) - \xi\right\|^2\right]\right]
    \end{aligned}
\end{equation*}
where we consider the undefined $g_{\vtheta}(\rvx_{[1:0]}) = \mathrm{None}$ to simplify the formulation.
Suppose the training loss is sufficiently small for specific $\vtheta_*$, i.e., $L(\vtheta)\le \epsilon_{\text{score}}$, we have the following lemma
\begin{lemma}
    \label{lem:training_loss_equ}
    Under the previous notation, we have
    \begin{equation*}
        \arg\min_{\vtheta}\ L(\vtheta) = \arg\min_{\vtheta}\ \frac{1}{K}\sum_{k=1}^K \E_{\rvx_{[1:k-1]}\sim p_{*,[1:k-1]}}\left[\frac{1}{R}\sum_{r=0}^{R-1} (1-e^{-2(T-t_r)})\cdot \tilde{L}_{k,r}(\vtheta|\rvx_{[1:k-1]})\right] 
    \end{equation*}
    where
    \begin{equation*}
        \tilde{L}_{k+1,r}(\vtheta|\vx_{[1:k]}) = \E_{\rvy_0\sim p_{*,k+1|[1:k]}(\cdot|\vx_{[1:k]}),\,\xi\sim \mathcal{N}(\vzero,\mI)}\left[\left\|\vs_{\vtheta}(\rvy^\prime|(T-t_r),g_{\vtheta}(\vx_{[1:k]})) - \grad\ln q_{T-t_r}(\rvy^\prime|\vx_{[1:k]})\right\|^2\right]
    \end{equation*}
    and $\rvy^\prime = e^{-(T-t_r)}\cdot \rvy_0 + \sqrt{1-e^{-2(T-t_t)}}\cdot \xi$.
\end{lemma}
\begin{proof}
    For the training loss $L(\vtheta)$, we consider the summation component for each pair $(k,r)$, i.e.,
    \begin{equation*}
        \begin{aligned}
            &\E_{\rvx_{[1:K]}\sim p_*, \xi\sim \mathcal{N}(\vzero, \mI_{d_k})}\left[\left\|\xi_{\vtheta}\left(e^{-(T-t_r)} \rvx_{k} - \sqrt{1-e^{-2(T-t_r)}}\xi \Big|T-t_r, f_{\vtheta}(\rvx_{[1:k-1]})\right) - \xi\right\|^2\right]\\
            & = \E_{\rvx_{[1:k-1]}\sim p_{*,[1:k-1]}}\left[\E_{\rvx_k\sim p_{*, k|[1:k-1]}, \xi\sim \mathcal{N}(\vzero, \mI_{d_k})}\left[\left\|\xi_{\vtheta}\left(e^{-(T-t_r)} \rvx_{k} - \sqrt{1-e^{-2(T-t_r)}}\xi \Big|r\eta, f_{\vtheta}(\rvx_{[1:k-1]})\right) - \xi\right\|^2\right]\right].
        \end{aligned}
    \end{equation*}
    For any given $\rvx_{[1:k-1]} = \vx_{[1:k-1]}$, we consider the SDE.~\ref{sde:ideal_condi_forward}, then we have
    \begin{align}
        \label{eq:training_loss_decompose}
        \begin{aligned}
        & \E_{\rvy_0\sim q_0, \xi\sim \mathcal{N}(\vzero,\mI_{d_k})}\left[\left\|\xi_{\vtheta}(e^{-(T-t_r)}\rvy_0+\sqrt{1-e^{-2(T-t_r)}}\xi|r\eta,\vz)-\xi\right\|^2\right]\\
        & =(1-e^{-2(T-t_r)})\cdot \E_{\rvy_0, \xi}\left[\left\|\vs_{\vtheta}(e^{-(T-t_r)}\rvy_0+\sqrt{1-e^{-2(T-t_r)}}\xi|T-t_r,\vz)+\frac{\xi}{\sqrt{1-e^{-2(T-t_r)}}}\right\|^2\right]\\
        & =(1-e^{-2(T-t_r)})\cdot \E_{\rvy_0, \xi} \left[\left\|\vs_{\vtheta}(e^{-(T-t_r)}\rvy_0+\sqrt{1-e^{-2(T-t_r)}}\xi|T-t_r,\vz)\right\|^2 \right.\\
        & \qquad \left. + \underbrace{\frac{2}{\sqrt{1-e^{-2(T-t_r)}}}\left<\vs_{\vtheta}(e^{-(T-t_r)}\rvy_0+\sqrt{1-e^{-2(T-t_r)}}\xi|T-t_r,\vz), \xi\right>}_{\mathrm{Term 1}}\right]+C.
        \end{aligned}
    \end{align}
    Suppose $\varphi_{\sigma}$ to be the density function of $\mathcal{N}(\vzero,\sigma^2\mI)$, considering Term 1 in Eq.~\ref{eq:training_loss_decompose}, we have
    \begin{equation*}
        \begin{aligned}
            &\mathrm{Term 1} =\frac{2}{\sqrt{1-e^{-2(T-t_r)}}} \cdot  \int \xi\cdot \left(\int \vs_{\vtheta}(e^{-(T-t_r)}\vy_0+\sqrt{1-e^{-2(T-t_r)}}\xi|(T-t_r),\vz)\cdot q_0(\vy_0)\der \vy_0\right) \cdot \varphi_1(\xi) \der \xi\\
            & = \frac{2}{\sqrt{1-e^{-2(T-t_r)}}} \cdot \int q_0(\vy_0)\cdot \left( \int \grad_{\xi}\cdot \vs_{\vtheta}(e^{-(T-t_r)}\vy_0+\sqrt{1-e^{-2(T-t_r)}}\xi|T-t_r,\vz) \cdot  \varphi_1(\xi) \der \xi \right) \der \vy_0\\
            & = 2\cdot \int q_0(\vy_0)\cdot \left(\int \grad_{\vy^\prime}\cdot \vs_{\vtheta}(\vy^\prime|T-t_r,\vz)\cdot \varphi_1\left(\frac{\vy^\prime - e^{-(T-t_r)}\vy_0 }{\sqrt{1-e^{-2(T-t_r)}}}\right) \der \vy^\prime \right) \der \vy_0\\
            & =2\cdot \int \grad_{\vy^\prime}\cdot \vs_{\vtheta}(\vy^\prime|T-t_r,\vz) \cdot q_{T-t_r}(\vy^\prime) \der \vy^\prime = -2 \cdot \int \vs_{\vtheta}(\vy^\prime|T-t_r, \vz)\cdot \grad\ln q_{T-t_r}(\vy^\prime) \cdot q_{T-t_r}(\vy^\prime) \der \vy^\prime 
        \end{aligned}
    \end{equation*}
    where the first equation follows from 
    \begin{equation*}
        \int \left<\xi, \vv(\xi)\right> \varphi(\xi)\der \xi = \int \left<-\grad\ln \varphi_1(\xi), \vv(\xi) \right>\varphi_1(\xi) \der\xi = - \int \left<\grad \varphi_1(\xi), \vv(\xi)\right> \der \xi = \int \grad_{\xi}\cdot \vv(\xi) \cdot \varphi_1(\xi)\der \xi,
    \end{equation*}
    the second equation follows from introducing $\vy^\prime = e^{-(T-t_r)}\vy_0 + \sqrt{1-e^{-2(T-t_r)}}\xi$, and the last equation follows from integral by part. 
    Plugging the above equation into Eq.~\ref{eq:training_loss_decompose}, we have
    \begin{equation}
        \label{eq:training_loss_equ}
        \begin{aligned}
            & \E_{\rvy_0\sim q_0, \xi\sim \mathcal{N}(\vzero,\mI_{d_k})}\left[\left\|\epsilon_{\vtheta}(e^{-(T-t_r)}\rvy_0+\sqrt{1-e^{-2(T-t_r)}}\xi|T-t_r,\vz)-\xi\right\|^2\right]\\
            & =(1-e^{-2(T-t_r)})\cdot \left( \E_{\rvy_0, \xi}\left[\left\|\vs_{\vtheta}(e^{-(T-t_r)}\rvy_0+\sqrt{1-e^{-2(T-t_r)}}\xi|(T-t_r,\vz)\right\|^2\right] \right.\\
            &\qquad\left. - 2\E_{\rvy^\prime\sim q_{(T-t_r)}}\left[\left<\vs_{\vtheta}(\rvy^\prime|T-t_r, \vz)\cdot,\grad\ln q_{T-t_r)}(\rvy^\prime)\right>\right] + C\right)\\
            & =(1-e^{-2(T-t_r)})\cdot \left( \E_{\rvy^\prime\sim q_{T-t_r}}\left[\left\|\vs_{\vtheta}(\rvy^\prime|(T-t_r),\vz) - \grad\ln q_{T-t_r}(\rvy^\prime)\right\|^2\right] + C^\prime \right).
        \end{aligned}
    \end{equation}
    Here, we do not care about the explicit form of $C^\prime$ since it is independent with $\vtheta$ and will finally vanish with $\arg\min$ functions. 
    According to Eq.~\ref{sde:ideal_condi_forward}, we know $\{\rvy\}_{t=0}^T$ is an OU process.
    Suppose 
    \begin{equation}
        \label{def:inner_prac_loss}
        \tilde{L}_{k+1,r}(\vtheta|\vx_{[1:k]}) = \E_{\rvy_0\sim p_{*,k+1|[1:k]}(\cdot|\vx_{[1:k]}),\,\xi\sim \mathcal{N}(\vzero,\mI)}\left[\left\|\vs_{\vtheta}(\rvy^\prime|(T-t_r),g_{\vtheta}(\vx_{[1:k]})) - \grad\ln q_{T-t_r}(\rvy^\prime|\vx_{[1:k]})\right\|^2\right] 
    \end{equation}
    and $\rvy^\prime = e^{-(T-t_r)}\cdot \rvy_0 + \sqrt{1-e^{-2(T-t_t)}}\cdot \xi$,
    then we have
    \begin{equation*}
        \begin{aligned}
            \tilde{L}(\vtheta) =& \frac{1}{K}\sum_{k=1}^K \E_{\rvx_{[1:k-1]}\sim p_{*,[1:k-1]}}\left[\frac{1}{R}\sum_{r=0}^{R-1} (1-e^{-2(T-t_r)})\cdot \tilde{L}_{k,r}(\vtheta|\rvx_{[1:k-1]})\right]
        \end{aligned}
    \end{equation*}
    Since there is $L(\vtheta) = \tilde{L}(\vtheta)+C$, hence the proof is completed.
\end{proof}

\begin{lemma}
    \label{lem:init_smoothness_bound}
    Suppose that we have~\ref{a1}, then for any $k>1$, $\vx,\vx^\prime\in \R^{d_k}$ and $\vy\in \R^{d_1+d_2+\ldots+d_{k-1} }$, we have
    \begin{equation*}
        \left\|\grad\ln p_{*, k|[1:k-1]}(\vx|\vy) -  \grad\ln p_{*, [k|[1:k-1]]}(\vx^\prime|\vy)\right\|\le 2L\cdot \left\|\vx - \vx^\prime\right\|. 
    \end{equation*}
    Besides, we have $\left\|\grad^2 \ln p_{*,[1:1]}(\cdot)\right\|\le 2L$
\end{lemma}
\begin{proof}
    For any $K\ge k\ge 1$, suppose the random variable
    \begin{equation*}
        \left[\rvx_1, \rvx_2, \ldots, \rvx_k, \rvx_{k+1},\ldots, \rvx_K\right] \sim p_*,\quad \rvx\coloneqq \left[\rvx_1, \rvx_2, \ldots, \rvx_k\right],\quad \mathrm{and}\quad \rvy = \left[\rvx_{k+1},\ldots, \rvx_K\right].
    \end{equation*}
    According to the definition Eq.~\ref{eq:marginal_def_case}, the marginal distribution of $\rvx$ is
    \begin{equation*}
        p_{*,[1:k]}(\vx) = \int p_*(\vx,\vy) \der \vy,
    \end{equation*}
    which implies
    \begin{equation*}
        \begin{aligned}
            &\grad_{\vx}\ln p_{*,[1:k]}(\vx) = \frac{\grad_{\vx} p_{*,[1:k]}(\vx)}{p_{*,[1:k]}(\vx)} = \frac{ \partial_{\vx}\int \exp\left(-f_*(\vx,\vy)\right)\der\vy }{\int \exp\left(-f_*(\vx,\vy)\right)\der\vy} = \frac{\int -\partial_{\vx} f_*(\vx,\vy)\cdot \exp\left(-f_*(\vx,\vy)\right) \der \vy}{\int \exp\left(-f_*(\vx,\vy)\right)\der\vy}.
        \end{aligned}
    \end{equation*}
    Furthermore, if we check its Hessian matrix, it has
    \begin{equation*}
        \begin{aligned}
            & \grad^2_{\vx}\ln p_{*,[1:k]}(\vx) = \partial_{\vx}\left[ \frac{\int -\partial_{\vx} f_*(\vx,\vy)\cdot \exp\left(-f_*(\vx,\vy)\right) \der \vy}{\int \exp\left(-f_*(\vx,\vy)\right)\der\vy}\right]\\
            & = \frac{\int \left[-\partial^2_{\vx}f_*(\vx,\vy) + \partial_{\vx} f_*(\vx,\vy)\partial^\top_{\vx} f_*(\vx,\vy)\right] \cdot \exp\left(-f_*(\vx,\vy)\right) \der \vy }{\int \exp\left(-f_*(\vx,\vy)\right)\der\vy}\\
            &\quad - \left[ \frac{\int \partial_{\vx} f_*(\vx,\vy)\cdot \exp\left(-f_*(\vx,\vy)\right) \der \vy}{\int \exp\left(-f_*(\vx,\vy)\right)\der\vy}\right]\cdot \left[ \frac{\int \partial_{\vx} f_*(\vx,\vy)\cdot \exp\left(-f_*(\vx,\vy)\right) \der \vy}{\int \exp\left(-f_*(\vx,\vy)\right)\der\vy}\right]^\top \\
            & = \E_{\rvy\sim p_{*,[k+1:K]|[1:k]}(\cdot|\vx)}\left[-\partial^2_{\vx}f_*(\vx,\rvy)\right] + \Var_{\rvy\sim p_{*,[k+1:K]|[1:k]}(\cdot|\vx)}\left[\partial_{\vx} f_*(\vx,\rvy)\right].
        \end{aligned}
    \end{equation*}
    Then, such a matrix can be relaxed to
    \begin{equation}
        \label{ineq:marginal_hessian_bound}
        \begin{aligned}
            &\left\|\grad^2_{\vx}\ln p_{*,[1:k]}(\vx)\right\|\le \left\|\E_{\rvy\sim p_{*,[k+1:K]|[1:k]}(\cdot|\vx)}\left[\partial^2_{\vx}f_*(\vx,\rvy)\right]\right\| + \left\|\Var_{\rvy\sim p_{*,[k+1:K]|[1:k]}(\cdot|\vx)}\left[\partial_{\vx} f_*(\vx,\rvy)\right]\right\|\\
            & \le \left\|\E_{\rvy\sim p_{*,[k+1:K]|[1:k]}(\cdot|\vx)}\left[\left[\mI, \vzero\right]\cdot\grad^2 f_*(\vx,\rvy)\cdot \left[
            \begin{matrix}
                \mI\\
                \vzero
            \end{matrix}
            \right]\right]\right\| + \left\|\E_{\rvy\sim p_{*,[k+1:K]|[1:k]}(\cdot|\vx)}\left[\partial_{\vx} f_*(\vx,\rvy)\partial^\top_{\vx} f_*(\vx,\rvy)\right]\right\|\\
            &\le \left\|\E_{\rvy\sim p_{*,[k+1:K]|[1:k]}(\cdot|\vx)}\left[\grad^2 f_*(\vx,\rvy)\right]\right\| + \left\|\E_{\rvy\sim p_{*,[k+1:K]|[1:k]}(\cdot|\vx)}\left[\left\|\partial_{\vx} f_*(\vx,\rvy)\right\|^2\cdot \mI\right]\right\| \le 2L,
        \end{aligned}
    \end{equation}
    where the last inequality follows from~\ref{a1}.
    
    Then, suppose $k\ge 1$, we have
    \begin{equation*}
        p_{*,k|[1:k-1]}(\vx_k|\vx_{[1:k-1]}) = \frac{p_{*,[1:k]}(\vx_{[1:k]})}{\int p_{*,[1:k]}(\vx_1,\vx_2,\ldots,\vx_k)\der \vx_k} = p_{*,[1:k]}(\vx_{[1:k]})\cdot Z(\vx_{[1:k-1]}).
    \end{equation*}
    For notation simplicity, we define $\vy\coloneqq (\vx_1,\vx_2,\ldots,\vx_{k-1})$, then it has
    \begin{equation*}
        \begin{aligned}
            &\left\|\grad\ln p_{*,k|[1:k-1]}(\vx|\vy) - \grad \ln p_{*,k|[1:k-1]}(\vx^\prime|\vy)\right\| = \left\|\partial_{\vx} \ln p_{*,k|[1:k-1]}(\vy,\vx) - \partial_{\vx^\prime} \ln p_{*,k|[1:k-1]}(\vy,\vx^\prime)\right\|\\
            & = \left\|\left[\vzero,\mI\right]\cdot \left[ \grad\ln p_{*,k|[1:k-1]}(\vy,\vx) - \grad \ln p_{*,k|[1:k-1]}(\vy,\vx^\prime)\right]\right\|\le 2L\cdot \left\|\vx - \vx^\prime\right\|
        \end{aligned} 
    \end{equation*}
    where the last inequality follows from the Hessian upper bound of marginal $p^{(k)}$, i.e., Eq.~\ref{ineq:marginal_hessian_bound}.
    Hence, the proof is completed.
\end{proof}

\begin{lemma}[Bounded Second Moments]
    \label{lem:init_second_moment_bounded}
    Suppose we have \ref{a2}, then for any $K>k^\prime \ge 1$, we have
    \begin{equation*}
        \E_{\rvx_{[1:k^\prime]}\sim p_{*,[1:k^\prime]}}[\|\rvx_{1:k^\prime}\|^2]\le m_0
    \end{equation*}
    and
    \begin{equation*}
        \sum_{k=0}^{k^\prime} \E_{\rvx_{[1:k]}\sim p_{*,[1:k]}}\left[\E_{\rvy \sim p_{*,k+1|[1:k]}(\cdot|\rvx_{[1:k]})}\left[\|\rvy\|^2\right]\right]
    \end{equation*}
\end{lemma}
\begin{proof}
    For the second moment of $p_*$ and $K > k\ge 1$, we have
    \begin{equation*}
        \begin{aligned}
        &\E_{\rvx_{[1:k]}\sim p_{*,[1:k]}}[\|\rvx_{1:k}\|^2] = \int \|\vx_{[1:k]}\|^2\cdot p_{*,[1:k]}(\vx_{[1:k]})\der \vx_{[1:k]}\\
        & = \int \|\vx_{[1:k]}\|^2 \cdot \int p_*(\vx_{[1:k]},\vx_{[k+1:K]}) \der \vx_{[1:k]} \der \vx_{[k+1:K]}\\
        & \le \int \|(\vx_1,\vx_{[2:K-1]})\|^2 p_*(\vx_{1:k},\vx_{[k+1:K]}) \der \vx_{1:k} \der \vx_{[k+1:K]} = m_0.
        \end{aligned}
    \end{equation*}
    Besides, we have
    \begin{equation*}
        \begin{aligned}
            &\sum_{k=0}^{K-1} \E_{\rvx_{[1:k]}\sim p_{*,[1:k]}}\left[\E_{\rvy \sim p_{*,k+1|[1:k]}(\cdot|\rvx_{[1:k]})}\left[\|\rvy\|^2\right]\right]\\
            &=\sum_{k=0}^{K-1} \int \|\vx_{k+1}\|^2 \cdot p_{*,[1:k+1]}(\vx_{[1:k]},\vx_{k+1})\der (\vx_{1:k},\vx_{k+1})\\
            &= \sum_{k=0}^{K-1} \int \|\vx_{k+1}\|^2 p_{*}(\vx_{[1:K]}) \der \vx_{[1:K]} = \int \sum_{k=0}^{K-1}\|\vx_{k+1}\|^2 p_{*}(\vx_{[1:K]}) \der \vx_{[1:K]} = m_0 .
        \end{aligned}
    \end{equation*}
    Hence, the proof is completed.
\end{proof}

%% file: 0_contents/0Xappendix/0X2_inference_complexity.tex
\section{Inference Complexity}
In this section, the notation of the inference process follows from those defined in Section~\ref{sec:app_probsetting}.

\begin{theorem}
    Suppose Assumption~\ref{a2}-\ref{a3} hold, 
    \begin{equation*}
        \delta\le \Big(0, \ln \sqrt{(4L)^{-2}+1}+(4L)^{-1}\Big]
    \end{equation*}
    if Alg.~\ref{alg:ard_outer} chooses the time sequence $\{\eta_r\}_{r=0}^{R-1}$ as
    \begin{equation*}
        \eta_r = \left\{
            \begin{aligned}
                & \eta && \text{when}\quad 0\le r<M\\
                & \eta/(1+\eta)^{r-M+1} && \text{when}\quad M\le r<N\\
                & \eta && \text{when}\quad N\le r\le R
            \end{aligned}
        \right.
    \end{equation*}
    where 
    \begin{equation*}
        M=\frac{T-1}{\eta},\quad N=M+\frac{2\ln(1/\delta)}{\eta},\quad \text{and}\quad R=N+\frac{\delta}{\eta},
    \end{equation*}
    then, the generated samples
    $[\hat{\rvx}_1, \hat{\rvx}_2, \ldots, \hat{\rvx}_K]$ follows the distribution $\hat{p}_*$, which satisfies
    \begin{equation*}
         \begin{aligned}
             &\KL{p_*}{\hat{p}_*}\lesssim  2e^{-2T}L\cdot \left(m_0+ d\right)\\
            &\quad+ (L^2R\eta^2+T\eta)\cdot d  + \eta m_0 + \eta KR\cdot \epsilon_{\mathrm{score}}^2.
         \end{aligned}
    \end{equation*}
\end{theorem}
\begin{proof}
    Suppose the inference process generate the sample $\hat{\rvx} = (\hat{\rvx}_1^, \ldots, \hat{\rvx}_K)$ with density function $\hat{p}_*$ which satisfies
    \begin{equation}
        \hat{p}_*(\hat{\vx}_{[1:K]}) = \hat{p}_{*, 1}(\hat{\vx}_1)\cdot \prod_{k=2}^K \hat{p}_{*,k|[1:k-1]}(\hat{\vx}_k|\hat{\vx}_{1:k-1}).
    \end{equation}
    We expect to have $\TVD{p_*}{\hat{p}_*}\le \epsilon$ or $\KL{p_*}{\hat{p}_*}\le 2\epsilon^2$, which can be relaxed to
    \begin{equation*}
        \begin{aligned}
            \KL{p_*}{\hat{p}_*}\le \KL{p_{*,[1:K-1]}}{\hat{p}_{*,[1:K-1]}} + \E_{\rvx_{[1:K-1]}\sim p_{*,[1:K-1]}}\left[\KL{p_{*,K|[1:K-1]}(\cdot|\rvx_{[1:K-1]})}{\hat{p}_{*,K|[1:K-1]}(\cdot|\rvx_{[1:K-1]})}\right]
        \end{aligned}
    \end{equation*}
    following from the chain rule of TV distance, i.e., Lemma~\ref{lem:tv_chain_rule}.
    Within a recursive manner, we have 
    \begin{equation}
        \label{ineq:infer_error_0}
        \begin{aligned}
            \KL{p_*}{\hat{p}_*}\le \KL{p_{*,1}}{\hat{p}_{*,1}} + \sum_{k=1}^{K-1} \E_{\rvx_{[1:k]}\sim p_{*,[1:k]}}\left[\KL{p_{*,k+1|[1:k]}(\cdot|\rvx_{[1:k]})}{\hat{p}_{*,k+1|[1:k]}(\cdot|\rvx_{[1:k]})}\right].
        \end{aligned}
    \end{equation}
     According to Lemma~\ref{lem:condi_rtk_error_eachk} and Remark~\ref{rmk:init_rev_error}, Eq.~\ref{ineq:infer_error_0} can further be relaxed to 
     \begin{equation}   
        \label{ineq:final_kl_mid_0}
        \begin{aligned}
            &\KL{p_*}{\hat{p}_*}\lesssim e^{-2T}\cdot \left(2Ld_{1} + \E_{p_{*,1}}\left[\|\rvy\|^2\right]\right) + \eta\cdot \sum_{r=0}^{R-1} \tilde{L}_{1,r}(\vtheta) +  d_{1}L^2R\eta^2 + d_{1}T\eta+ \eta \E_{p_{*,1}}\left[\|\rvy\|^2\right]\\
            &\quad  + \underbrace{\sum_{k=1}^{K-1}\E_{\rvx_{[1:k]}\sim p_{*,[1:k]}}\left[e^{-2T}\cdot \left(2Ld_{k+1} + \E_{p_{*,k+1|[1:k]}(\cdot|\vx_{[1:k]})}\left[\|\rvy\|^2\right] \right)\right]}_{\text{Term 1}} \\
            & \underbrace{+ \sum_{k=1}^{K-1}\E_{\rvx_{[1:k]}\sim p_{*,[1:k]}}\left[\eta\cdot \sum_{r=0}^{R-1} \tilde{L}_{k+1,r}(\vtheta|\vx_{[1:k]}) +  d_{k+1}L^2R\eta^2 + d_{k+1}T\eta+ \eta \E_{p_{*,k+1|[1:k]}(\cdot|\vx_{[1:k]})}\left[\|\rvy\|^2\right]\right]}_{\text{Term 2}}.
        \end{aligned}
    \end{equation}
    For Term 1 of Eq.~\ref{ineq:final_kl_mid_0}, we have
    \begin{equation}
        \label{ineq:final_kl_mid_0_term_1}
        \begin{aligned}
            \text{Term 1} = e^{-2T}\cdot \left(2L\sum_{k=1}^{K-1}d_{k+1} + \E_{\rvx_{[1:k]}\sim p_{*,[1:k]}}\left[\E_{p_{*,k+1|[1:k]}(\cdot|\vx_{[1:k]})}\left[\|\rvy\|^2\right]\right]\right),
        \end{aligned}
    \end{equation}
    For Term 2 of Eq.~\ref{ineq:final_kl_mid_0}, with the same technique, we have
    \begin{equation}
        \label{ineq:final_kl_mid_0_term_2}
        \begin{aligned}
            \text{Term 2} =& \eta\cdot \sum_{k=1}^{K-1}\sum_{r=0}^{R-1} \E_{\rvx_{[1:k]}\sim p_{*,[1:k]}}\left[\tilde{L}_{k+1,r}(\vtheta|\vx_{[1:k]})\right] + (L^2R\eta^2+T\eta)\cdot \sum_{k=1}^{K-1}d_{k+1}\\
            &+ \eta\cdot \sum_{k=1}^K\E_{\rvx_{[1:k]}\sim p_{*,[1:k]}}\left[\E_{p_{*,k+1|[1:k]}(\cdot|\vx_{[1:k]})}\left[\|\rvy\|^2\right]\right].
        \end{aligned}
    \end{equation}
    Plugging Eq.~\ref{ineq:final_kl_mid_0_term_1} and Eq.~\ref{ineq:final_kl_mid_0_term_2} into Eq.~\ref{ineq:final_kl_mid_0}, we have
    \begin{equation*}
        \small
        \begin{aligned}
            \KL{p_*}{\hat{p}_*}\lesssim & 2e^{-2T}L\cdot \sum_{k=0}^{K-1} \left(d_{k+1} + \E_{\rvx_{[1:k]}\sim p_{*,[1:k]}}\left[\E_{\rvy \sim p_{*,k+1|[1:k]}(\cdot|\rvx_{[1:k]})}\left[\|\rvy\|^2\right]\right] \right)\\
            & + \eta\cdot \sum_{k=0}^{K-1}\sum_{r=0}^{R-1} \E_{\rvx_{[1:k]}\sim p_{*,[1:k]}}\left[\tilde{L}_{k+1,r}(\vtheta|\vx_{[1:k]})\right] + (L^2R\eta^2+T\eta)\cdot \sum_{k=0}^{K-1}d_{k+1} \\
            & + \eta\cdot \sum_{k=0}^K\E_{\rvx_{[1:k]}\sim p_{*,[1:k]}}\left[\E_{p_{*,k+1|[1:k]}(\cdot|\vx_{[1:k]})}\left[\|\rvy\|^2\right]\right]\\
            \le & 2e^{-2T}L\cdot \left(m_0+ \sum_{k=1}^{K} d_{k}\right) + \frac{\eta KR}{KR}\cdot \sum_{k=0}^{K-1}\sum_{r=0}^{R-1} \E_{\rvx_{[1:k]}\sim p_{*,[1:k]}}\left[\tilde{L}_{k+1,r}(\vtheta|\vx_{[1:k]})\right] + (L^2R\eta^2+T\eta)\cdot \sum_{k=0}^{K-1}d_{k+1}  + \eta m_0\\
            \le & 2e^{-2T}L\cdot \left(m_0+ \sum_{k=1}^{K} d_{k}\right) + (L^2R\eta^2+T\eta)\cdot \sum_{k=0}^{K-1}d_{k+1}  + \eta m_0 + \eta KR\cdot \epsilon_{\mathrm{score}}^2
        \end{aligned}
    \end{equation*}
    where the second inequality follows from Lemma~\ref{lem:init_second_moment_bounded} and the last inequality follows from a small training error, i.e., \ref{a3}.
    Without loss of generality, we suppose $L\ge 1$, 
    \begin{equation*}
        d= \sum_{k=1}^K d_k\quad \text{and}\quad c=\ln \frac{\sqrt{\frac{1}{4L^2} + 4} + \frac{1}{2L}}{2} < 1,
    \end{equation*}
    then by requiring
    \begin{equation}
        \label{def:hyper_req}
        \begin{aligned}
            &T = \Theta\left(\ln \frac{8L(m_0 + d)}{\epsilon^2}\right),\quad   R \eta = \Theta(T+\ln 1/c) =\Theta\left(\ln \frac{16L(m_0 + d)}{\epsilon^2\cdot \sqrt{\frac{1}{4L^2} + 4} + \frac{1}{2L}} \right),\\
            &\epsilon_{\mathrm{score}} = \Theta\left(\frac{\epsilon}{2\sqrt{K\cdot R\eta}}\right) = \Theta\left(\frac{\epsilon}{2\sqrt{K}}\cdot \left(\ln \frac{16L(m_0 + d)}{\epsilon^2\cdot \sqrt{\frac{1}{4L^2} + 4} + \frac{1}{2L}}\right)^{-1/2}\right),\quad \text{and}\\
            & \eta = O\left(\frac{\epsilon^2}{4L^2(R\eta)d}\right) = O\left(\frac{\epsilon^2}{4L^2 d} \cdot \left(\ln \frac{16L(m_0 + d)}{\epsilon^2\cdot \sqrt{\frac{1}{4L^2} + 4} + \frac{1}{2L}}\right)^{-1}\right),
        \end{aligned}
    \end{equation}
    we will have $\KL{p_*}{\hat{p}_*}\lesssim \epsilon^2$ and $\TVD{p_*}{\hat{p}_*}\le \epsilon$.
    To calculate the gradient complexity, which is
    \begin{equation}
        \label{ineq:grad_comp}
        \begin{aligned}
            O\left(KR\right) = O\left(\frac{KR\eta}{\eta}\right) = &  O\left(K\cdot \ln \frac{16L(m_0 + d)}{\epsilon^2\cdot \sqrt{\frac{1}{4L^2} + 4} + \frac{1}{2L}}\cdot \frac{4L^2d}{\epsilon^2}\cdot \ln \frac{16L(m_0 + d)}{\epsilon^2\cdot \sqrt{\frac{1}{4L^2} + 4} + \frac{1}{2L}}  \right) \\
            = & O\left(\frac{KL^2d}{\epsilon^2}\cdot\ln \frac{L(m_0+d)}{\epsilon}\right).
        \end{aligned}
    \end{equation}
    Hence, the proof is completed.
\end{proof}

\begin{lemma}
    \label{lem:condi_rtk_error_eachk}
    For any $k\ge 1$, for any $k$-tuples $\vx_{[1:k]}\in\R^{d_1+d_2+\ldots+d_k}$, we consider the SDE.~\ref{sde:prac_condi_reverse} to simulate the reverse process of SDE.~\ref{sde:ideal_condi_forward}, 
    then we have
    \begin{equation*}
        \begin{aligned}
            && \KL{p_{*,k+1|[1:k]}(\cdot|\rvx_{[1:k]})}{\hat{p}_{*,k+1|[1:k]}(\cdot|\rvx_{[1:k]})} \lesssim e^{-2T}\cdot \left(2Ld_{k+1} + \E_{p_{*,k+1|[1:k]}(\cdot|\vx_{[1:k]})}\left[\|\rvy\|^2\right] \right)\\
            && + \eta\cdot \sum_{r=0}^{R-1} \tilde{L}_{k+1,r}(\vtheta|\vx_{[1:k]}) +  d_{k+1}L^2R\eta^2 + d_{k+1}T\eta+ \eta \E_{p_{*,k+1|[1:k]}(\cdot|\vx_{[1:k]})}\left[\|\rvy\|^2\right].
        \end{aligned}
    \end{equation*}
\end{lemma}
\begin{proof}
     Similar as \cite{benton2024nearly} and \cite{chen2023improved}, we consider the step size satisfying $\eta_r \leq \eta \min(1, T-t_{r+1})$. $\eta$ is the parameter for controlling the maximum step size. We denote the conditional variance of $\rvy_t$ given $\rvy_s$ as $\sigma^2_{s,t}:=1-e^{-2(t-s)}$ and the conditional expectation as $\alpha_{s,t} := \mathrm{E}(\rvy_t|\rvy_s)=e^{-(t-s)}$ where $0 \leq s \le t \le T$. And we denote the posterior mean as $\boldsymbol{\mu}_t :=\E_{q_{0|t}}(\rvy_0)$ and posterior variance as $\mathbf{\Sigma}_t :=\mathrm{Cov}_{q_{0|t}}(\rvy_0)$. Suppose $\hat{p}_{*,k+1|[1:k]}(\cdot|\vx_{[1:k]}) = \hat{q}_T(\cdot)$, we have
    \begin{equation}
        \label{ineq:kl_error_eachk}
        \begin{aligned}
            &\KL{p_{*,k+1|[1:k]}(\cdot|\rvx_{[1:k]})}{\hat{p}_{*,k+1|[1:k]}(\cdot|\rvx_{[1:k]})}\\
            &  = \KL{q^\gets_T}{\hat{q}_T} \le \KL{q^\gets_{t_{R-1}}}{\hat{q}_{t_{R-1}}}+ \E_{\rvy^\gets\sim q^\gets_{t_{R-1}}}\left[\KL{q^\gets_{t_R|t_{R-1}}(\cdot|\rvy^\gets)}{\hat{q}_{t_R|t_{R-1}}(\cdot|{\rvy}^\gets)}\right]\\
            & \le \KL{q^\gets_0}{\hat{q}_0} + \underbrace{ \sum_{r=0}^{R-1} \E_{{\rvy}^\gets\sim {q}^\gets_{t_r}}\left[\KL{{q}^\gets_{t_{r+1}|t_r}(\cdot|\tilde{\rvy})}{\hat{q}_{t_{r+1}|t_r}(\cdot|\tilde{\rvy})}\right]}_{\text{reverse transition error}}
        \end{aligned}
    \end{equation}
    where the first inequality follows from the chain rule, i.e., Lemma~\ref{lem:tv_chain_rule}, of KL divergence, and the second one follows from the recursive manner. 
    Besides, we have $\hat{q}_0 = \varphi_1$ which denotes the density function of $\mathcal{N}(\vzero, \mI)$.
    \paragraph{Initialization Error.} We first consider to upper bound $\TVD{q_T}{\varphi_1}$. 
    Due to Lemma~\ref{lem:initialization_error}, we have
    \begin{equation}
        \label{ineq:kl_error_init}
        \begin{aligned}
            & \KL{{q}^\gets_0}{\hat{q}_0} = \KL{q_T}{\varphi_1}\le e^{-2T}\cdot \KL{p^\to_{k+1|[1:k], 0}(\cdot|\vx_{1:k})}{\varphi_1}\\
            & \le e^{-2T}\cdot \left(2Ld_{k+1} + \E_{p_{*,k+1|[1:k]}(\cdot|\vx_{[1:k]})}\left[\|\rvy\|^2\right] \right).
        \end{aligned}
    \end{equation}

    \paragraph{Reverse Transition Error.}
    According to Lemma~\ref{lem:kl_to_drift_diff_girsanov}, the reverse transition error can be relaxed as
    \begin{equation}
        \label{ineq:kl_error_beyond_init_0}
        \begin{aligned}
            &\sum_{r=0}^{R-1} \E_{\rvy^\gets\sim q^\gets_{t_r}}\left[\KL{q^\gets_{t_{r+1}|t_r}(\cdot|\rvy^\gets)}{\hat{q}_{t_{r+1}|t_r}(\cdot|\rvy^\gets)}\right]\\
            &\le  \sum_{r=0}^{R-1} \E_{\rvy^\gets_{[0:T]}\sim {Q}^\gets_{[0:T]}}\left[\int_{t_r}^{t_{r+1}} \left\|\vs_{\vtheta}(\rvy^\gets_{t_r}|T - t_r,\vz )-  \grad\ln q_{T-t}(\rvy^\gets_t) \right\|^2 \der t \right]\\
            & \le \underbrace{\sum_{r=0}^{R-1} \eta_r\cdot \E_{\rvy^\gets_{t_r}\sim q^\gets_{t_r}}\left[\left\|\vs_{\vtheta}(\rvy^\gets_{t_r}|T - t_r,\vz )-  \grad\ln q_{T-t_r}(\rvy^\gets_{t_r})\right\|^2\right]}_{\text{score estimation error}}\\
            &\quad + \underbrace{\sum_{r=0}^{R-1} \E_{\tilde{\rvy}_{[0:T]}\sim \tilde{Q}_{[0:T]}}\left[\int_{t_r}^{t_{r+1}} \left\|\grad\ln q_{T-t_r}(\tilde{\rvy}_{t_r})-  \grad\ln q_{T-t}(\tilde{\rvy}_t) \right\|^2 \der t \right]}_{\text{discretization error}}
        \end{aligned}
    \end{equation}
    Since we have $q^\gets_{T-t} = q_{t}$ with Eq.~\ref{sde:ideal_condi_reverse}, the score estimation error satisfies
    \begin{equation}
        \label{ineq:score_estimation_error_inner}
        \begin{aligned}
            &\sum_{r=0}^{R-1} \eta_r \cdot \E_{\rvy^\gets_{t_r}\sim q^\gets_{t_r}}\left[\left\|\vs_{\vtheta}(\rvy^\gets_{t_r}|T - t_r,\vz )-  \grad\ln q_{T-t_r}(\rvy^\gets_{t_r})\right\|^2\right]\\
            &\le \eta\cdot \sum_{r=0}^{R-1} \E_{\rvy^\gets_{t_r}\sim q_{T-t_r}}\left[\left\|\vs_{\vtheta}(\rvy^\gets_{t_r}|T - t_r,\vz )-  \grad\ln q_{T-t_r}(\rvy^\gets_{t_r})\right\|^2\right]\\
            & = \eta\cdot \sum_{r=0}^{R-1} \E_{\rvx^\prime\sim p^\to_{k+1|[1:k], t_r}(\cdot|\vx_{[1:k]})}\left[\left\|\vs_{\vtheta}(\rvx^\prime|T - t_r,g_{\vtheta}(\vx_{[1:k]})) - \grad\ln p^\to_{k+1|[1:k], T - t_r}(\rvx^\prime|\vx_{[1:k]})\right\|^2\right]\\
            &= \eta\cdot  \sum_{r=0}^{R-1} \tilde{L}_{k+1,r}(\vtheta|\vx_{[1:k]})
        \end{aligned}
    \end{equation}
    where the second inequality follows from the definition $q_t = q^\gets_{T-t}$ and the last equation follows from the definition Eq.~\ref{def:inner_prac_loss} in Lemma~\ref{lem:training_loss_equ}.    
    Considering the discretization error, we have
    \begin{equation}
        \label{ineq:discretization_error_inner}
        \begin{aligned}
            & \sum_{r=0}^{R-1} \E_{\rvy^\gets_{[0:T]}\sim Q^\gets_{[0:T]}}\left[\int_{t_r}^{t_{r+1}} \left\|\grad\ln q_{T-t_r}(\rvy^\gets_{t_r})-  \grad\ln q_{T-t}(\rvy^\gets_t) \right\|^2 \der t \right]\\
            & \lesssim d_{k+1}L^2R\eta^2 + d_{k+1}T\eta+ \eta \E_{p_{*,k+1|[1:k]}(\cdot|\vx_{[1:k]})}\left[\|\rvy\|^2\right]
        \end{aligned}
    \end{equation}
    due to Lemma~\ref{lem:discretization_error_each_k}.
    Therefore, plugging Eq.~\ref{ineq:score_estimation_error_inner} and Eq.~\ref{ineq:discretization_error_inner} into Eq.~\ref{ineq:kl_error_beyond_init_0}, we have
    \begin{equation}
        \label{ineq:kl_error_beyond_init_fin}
        \begin{aligned}
            & \sum_{r=0}^{R-1} \E_{\rvy^\gets\sim q^\gets_{t_r}}\left[\KL{{q}^\gets_{t_{r+1}|t_r}(\cdot|\rvy^\gets)}{\hat{q}_{t_{r+1}|t_r}(\cdot|\rvy^\gets)}\right]\\
            & \lesssim \eta\cdot \sum_{r=0}^{R-1} \tilde{L}_{k+1,r}(\vtheta|\vx_{[1:k]}) +  d_{k+1}L^2R\eta^2 + d_{k+1}T\eta+ \eta \E_{p_{*,k+1|[1:k]}(\cdot|\vx_{[1:k]})}\left[\|\rvy\|^2\right].
        \end{aligned}
    \end{equation}
    Combining with Eq.~\ref{ineq:kl_error_init}, Eq.~\ref{ineq:kl_error_eachk} can be written as
    \begin{equation*}
        \begin{aligned}
            && \KL{p_{*,k+1|[1:k]}(\cdot|\rvx_{[1:k]})}{\hat{p}_{*,k+1|[1:k]}(\cdot|\rvx_{[1:k]})} \lesssim e^{-2T}\cdot \left(2Ld_{k+1} + \E_{p_{*,k+1|[1:k]}(\cdot|\vx_{[1:k]})}\left[\|\rvy\|^2\right] \right)\\
            && + \eta\cdot \sum_{r=0}^{R-1} \tilde{L}_{k+1,r}(\vtheta|\vx_{[1:k]}) +  d_{k+1}L^2R\eta^2 + d_{k+1}T\eta+ \eta \E_{p_{*,k+1|[1:k]}(\cdot|\vx_{[1:k]})}\left[\|\rvy\|^2\right].
        \end{aligned}
    \end{equation*}
    and the proof is completed.
\end{proof}

\begin{remark}
    \label{rmk:init_rev_error}
    In Lemma~\ref{lem:condi_rtk_error_eachk}, we require $k\ge 1$ and calculate the upper bound between the conditional transition kernel when $\rvx_{[1:k]} = \vx_{[1:k]}$.
    While this lemma can easily be adapted to the case for the unconditional generation of the distribution $\hat{p}_{*,1}$.
    In the process to generate $\hat{\rvx}_1$, we should only consider Eq.~\ref{sde:ideal_condi_forward} -- Eq.~\ref{sde:prac_condi_reverse} as $q_0 = q^\gets_T = p_{*,1}$ rather than $q_0=\hat{q}_T = p_{*, k+1|[1:k]}(\cdot|\vx_{[1:k]})$. 
    Then, suppose $\hat{p}_{*,1}(\vx) = \hat{q}_T(\vx) $, we have
    \begin{equation*}
        \begin{aligned}
            &\KL{p_{*,1}}{\hat{p}_{*,1}} = \KL{q^\gets_T}{\hat{q}_T} \le \KL{q^\gets_{t_{R-1}}}{\hat{q}_{t_{R-1}}}+ \E_{\rvy^\gets\sim q^\gets_{t_{R-1}}}\left[\KL{q^\gets_{t_R|t_{R-1}}(\cdot|\rvy^\gets)}{\hat{q}_{t_R|t_{R-1}}(\cdot|\rvy^\gets)}\right]\\
            & \le \KL{q^\gets_0}{\hat{q}_0} + \underbrace{ \sum_{r=0}^{R-1} \E_{\rvy^\gets\sim q^\gets_{t_r}}\left[\KL{q^\gets_{t_{r+1}|t_r}(\cdot|\rvy^\gets)}{\hat{q}_{t_{r+1}|t_r}(\cdot|\rvy^\gets)}\right]}_{\text{reverse transition error}}
        \end{aligned}
    \end{equation*}
    due to Lemma~\ref{lem:kl_chain_rule}.
    The control of $\KL{q^\gets_0}{\hat{q}_0}$ is similar to Eq.~\ref{ineq:kl_error_init}, 
    and have
    \begin{equation*}
        \begin{aligned}
            &\KL{q^\gets_0}{\hat{q}_0} = \KL{q_T}{\varphi_1} \le e^{-2T}\cdot \KL{p^\to_{k+1|[1:k], 0}(\cdot|\vx_{1:k})}{\varphi_1} \le e^{-2T}\cdot \left(2Ld_{1} + \E_{p_{*,1}}\left[\|\rvy\|^2\right] \right)
        \end{aligned}
    \end{equation*}
    where the last inequality follows from Lemma~\ref{lem:init_second_moment_bounded}.
    Additionally, the control of the expected conditional KL divergence gap is similar to Eq.~\ref{ineq:kl_error_beyond_init_fin}, which satisfies
    \begin{equation*}
        \begin{aligned}
            & \sum_{r=0}^{R-1} \E_{\rvy^\gets\sim q^\gets_{t_r}}\left[\KL{q^\gets_{t_{r+1}|t_r}(\cdot|\rvy^\gets)}{\hat{q}_{t_{r+1}|t_r}(\cdot|\rvy^\gets)}\right] \lesssim \eta\cdot \sum_{r=0}^{R-1} \tilde{L}_{1,r}(\vtheta) +  d_{1}L^2R\eta^2 + d_{1}T\eta+ \eta \E_{p_{*,1}}\left[\|\rvy\|^2\right].
        \end{aligned}
    \end{equation*}
    Here, with a little abuse of notation, we have 
    \begin{equation*}
        \tilde{L}_{1,r}(\vtheta) = \tilde{L}_{1,r}(\vtheta|\vx_{[1:0]})\quad \text{and}\quad \vx_{[1:0]} = \mathrm{none}
    \end{equation*}
    in the training loss.
    Therefore, in summary, we have
    \begin{equation}
        \label{ineq:uncondi_ke1_error}
        \KL{p_{*,1}}{\hat{p}_{*,1}}\le e^{-2T}\cdot \left(2Ld_{1} + \E_{p_{*,1}}\left[\|\rvy\|^2\right]\right) + \eta\cdot \sum_{r=0}^{R-1} \tilde{L}_{1,r}(\vtheta) +  d_{1}L^2R\eta^2 + d_{1}T\eta+ \eta \E_{p_{*,1}}\left[\|\rvy\|^2\right].
    \end{equation}
\end{remark}

\begin{lemma}[Adapted from Theorem 4 in~\cite{vempala2019rapid}]
    \label{lem:initialization_error}
    Along the Langevin dynamics SDE.~\ref{sde:ideal_condi_forward}, for any $t\in[0,T]$ and $k\ge 1$, we have
    \begin{equation*}
        \KL{p^\to_{k+1|[1:k], t}(\cdot|\vx_{1:k})}{\varphi_1}\le e^{-2t}\cdot \left(2Ld + \E_{p_{*,k+1|[1:k]}(\cdot|\vx_{[1:k]})}\left[\|\rvy\|^2\right] \right)
    \end{equation*}
\end{lemma}
\begin{proof}
    The Fokker-Planck equation of SDE.~\ref{sde:ideal_condi_forward}, i.e.,
    \begin{equation*}
        \begin{aligned}
            & \frac{\partial p^\to_{k+1|[1:k], t}(\vy|\vx_{1:k})}{\partial t} = \grad\cdot \left(p^\to_{k+1|[1:k], t}(\vy|\vx_{1:k}) \cdot \vy \right) + \Delta p^\to_{k+1|[1:k], t}(\vy|\vx_{1:k})\\
            &= \grad\cdot \left(p^\to_{k+1|[1:k], t}(\vy|\vx_{1:k})\grad\ln \frac{p^\to_{k+1|[1:k], t}(\vy|\vx_{1:k})}{e^{-\|\vy\|^2/2}}\right)
        \end{aligned}
    \end{equation*}
    denotes its stationary distribution follows from the standard Gaussian $\mathcal{N}(\vzero,\mI)$ with density function $\varphi_1$.
    Due to Lemma~\ref{lem:strongly_lsi}, the $1$-strongly log-concave standard Gaussian satisfies LSI with a constant $1$, which means
    for any distribution with density function $p$, we have
    \begin{equation*}
        \KL{q}{\varphi_1}\le \FI{q}{\varphi_1} =  \frac{1}{2}\int q(\vx)\left\|\grad\ln \frac{q(\vy)}{\varphi_1(\vy)}\right\|^2\der\vx,
    \end{equation*}
    which implies
    \begin{equation*}
        \begin{aligned}
            \frac{\der \KL{p^\to_{k+1|[1:k], t}(\vy|\vx_{1:k})}{\varphi_1}}{\der t} & = -\int p^\to_{k+1|[1:k], t}(\vy|\vx_{1:k})\left\|\grad\ln \frac{p^\to_{k+1|[1:k], t}(\vy|\vx_{1:k})}{\varphi_1(\vy)}\right\|^2\der \vy\\
            & \le 2\KL{p^\to_{k+1|[1:k], t}(\vy|\vx_{1:k})}{\varphi_1}.
        \end{aligned}
    \end{equation*}
    According to Gr\"{o}nwall's inequality, it has
    \begin{equation}
        \label{ineq:ou_decrease}
        \begin{aligned}
            \KL{p^\to_{k+1|[1:k], t}(\cdot|\vx_{1:k})}{\varphi_1}\le e^{-2t}\cdot \KL{p^\to_{k+1|[1:k], 0}(\cdot|\vx_{1:k})}{\varphi_1}
        \end{aligned}
    \end{equation}
    where the RHS satisfies
    \begin{equation}
        \label{ineq:ou_init}
        \begin{aligned}
            &\KL{p^\to_{k+1|[1:k], 0}(\cdot|\vx_{1:k})}{\varphi_1} = \KL{p_{*,k+1|[1:k]}(\cdot|\vx_{[1:k]})}{\varphi_1} \le \FI{p_{*,k+1|[1:k]}(\cdot|\vx_{[1:k]})}{\varphi_1}\\
            & = \int p_{*, k+1|[1:k]}(\vy|\vx_{[1:k]})\cdot \left\|\grad \ln \frac{p_{*, k|[1:k]}(\vy|\vx_{[1:k]})}{\exp(-\|\vy\|^2/2)}\right\|\der \vy \\
            &\le 2Ld +\E_{p_{*,k+1|[1:k]}(\cdot|\vx_{[1:k]})}\left[\|\rvy\|^2\right].
        \end{aligned}
    \end{equation}
    The last inequality follows from the combination of Lemma~\ref{lem:init_smoothness_bound} and Lemma~\ref{lem:exp_score_bound}.
    Then, combining Eq.~\ref{ineq:ou_decrease} and Eq.~\ref{ineq:ou_init}, for any $t\in[0,T]$, we have
    \begin{equation*}
        \KL{p^\to_{k+1|[1:k], t}(\cdot|\vx_{1:k})}{\varphi_1}\le e^{-2t}\cdot \left(2Ld + \E_{p_{*,k+1|[1:k]}(\cdot|\vx_{[1:k]})}\left[\|\rvy\|^2\right] \right)
    \end{equation*}
    and the proof is completed.
\end{proof}
\begin{lemma}
    With the same notation in Lemma B.2, we have
    \label{lem:kl_to_drift_diff_girsanov}
    \begin{equation*}
        \mathbb{E}_{\rvy\sim \tilde{q}_{t_r}}\left[\KL{q^\gets_{t_{r+1}|t_r}(\cdot|\rvy)}{\hat{q}_{t_{r+1}|t_r}(\cdot|\rvy)}\right]\le \frac{1}{2} \int^{t_{r+1}}_{t_r} \mathbb{E}_{(\rvy,\rvy^\prime)\sim q^\gets_{t_r,t}, }\left[\left\|\grad\ln q^\gets_t(\rvy^\prime) - \vs_{\vtheta}(\rvy|T - t_r,\vz )\right\|^2\right] \der t
    \end{equation*}
\end{lemma}
\begin{proof}
    Let's consider the process \ref{sde:ideal_condi_reverse} and  \ref{sde:prac_condi_reverse}, by Lemma \ref{lemma:c.1} and \ref{lemma:c.2}, for $t \in (t_r,t_{r+1}]$ and $\rvy^\gets_{t_r} = \vy$, we have
\begin{align}
\label{eq:lemma3.1_5}
    &\frac{\partial}{\partial t} \KL{q^\gets_{t|t_r}(\cdot|\vy)}{\hat{q}_{t|t_r}(\cdot|\vy)} \\
    &= -\mathbb{E}_{\rvy^\prime \sim q^\gets_{t|t_r}(\cdot|\vy)}\Big\|\nabla \ln \frac{q^\gets_t(\rvy^\prime|\vy)}{\hat{q}_t(\rvy^\prime|\vy)}\Big\|^2 + \mathbb{E}_{\rvy^\prime\sim q^\gets_{t|t_r}(\cdot|\vy)} \left[ \left\langle \grad\ln q^\gets_t(\rvy^\prime) - \vs_{\vtheta}(\vy|T - t_r,\vz ), \nabla \ln \frac{q^\gets_{t|t_r}(\rvy^\prime|\vy)}{\hat{q}_{t|t_r}(\rvy^\prime|\vy)}  \right\rangle \right] \notag\\
    & \le \frac{1}{2}\mathbb{E}_{\rvy^\prime\sim q^\gets_{t|t_r}(\cdot|\vy)} \left[\left\|\grad\ln q^\gets_t(\rvy^\prime) - \vs_{\vtheta}(\vy|T - t_r,\vz )\right\|^2\right].
\end{align}
The last inequality holds by the fact that $\langle \rvw,\rvv \rangle \leq \frac{1}{2} \|\rvw\|^2+\frac{1}{2} \|\rvv\|^2$. Integrating both sides of Eq.~\ref{eq:lemma3.1_5} and utilizing Lemma \ref{lemma:c.2}, we obtain
\begin{align*}
    \KL{q^\gets_{t_{r+1}|t_r}(\cdot|\vy)}{\hat{q}_{t_{r+1}|t_r}(\cdot|\vy)} \le \frac{1}{2} \int^{t_{r+1}}_{t_r} \mathbb{E}_{\rvy^\prime\sim q^\gets_{t|t_r}(\cdot|\vy)} \left[\left\|\grad\ln q^\gets_t(\rvy^\prime) - \vs_{\vtheta}(\vy|T - t_r,\vz )\right\|^2\right] \der t.
\end{align*}
Then, integrating on the both sides w.r.t. $\tilde{q}_{t_r}$ yields
\begin{equation}
    \label{eq:lemma3.1_6}
    \begin{aligned}
        &\mathbb{E}_{\rvy\sim q^\gets_{t_r}}\left[\KL{q^\gets_{t_{r+1}|t_r}(\cdot|\rvy)}{\hat{q}_{t_{r+1}|t_r}(\cdot|\rvy)}\right]\\
        &\le \frac{1}{2} \int^{t_{r+1}}_{t_r} \mathbb{E}_{(\rvy,\rvy^\prime)\sim q^\gets_{t_r,t}, }\left[\left\|\grad\ln q^\gets_t(\rvy^\prime) - \vs_{\vtheta}(\rvy|T - t_r,\vz )\right\|^2\right] \der t.
    \end{aligned}
\end{equation}
\end{proof}

\begin{lemma}
\label{lem:discretization_error_each_k}
Assume $L \ge 1 $, for step sizesatisfies
\begin{align*}
     \eta_r \le \min\{1,\eta,\eta(T-t_{r+1})\},
\end{align*}
the discretization error in Eq.~\ref{ineq:kl_error_beyond_init_0} can be bounded as
\begin{equation*}
    \begin{aligned}
        &\sum_{r=0}^{R-1} \E_{\rvy^\gets_{[0:T]}\sim Q^\gets_{[0:T]}}\left[\int_{t_r}^{t_{r+1}} \left\|\grad\ln q_{T-t_r}(\rvy^\gets_{t_r})-  \grad\ln q_{T-t}(\rvy^\gets_t) \right\|^2 \der t \right]\lesssim    dL^2R\eta^2 + dT\eta+ \E_{p_{*,k+1|[1:k]}(\cdot|\vx_{[1:k]})}\left[\|\rvy\|^2\right]\eta.
    \end{aligned}
\end{equation*}
% where $\cdots$
\end{lemma}
\begin{proof}
Following \cite{benton2024nearly} and \cite{chen2023improved}, we divide our time period into three intervals: $[0, T-1]$, $(T-1, T-\delta]$, and $(T-\delta, T]$, and treat each interval separately. Here, $\delta$ is a positive constant, and we set $\delta \le \ln \frac{\sqrt{\frac{1}{4L^2} + 4} + \frac{1}{2L}}{2}$ to satisfy the condition of Lemma \ref{lemma:c.9}, which ensures the Lipschitz property of the score function, thereby allowing the discretization error to be bounded near the end of the data distribution.

Similar as \cite{benton2024nearly} and \cite{chen2023improved}, we consider the step size satisfying $\eta_r \leq \eta \min(1, T-t_{r+1})$. $\eta$ is the parameters for controlling the maximum step size. We also assume there is an index $M$ and $N$ with $t_M = T - 1$ and $t_N = T-\delta$ as the work \cite{benton2024nearly}. This assumption is purely for presentation clarity and our argument works similarly without it. To guarantee $\eta_r \leq \eta \min(1, T-t_{r+1})$, following the setting \cite{benton2024nearly}, we first assume $M = \alpha R$ and $N = \beta R$ for convenience and further have
\begin{itemize}
    \item At the time interval $[0, T-1]$ where $t_0, \cdots, t_M$ are linearly spaced on $[0, T-1]$, $\eta \ge \frac{T-1}{M} = \frac{T-1}{\alpha R}$ which can be satisfied by taking $\eta = \Omega \Big(\frac{T}{R}\Big)$.
    \item At the time interval $[T-1, T-\delta]$ where $t_{M+1}, \cdots, t_{N}$ are exponentially decaying from $\frac{\eta}{1+\eta}$ to $\frac{\eta}{(1+\eta)^{N-M}}$. This condition can be satisfied with $\eta \ge (\frac{1}{\delta})^{\frac{1}{^{N-M}}}-1$.
    Assume $(\beta-\alpha)R \ge \log \frac{1}{\delta}$ and $e^{\frac{\ln \frac{1}{\delta}}{(\beta-\alpha)R}} \le 1+(e-1)\frac{\ln \frac{1}{\delta}}{(\beta-\alpha)R}$, we can take $\eta = \Omega\Big(\frac{\ln \frac{1}{\delta}}{R}\Big)$.
    \item At the time interval $(T-\delta,T]$, the Lipschitz property of the score function at this time interval can be satisfied if we take $\delta \in \Big(0, \ln \frac{\sqrt{\frac{1}{4L^2} + 4} + \frac{1}{2L}}{2}\Big]$.
\end{itemize}
Therefore, if we choose the constant $c = \ln \frac{\sqrt{\frac{1}{4L^2} + 4} + \frac{1}{2L}}{2}$ and $\eta = \Theta\Big(\frac{T+\ln \frac{1}{c}}{R}\Big)$, we can have $\eta_r \leq \eta \min(1, T-t_{r+1})$ for $r=0,\cdots,R-1$ at time interval $[0,T]$. Then, we split the discretization error into three parts as
% Such $N$ and $M$ exist such that $s_{T-\delta} := \sum^{N-1}_{r=M} \eta_r = 1-\delta$ as $\lim_{\delta \to 0} s_{T-\delta} \to 1$ with the typical choice of step sizes $\eta_r \leq \eta \min(1, T-t_{r+1})$. 
    \begin{align}
    \label{eq:term-1-2-bound}
         &\sum_{r=0}^{R-1} \E_{\rvy^\gets_{[0:T]}\sim Q^\gets_{[0:T]}}\left[\int_{t_r}^{t_{r+1}} \left\|\grad\ln q^\gets_{t_r}(\rvy^\gets_{t_r})-  \grad\ln q^\gets_{t}(\rvy^\gets_t) \right\|^2 \der t \right] \notag\\
         &= \underbrace{\sum_{r=0}^{M-1} \E_{\rvy^\gets_{[0:T]}\sim Q^\gets_{[0:T]}}\left[\int_{t_r}^{t_{r+1}} \left\|\grad\ln q^\gets_{t_r}(\ry^\gets_{t_r})-  \grad\ln q^\gets_{t}(\rvy^\gets_t) \right\|^2 \der t \right]}_{\text{term I}}\notag\\
         &+ \underbrace{\sum_{r=M}^{N-1} \E_{\rvy^\gets_{[0:T]}\sim Q^\gets_{[0:T]}}\left[\int_{t_r}^{t_{r+1}} \left\|\grad\ln q^\gets_{t_r}(\rvy^\gets_{t_r})-  \grad\ln q^\gets_{t}(\rvy^\gets_t) \right\|^2 \der t \right]}_{\text{term II}}\notag\\
         &+\underbrace{\sum_{r=N}^{R-1} \E_{\rvy^\gets_{[0:T]}\sim Q^\gets_{[0:T]}}\left[\int_{t_r}^{t_{r+1}} \left\|\grad\ln q^\gets_{t_r}(\rvy^\gets_{t_r})-  \grad\ln q^\gets_{t}(\rvy^\gets_t) \right\|^2 \der t \right]}_{\text{term III}}.
    \end{align}

\paragraph{Term I and Term II.} Following the proof of Lemma 2 in \cite{benton2024nearly}. We first define our target as
\begin{align*}
    E_{s,t}:=\E_{\rvy^\gets_{[0:T]}\sim Q^\gets_{[0:T]}}\left[\left\|\grad\ln q^\gets_{t}(\rvy^\gets_{t})-  \grad\ln q^\gets_{s}(\rvy^\gets_s) \right\|^2 \right]
\end{align*}
where $0 \le s \le t \le T$.

By Lemma \ref{lemma3.4}, we obtain
\begin{align*}
    \frac{\der E_{s,t}}{\der t} &=\frac{\der}{\der t}\E_{Q^\gets}[\|\grad\ln q^\gets_{t}(\rvy^\gets_{t})\|^2] - 2 \frac{\der}{\der t}\E_{Q^\gets}[\grad\ln q^\gets_{t}(\rvy^\gets_{t})\cdot  \grad\ln \rvy^\gets_{s}(\rvy^\gets_s)]\\
    &= 2\mathbb{E}_{Q^\gets}[\|\nabla^2 \ln q^\gets_{t}(\rvy^\gets_{t})\|^2_F]-2\mathbb{E}_{Q^\gets}[\|\nabla \ln q^\gets_{t}(\rvy^\gets_{t})\|^2]+2\E_{Q^\gets}[|\grad\ln q^\gets_{t}(\tilde{\rvy}_{t})\cdot  \grad\ln q^\gets_{s}(\rvy^\gets_s)]
\end{align*}
Using Young’s inequality, we further have
\begin{align}
\label{eq:term-1-2-bound-1}
    \frac{\der E_{s,t}}{\der t} &= 2\mathbb{E}_{Q^\gets}[\|\nabla^2 \ln q^\gets_{t}(\tilde{\rvy}_{t})\|^2_F]-2\mathbb{E}_{Q^\gets}[\|\nabla \ln q^\gets_{t}(\rvy^\gets_{t})\|^2]+2\E_{Q^\gets}[|\grad\ln q^\gets_{t}(\rvy^\gets_{t})\cdot  \grad\ln q^\gets_{s}(\rvy^\gets_s)]\notag\\
    & \le \underbrace{\mathbb{E}_{Q^\gets}[\|\nabla \ln q^\gets_{s}(\rvy^\gets_{s})\|^2]}_{\text{Term 1}} + \underbrace{2\mathbb{E}_{Q^\gets}[\|\nabla^2 \ln q^\gets_{t}(\rvy^\gets_{t})\|^2_F]}_{\text{Term 2}}
\end{align}
For Term 1, by Lemma \ref{lemma3.5}, we can further bound Eq.~\ref{eq:term-1-2-bound-1} as
\begin{align*}
    &\mathbb{E}_{{\rvy}_{[0:T]}\sim {Q}_{[0:T]}}[\|\nabla \ln q_{T-s}({\rvy}_{T-s})\|^2] \\
    &= \sigma^{-4}_{T-s}\E_{Q}[\|{\rvy}_{T-s}\|^2] - 2\sigma^{-4}_{T-s} e^{s-T}\E_{Q}[{\rvy}_{T-s} \cdot \boldsymbol{\mu}_{T-s}] + e^{-2({T-s})}\sigma^{-4}_{T-s}\E_{Q}[\|\boldsymbol{\mu}_{T-s}\|^2]\\
    &=\sigma^{-4}_{T-s}\E_{Q}[\|e^{s-T}{\rvy}_0+\sigma_{T-s} \rvz\|^2]- 2\sigma^{-4}_{T-s} e^{s-T}\E_{Q}[{\rvy}_0 \cdot \E[{\rvy}_{T-s} |{\rvy}_0 ]]  + e^{2(s-T)}\sigma^{-4}_{T-s}\E_{Q}[\|\boldsymbol{\mu}_{T-s}\|^2]\\
    &\le e^{2(s-T)}\sigma^{-4}_{T-s}\E_{\rvy \sim q_0}\left[\|\rvy\|^2\right]+\sigma^{-2}_{T-s}d -2\sigma^{-4}_{T-s} e^{2(s-T)}\E_{\rvy \sim q_0}\left[\|\rvy\|^2\right]+ e^{2(s-T)}\sigma^{-4}_{T-s}\E_{Q}[\|\boldsymbol{\mu}_{T-s}\|^2]\\
    & \overset{(a)}{=}\sigma^{-2}_{T-s}d +e^{2(s-T)}\sigma^{-4}_{T-s}\E_{\rvy \sim q_0}\left[\|\rvy\|^2\right]-2\sigma^{-4}_{T-s} e^{2(s-T)}\E_{\rvy \sim q_0}\left[\|\rvy\|^2\right]+ e^{2(s-T)}\sigma^{-4}_{T-s}[\E_{\rvy \sim q_0}\left[\|\rvy\|^2\right]-\E_{Q}[\text{Tr}(\mathbf{\Sigma}_{T-s})]]\\
    & \le \sigma^{-2}_{T-s}d 
\end{align*}
where $\rvz$ is the standard Gaussian noise and Step (a) holds by the fact that $\text{Tr}(\mathbf{\Sigma}_t) = \E[\|\rvy_0\|^2|\rvy_t] -\E^2[\rvy_0|\rvy_t]= \E[\|\rvy_0\|^2|\rvy_t] -\|\boldsymbol{\mu}_t\|^2 $ and $\E_{\tilde{Q}}[\text{Tr}(\mathbf{\Sigma}_t)] = \E_{p_{*,k+1|[1:k]}(\cdot|\vx_{[1:k]})}\left[\|\rvy\|^2\right]-\E_{Q^\gets}[\|\boldsymbol{\mu}_t\|^2] $. That is 
\begin{align*}
    \mathbb{E}_{Q^\gets}[\|\nabla \ln q^\gets_{s}(\rvy^\gets_{s})\|^2] = \mathbb{E}_{Q}[\|\nabla \ln q_{T-s}({\rvy}_{T-s})\|^2]  \le \sigma^{-2}_{T-s}d
\end{align*}

For Term 2, by Lemma \ref{lemma3.5}, we expand Eq.~\ref{eq:term-1-2-bound-1} as
\begin{align}
\label{eq:term-1-2-bound-f}
    &\mathbb{E}_{Q^\gets}[\|\nabla^2 \ln q^\gets_{t}(\tilde{\rvy}_{t})\|^2_F] \\
    &= \sigma^{-4}_{T-t}d- 2\sigma^{-6}_{T-t}e^{-2(T-t)}\mathbb{E}_{Q^\gets}[\text{Tr}(\mathbf{\Sigma}_{T-t})]+e^{-4(T-t)}\sigma^{-8}_{T-t}\mathbb{E}_{Q^\gets}[\text{Tr}(\mathbf{\Sigma}^\top_{T-t}\mathbf{\Sigma}_{T-t})]\notag\\
    &=\sigma^{-4}_{T-t}d- 2\sigma^{-6}_{T-t}e^{-2(T-t)}\mathbb{E}_{Q^\gets}[\text{Tr}(\mathbf{\Sigma}_{T-t})]-\frac{1}{2}e^{-2(T-t)}\sigma^{-4}_{T-t}
    \frac{\der}{\der t} \mathbb{E}_{Q^\gets}[\text{Tr}(\mathbf{\Sigma}_{T-t})]\notag\\
    & \le \sigma^{-4}_{T-t}d-
    \frac{1}{2}\frac{\der}{\der t} \sigma^{-4}_{T-t}\mathbb{E}_{Q^\gets}[\text{Tr}(\mathbf{\Sigma}_{T-t})]
\end{align}
The second equality holds by Lemma \ref{lemma:4.5}. The third inequality holds by the fact that $e^{-2(T-t)} \le 1$ and $\frac{\der}{\der t} \mathbb{E}_{\tilde{Q}}[\text{Tr}(\mathbf{\Sigma}_{T-t})] \le 0$ inferred by Lemma \ref{lemma:4.5}.

So, we can have the upper bound combined with two parts $E^{(1)}_{s,t}$ and $E^{(2)}_{s,t}$ as follows
\begin{align*}
    \frac{\der E_{s,t}}{\der t}  \le \underbrace{\sigma^{-2}_{T-s}d + 2\sigma^{-4}_{T-t}d }_{E^{(1)}_{s,t}}\underbrace{ -\frac{\der}{\der r} (\sigma^{-4}_{T-r}\E_{Q^\gets}[\text{Tr}(\mathbf{\Sigma}_{T-r})])|_{r=t}}_{E^{(2)}_{s,t}}
\end{align*}
Therefore, when $s=t_r$, by Lemma \ref{lemma3.6}, we have the upper bound of summation of term I and term II in Eq.~\ref{eq:term-1-2-bound} as
\begin{align}
\label{eq:term-1-2-bound-2}
\sum^{N-1}_{r=0} \int^{t_{r+1}}_{t_r} E_{t_r,t} &\le \sum^{N-1}_{r=0} \int^{t_{r+1}}_{t_r} \int^t_{t_r} E^{(1)}_{t_r,r} \der r + \sum^{N-1}_{r=0} \int^{t_{r+1}}_{t_r} \int^t_{t_r} E^{(2)}_{t_r,r} \der r \notag\\
& \lesssim dN \eta^2 + dT\eta+ \E_{p_{*,k+1|[1:k]}(\cdot|\vx_{[1:k]})}\left[\|\rvy\|^2\right]\eta.
\end{align}

Finally, we have
\begin{align}
\label{eq:term12}
    \text{Term I + Term II}&= \sum^{N-1}_{r=0}\E_{\rvy^\gets_{[0:T]}\sim Q^\gets_{[0:T]}}\left[\int^{t_{r+1}}_{t_r}\left\|\grad\ln q^\gets_{t_r}(\rvy^\gets_{t_r})-  \grad\ln q^\gets_{t}(\rvy^\gets_t) \right\|^2 \der t \right] \notag\\
    &{\lesssim } dN \eta^2 + dT\eta+ \E_{p_{*,k+1|[1:k]}(\cdot|\vx_{[1:k]})}\left[\|\rvy\|^2\right]\eta.
\end{align}

\paragraph{Term III.}Since at this time interval, the corresponding forward process satisfies $\sigma_{0,t}^2 \leq \frac{\alpha_{0,t}}{2L}$, for $t \in (t_r,t_{r+1}]$, by Lemma \ref{lemma:c.6}, we have
\begin{align*}
% \label{eq:12}
    &\E_Q \left[\left\|\grad\ln q_{T-t}({\rvy}_{T-t})-  \grad\ln {q}_{T-t_r}({\rvy}_{T-t}) \right\|^2 \right] \\
    &\overset{(a)}{\le }4\mathbb{E}_Q  \left\| \nabla \log {q}_{T-t}({\rvy}_{T-t}) - \nabla \log {q}_{T-t}(\alpha_{T-t,T-t_r}^{-1} {\rvy}_{T-t_r}) \right\|^2 + 2 \mathbb{E}_Q  \left\| \nabla \log {q}_{T-t}({\rvy}_{T-t}) \right\|^2 \left( 1 - \alpha_{T-t,T-t_r}^{-1} \right)^2\\
    & \overset{(b)}{\le } 16L^2 \alpha^{-2}_{T-t} \E_Q \| {\rvy}_{T-t} - \alpha_{T-t,T-t_r}^{-1} {\rvy}_{T-t_r} \|^2+ 2 \E_Q  \left\| \nabla \log {q}_{T-t}({\rvy}_{T-t}) \right\|^2 \left( 1 - \alpha_{T-t,T-t_r}^{-1} \right)^2\\
    &\overset{}{\le } 16dL^2(e^{2(t-t_r)}-1)+ 2\E_Q  \left\| \nabla \log {q}_{T-t}({\rvy}_{T-t}) \right\|^2 \left( 1 - \alpha_{T-t,T-t_r}^{-1} \right)^2 \\
    &\overset{(c)}{\le } 32e^2dL^2(t-t_r)+ 2 \E_Q \left\| \nabla \log {q}_{T-t}({\rvy}_{T-t}) \right\|^2 (t-t_r)^2
\end{align*}
The first inequality (a) holds by Lemma \ref{lemma:c.6} and the second inequality (b) holds by Lemma \ref{lemma:c.9}.
Step $(c)$ holds by the assumption that the step size at the given time interval $t-t_r \leq \eta_r \le 1$. 
Since $ \nabla \log {q}_{t}$ is 2L$\alpha^{-1}_t$-Lipschitz, we obtain
\begin{align}
\label{eq:13}
    &\E_Q  \left\| \nabla \log {q}_{T-t}({\rvy}_{T-t}) \right\|^2 \\
    & = \int {q}_{T-t}({\rvy}_{T-t}) \left\| \nabla \log {q}_{T-t}({\rvy}_{T-t}) \right\|^2 d{\rvy}_{T-t} \notag\\
    & = \int \langle \nabla {q}_{T-t}({\rvy}_{T-t}),  \nabla \log {q}_{T-t}({\rvy}_{T-t}) \rangle d{\rvy}_{T-t} \notag\\
    & = \int {q}_{T-t}({\rvy}_{T-t}) \Delta  \log {q}_{T-t}({\rvy}_{T-t}) d{\rvy}_{T-t}   \notag\\
    & \le 2dL.
\end{align}
That is, we have
\begin{align}
\label{eq:12}
    &\E_{Q^\gets}\left[\left\|\grad\ln q^\gets_{t}(\rvy^\gets_{t})-  \grad\ln q^\gets_{t_r}(\rvy^\gets_{t_r}) \right\|^2 \right] \le   32e^2dL^2(t-t_r)+ 4dL (t-t_r)^2.
\end{align}
Then, plug Eq.~\ref{eq:13} into Eq.~\ref{eq:12}, for step sizesatisfies
\begin{align*}
     \eta_r \le \min\{1,\eta,\eta(T-t_{r+1})\},
\end{align*}
we have
\begin{align}
\label{eq:term3}
    \text{Term III}&= \sum^{R-1}_{r=N}\E_{\tilde{\rvy}_{[0:T]}\sim \tilde{Q}_{[0:T]}}\left[\int^{t_{r+1}}_{t_r}\left\|\grad\ln \tilde{q}_{t_r}(\tilde{\rvy}_{t_r})-  \grad\ln \tilde{q}_{t}(\tilde{\rvy}_t) \right\|^2 \der t \right] \notag\\
    &{\lesssim } dL^2 \sum^{R-1}_{r=N}h^2_k \lesssim dL^2R\eta^2.
\end{align}
Finally, combine Eq.~\ref{eq:term12} with \ref{eq:term3}, assume $L \ge 1 $, the discretization error can be bounded as
\begin{align*}
    \sum_{r=0}^{R-1} \E_{\rvy^\gets_{[0:T]}\sim Q^\gets_{[0:T]}}\left[\int_{t_r}^{t_{r+1}} \left\|\grad\ln q^\gets_{t_r}(\rvy^\gets_{t_r})-  \grad\ln q^\gets_{t}(\rvy^\gets_t) \right\|^2 \der t \right] \lesssim    dL^2R\eta^2 + dT\eta+ \E_{p_{*,k+1|[1:k]}(\cdot|\vx_{[1:k]})}\left[\|\rvy\|^2\right]\eta.
\end{align*}
\end{proof}

\begin{lemma}
\label{lemma3.3}
If  $\rvy^\gets_{t}$ is the solution to SDE \cref{sde:ideal_condi_reverse}, then for all $t \in [0,T]$, we have
\begin{align}
    \der(\nabla \ln q^\gets_{t}(\rvy^\gets_{t})) = \sqrt{2}\nabla^2 \ln q^\gets_{t}(\rvy^\gets_{t})\der\mB_t - \nabla  \log q^\gets_t(\rvy^\gets_{t}) \der t
\end{align}
\end{lemma}
\begin{proof}
    Since $\grad\ln q^\gets_{t_r}(\rvy^\gets_{t_r})$ in process \ref{sde:ideal_condi_reverse} is smooth for $t \in [0,T]$, by  It\^o lemma, we expand
\begin{align}
\label{eq:3.3-15}
    &\der(\nabla \ln q^\gets_{t}(\rvy^\gets_{t})) \notag\\
    &= \nabla^2 \ln q^\gets_{t}(\rvy^\gets_{t}) d\rvy^\gets_t + \frac{\der(\nabla \ln q^\gets_{t}(\rvy^\gets_{t}))}{\der t}\der t + \frac{1}{2}\cdot(\sqrt{2})^2 \Delta (\nabla \ln q^\gets_{t}(\rvy^\gets_{t})) \der t \notag\\
    &= \Big\{\nabla^2 \ln q^\gets_{t}(\rvy^\gets_{t})\cdot (\left(\rvy^\gets_t +2\grad\ln q_{T-t}(\rvy^\gets_t)\right)+\Delta (\nabla \ln q^\gets_{t}(\rvy^\gets_{t}))\Big\}\der t + \sqrt{2}\nabla^2 \ln q^\gets_{t}(\rvy^\gets_{t})\der\mB_t  + \der(\nabla \ln q^\gets_{t}(\rvy^\gets_{t})).
\end{align}
By Fokker–Planck Eq.~for the forward process \ref{sde:ideal_condi_forward}, we have
\begin{align*}
    \frac{\der q_{t}({\rvy}_{t})}{\der t} = \nabla \cdot ({\rvy}_{t} q_{t}({\rvy}_{t})) + \Delta {q}_{t}({\rvy}_{t}).
\end{align*}
We define $f(\rvy_{t}) = \log q_t(\rvy_{t})$ and $q_t(\rvy_{t}) = \exp(f)$, thus
\begin{align*}
     \frac{\der q_{t}({\rvy}_{t})}{\der  t}&= q_t(\rvy_{t})\frac{\der  f({\rvy}_{t})}{\der  t}\\&= \nabla \cdot ({\rvy}_{t} q_{t}({\rvy}_{t})) + \Delta {q}_{t}({\rvy}_{t})\\
     & = q_t(\rvy_{t})(d+{\rvy}_{t}\nabla f(\rvy_{t})) +  q_t(\rvy_{t})(\|\nabla f(\rvy_{t})\|^2+\Delta f(\rvy_{t}))\\
     & = q_t(\rvy_{t})(d+{\rvy}_{t}\nabla  \log q_t(\rvy_{t}) +  q_t(\rvy_{t})(\|\nabla  \log q_t(\rvy_{t})\|^2+\Delta  \log q_t(\rvy_{t}))
\end{align*}
So, back to the reverse process \ref{sde:ideal_condi_reverse} where $t \mapsto T - t$
\begin{align*}
     &\der  \nabla \log q_{T-t}(\rvy_{T-t}) \\
     & = -\nabla \Big(d+{\rvy}_{T-t} \cdot \nabla \log q_{T-t}(\rvy_{T-t}) +  \|\nabla  \log q_{T-t}(\rvy_{T-t})\|^2+\Delta  \log q_{T-t}(\rvy_{T-t})\Big) \notag\\
     &=-\Big(\nabla  \log q_{T-t}(\rvy_{T-t}) + {\rvy}_{T-t} \cdot \nabla^2 \log q_{T-t}(\rvy_{T-t}) + 2 \nabla  \log q_{T-t}(\rvy_{T-t}) \cdot \nabla^2  \log q_{T-t}(\rvy_{T-t}) \\
     &\quad + \nabla (\Delta  \log q_{T-t}(\rvy_{T-t}))\Big)\der t.
\end{align*}
It means, for the reverse process, we have
\begin{align}
\label{eq:3.3-16}
     \der  \nabla \log q^\gets_{t}(\rvy^\gets_{t})
     &=-\Big(\nabla  \log q^\gets_t(\rvy^\gets_{t}) + \rvy^\gets_{t} \cdot \nabla^2 \log q^\gets_t(\rvy^\gets_{t}) + 2 \nabla  q^\gets_t(\rvy^\gets_{t}) \cdot \nabla^2  q^\gets_t(\rvy^\gets_{t}) + \nabla (\Delta  q^\gets_t(\rvy^\gets_{t}))\Big)\der t.
\end{align}
Plugging Eq.~\ref{eq:3.3-16} into \ref{eq:3.3-15}, we have
\begin{align*}
    \der(\nabla \ln q^\gets_{t}(\rvy^\gets_{t})) = \sqrt{2}\nabla^2 \ln q^\gets_{t}(\rvy^\gets_{t})\der\mB_t - \nabla  \log q^\gets_t(\rvy^\gets_{t}) \der t
\end{align*}
The proof is complete.
\end{proof}

\begin{lemma}
\label{lemma3.4}
If  $\rvy^\gets_{t}$ is the solution to SDE \cref{sde:ideal_condi_reverse}, then for all $t \in [0,T]$, we have
\begin{align*}
2\mathbb{E}_{Q^\gets}[\|\nabla \ln q^\gets_{t}(\rvy^\gets_{t})\|^2]+\frac{\der \mathbb{E}_{Q^\gets}[\|\nabla \ln q^\gets_{t}(\rvy^\gets_{t})\|^2]}{\der t}  = 2\mathbb{E}_{Q^\gets}[\|\nabla^2 \ln q^\gets_{t}(\rvy^\gets_{t})\|^2_F].
\end{align*}
\end{lemma}
\begin{proof}
By Lemma \ref{lemma3.3}, we have
\begin{align*}
    \der(e^t\nabla \ln q^\gets_{t}(\rvy^\gets_{t})) = \sqrt{2}e^t\nabla^2 \ln q^\gets_{t}(\rvy^\gets_{t})\der\mB_t.
\end{align*}
By Eq.~\ref{eq:term-1-2-bound-f}, we know $\int^2_s \E_{Q^\gets}[\nabla^2 \ln q^\gets_{t}(\rvy^\gets_{t})] \le \infty$ which is square integrability and then applying It\^o isometry, for $0\le s \le t \le T$, we obtain
\begin{align}
\label{eq:term-1-2-1}
    \frac{\der}{\der t} \mathbb{E}_{Q^\gets} [\|e^t\nabla \ln q^\gets_{t}(\rvy^\gets_{t}) - e^s\nabla \ln q^\gets_{s}(\rvy^\gets_{s})\|^2] = 2e^{2t} \mathbb{E}_{Q^\gets}[\|\nabla^2 \ln q^\gets_{t}(\rvy^\gets_{t})\|^2_F].
\end{align}
where $\|\mA\|^2_F = \text{Tr}(\mA^\top \mA)$ is squared Frobenius norm of a matrix. We further expand Eq.~\ref{eq:term-1-2-1} as 
\begin{align}
\label{eq:term-1-2-2} 
&2\mathbb{E}_{Q^\gets}[\|\nabla \ln q^\gets_{t}(\rvy^\gets_{t})\|^2]+\frac{\der \mathbb{E}_{\tilde{Q}}[\|\nabla \ln q^\gets_{t}(\rvy^\gets_{t})\|^2]}{\der t} -2e^{s-t} \mathbb{E}_{Q^\gets}[\nabla \ln q^\gets_{t}(\rvy^\gets_{t}) \cdot \nabla \ln q^\gets_{s}(\rvy^\gets_{s})] \notag\\
&-2e^{s-t} \frac{\der}{\der t} \mathbb{E}_{Q^\gets}[\nabla \ln q^\gets_{t}(\rvy^\gets_{t}) \cdot \nabla \ln q^\gets_{s}(\rvy^\gets_{s})] = 2\mathbb{E}_{Q^\gets}[\|\nabla^2 \ln q^\gets_{t}(\rvy^\gets_{t})\|^2_F].
\end{align}
Given any $s$, we have $t \ge s$, 
\begin{align*}
      \der(\nabla \ln q^\gets_{s}(\rvy^\gets_{s}) \cdot \nabla \ln q^\gets_{t}(\rvy^\gets_{t})) = \sqrt{2}\nabla \ln q^\gets_{s}(\rvy^\gets_{s}) \cdot \nabla^2 \ln q^\gets_{t}(\rvy^\gets_{t})\der\mB_t -\nabla \ln q^\gets_{s}(\rvy^\gets_{s}) \cdot \nabla  \log q^\gets_t(\rvy^\gets_{t}) \der t.
\end{align*}
And by Eq.~\ref{eq:term-1-2-bound-f}, we know $\int^2_s \E_{Q^\gets}[\nabla^2 \ln q^\gets_{t}(\rvy^\gets_{t})] \le \infty$ which is square integrability, applying It\^o isometry take the expectation at both side, we take
\begin{align*}
    \frac{\der}{\der t}\mathbb E_{Q^\gets}[{\nabla \ln q^\gets_{s}(\rvy^\gets_{s}) \cdot \nabla \ln q^\gets_{t}(\rvy^\gets_{t})}]= - \mathbb E_{Q^\gets}[\nabla \ln q^\gets_{s}(\rvy^\gets_{s}) \cdot \nabla  \log q^\gets_t(\rvy^\gets_{t})].
\end{align*}
Therefore, Eq.~\ref{eq:term-1-2-2} can be simplified as
\begin{align*}
2\mathbb{E}_{Q^\gets}[\|\nabla \ln q^\gets_{t}(\rvy^\gets_{t})\|^2]+\frac{\der \mathbb{E}_{\tilde{Q}}[\|\nabla \ln q^\gets_{t}(\rvy^\gets_{t})\|^2]}{\der t}  = 2\mathbb{E}_{Q^\gets}[\|\nabla^2 \ln q^\gets_{t}(\rvy^\gets_{t})\|^2_F].
\end{align*}
\end{proof}

\begin{lemma}
\label{lemma3.5}
For all $t>0$, we have
\begin{align*}
      &\nabla \log q^\gets_{t}(\rvy^\gets_t) = \nabla \log q_{T-t}(\rvy_{T-t}) = \sigma_{T-t}^{-2}\rvy^\gets_t - e^{t-T}\sigma_{T-t}^{-2}\boldsymbol{\mu}_{T-t},\\
      &\nabla^2 \log q^\gets_{t}(\rvy^\gets_t) = \nabla^2 \log q_{T-t}(\rvy_{T-t})  =  -\sigma_{T-t}^{-2}I + e^{-2(T-t)} \sigma_{T-t}^{-4} \mathbf{\Sigma}_{T-t}.
\end{align*}
\end{lemma}
\begin{proof}
For the forward process \ref{sde:ideal_condi_forward}, we have
\begin{align*}
    \nabla \ln q_{t}({\rvy}_{t})  &= \frac{1}{{q}_{t}({\rvy}_{t})} \int \nabla \log {q}_{t|0}({\rvy}_{t}|{\rvy}_{0}) {q}_{0,t}({\rvy}_{0},{\rvy}_{t}) d{\rvy}_{0}\\
    &=-\E_{q_{0|t}}[\sigma_t^{-2}(\rvy_t - \alpha_{0,t}\rvy_0)]\\
    &=-\sigma_t^{-2}\rvy_t + e^{-t}\sigma_t^{-2}\boldsymbol{\mu}_t
\end{align*}
where the last equality holds by the forward process $q_{t|0}(\rvy_t|\rvy_0) = \mathcal{N}(\rvy_t;\alpha_{0,t}\rvy_0, {\sigma}_t^2 I)$ and $\nabla \log q_{t|0}(\rvy_t|\rvy_0) = -\sigma_t^{-2}(\rvy_t - \alpha_{0,t}\rvy_0)$.

Then, the second-order derivative is
  \begin{align*}
    &\nabla^2 \log q_t(\rvy_t) \\
    &= \frac{1}{q_t(\rvy_t)}\int \nabla^2 \log q_{t|0}(\rvy_t|\rvy_0)q_{0,t}(\rvy_0,\rvy_t) d\rvy_0 + \frac{1}{q_t(\rvy_t)}\int (\nabla \log q_{t|0}(\rvy_t|\rvy_0)(\nabla \log q_{t|0}(\rvy_t|\rvy_0))^\top q_{0,t}(\rvy_0,\rvy_t) d\rvy_0 \\
    &- \frac{1}{q_t(\rvy_t)^2}\left(\int \nabla \log q_{t|0}(\rvy_t|\rvy_0) q_{0,t}(\rvy_0,\rvy_t) d\rvy_0 \right)\left(\int \nabla \log q_{t|0}(\rvy_t|\rvy_0) q_{0,t}(\rvy_0,\rvy_t) d\rvy_0 \right)^\top \\
    &= -\frac{1}{\sigma_t^2} I + \mathbb{E}_{q_{0|t}(\cdot|\rvy_t)}[\sigma_t^{-4}(\rvy_t - \rvy_0 e^{-t})(\rvy_t - \rvy_0e^{-t})^\top] \\
    &- \mathbb{E}_{q_{0|t}(\cdot|\rvy_t)}[-\sigma_t^{-2}(\rvy_t - \rvy_0e^{-t})] \mathbb{E}_{q_{0|t}(\cdot|\rvy_t)}[-\sigma_t^{-2}(\rvy_t - \rvy_0e^{-t})]^\top \\
    &= -\sigma_t^{-2}I + \sigma_t^{-4} \text{Cov}_{q_{0|t}(\cdot|\rvy_t)}(\rvy_t - \rvy_0e^{-t})\\
    &= -\sigma_t^{-2}I + e^{-2t} \sigma_t^{-4} \mathbf{\Sigma}_t.
    \end{align*}
Therefore,
\begin{align*}
      &\nabla \log q^\gets_{t}(\rvy^\gets_t) = \nabla \log q_{T-t}(\rvy_{T-t}) = \sigma_{T-t}^{-2}\tilde{\rvy}_t - e^{t-T}\sigma_{T-t}^{-2}\boldsymbol{\mu}_{T-t},\\
      &\nabla^2 \log q^\gets_{t}(\rvy^\gets_t) = \nabla^2 \log q_{T-t}(\rvy_{T-t})  =  -\sigma_{T-t}^{-2}I + e^{-2(T-t)} \sigma_{T-t}^{-4} \mathbf{\Sigma}_{T-t}.
\end{align*}
\end{proof}

\begin{lemma}
\label{lemma3.6}
Define
\begin{align*}
      \frac{\der E_{s,t}}{\der t}  \le \underbrace{\sigma^{-2}_{T-s}d + 2\sigma^{-4}_{T-t}d }_{E^{(1)}_{s,t}}\underbrace{ -\frac{\der}{\der r} (\sigma^{-4}_{T-r}\E_{\tilde{Q}}[\text{Tr}(\mathbf{\Sigma}_{T-r})])|_{r=t}}_{E^{(2)}_{s,t}}.
\end{align*}
The error terms $E^{(1)}_{s,t}$ and $E^{(2)}_{s,t}$ satisfies
\begin{align*}
    \sum^{N-1}_{k=M} \int^{t_{r+1}}_{t_r} \int^t_{t_r} E^{(1)}_{t_r,r} \der r \der t \lesssim d N\eta^2,
\end{align*}
\begin{align*}
    \sum^{N-1}_{r=0} \int^{t_{r+1}}_{t_r} \int^t_{t_r} E^{(2)}_{t_r,r} \der r \der t& \lesssim \eta \E_{p_{*,k+1|[1:k]}(\cdot|\vx_{[1:k]})}\left[\|\rvy\|^2\right]+ d N\eta^2
\end{align*}
\end{lemma}
\begin{proof}
We first define
\begin{align*}
      \frac{\der E_{s,t}}{\der t}  \le \underbrace{\sigma^{-2}_{T-s}d + 2\sigma^{-4}_{T-t}d }_{E^{(1)}_{s,t}}\underbrace{ -\frac{\der}{\der r} (\sigma^{-4}_{T-r}\E_{\tilde{Q}}[\text{Tr}(\mathbf{\Sigma}_{T-r})])|_{r=t}}_{E^{(2)}_{s,t}}
\end{align*}
For $E^{(1)}_{s,t}$, consider the time interval $s,t \in [0,T-1]$ and $\sigma^2_{T-s} \ge \sigma^2_{T-t} = 1-e^{-2(T-t)} \ge \frac{4}{5}$, we derive
\begin{align*}
    \sum^{M-1}_{k=0} \int^{t_{r+1}}_{t_r} \int^t_{t_r} E^{(1)}_{t_r,r} \der r \der t \lesssim d\sum^{k=M}_{k=0}h^2_k \lesssim \eta d T
\end{align*}
Consider the time interval $(T-1,T-\delta]$, $\frac{1}{5} (T-t)\le \sigma^2_{T-t} \le \sigma^2_{T-s} \le 2(T-s)$,  we derive 
\begin{align*}
    &\sum^{N-1}_{k=M} \int^{t_{r+1}}_{t_r} \int^t_{t_r} E^{(1)}_{t_r,r} \der r \der t \\
    &\lesssim  d \sum^{N-1}_{k=M}\int^{t_{r+1}}_{t_r} \int^t_{t_r} (T-r)^{-2} \der r \\
    &\lesssim  d \sum^{N-1}_{k=M} \frac{h^2_k}{(T-t_{r+1})^2} \\
    &\lesssim d N\eta^2.
\end{align*}
The last inequality holds by the setting $\eta_r \le \eta \min(1,T-t_{r+1})$.

For $E^{(2)}_{s,t}$, we first have
\begin{align*}
    \sum^{N-1}_{r=0} \int^{t_{r+1}}_{t_r} \int^t_{t_r} E^{(2)}_{t_r,r} \der r \der t& = \sum^{N-1}_{r=0} \int^{t_{r+1}}_{t_r} \Big(\sigma^{-4}_{T-t_r}\E_{Q^\gets}[\text{Tr}(\mathbf{\Sigma}_{T-t_r})] - \sigma^{-4}_{T-t}\E_{Q^\gets}[\text{Tr}(\mathbf{\Sigma}_{T-t})]\Big)\der t \\
    &\overset{(a)}{\le} \sum^{N-1}_{r=0} \eta_r \sigma^{-4}_{T-t_r}\Big(\E_{Q^\gets}[\text{Tr}(\mathbf{\Sigma}_{T-t_r})] - \E_{Q^\gets}[\text{Tr}(\mathbf{\Sigma}_{T-t_{r+1}})]\Big)
\end{align*}
Step (a) holds by the fact that $\sigma^{-4}_{T-t}$ is increasing with $t$ and $\E_{Q^\gets}[\text{Tr}(\mathbf{\Sigma}_{T-t})]$ is decreasing with $t$ ($\frac{\der}{\der t} \mathbb{E}_{\tilde{Q}}[\text{Tr}(\mathbf{\Sigma}_{T-t})] \le 0$ inferred from Lemma \ref{lemma:4.5}).

Then, consider the time interval $[0,T-1]$ and $\eta_r \sigma^{-4}_{T-t_r} \le \frac{25}{16} \eta_r \le \frac{25}{16} \eta$, thus
\begin{align*}
    &\sum^{M-1}_{k=0} \eta_r \sigma^{-4}_{T-t_r}\Big(\E_{Q^\gets}[\text{Tr}(\mathbf{\Sigma}_{T-t_r})] - \E_{Q^\gets}[\text{Tr}(\mathbf{\Sigma}_{T-t_{r+1}})]\Big) \\
    &\le  \frac{25}{16} \eta \sum^{M-1}_{k=0} \Big(\E_{Q^\gets}[\text{Tr}(\mathbf{\Sigma}_{T-t_r})] - \E_{Q^\gets}[\text{Tr}(\mathbf{\Sigma}_{T-t_{r+1}})]\Big)\\
    & \le \frac{25}{16} \eta \E_{Q^\gets}[\text{Tr}(\mathbf{\Sigma}_{T})]\\
    & \lesssim \eta \E_{\rvy \sim q_0}\left[\|\rvy\|^2\right].
\end{align*}
The last inequality holds by $\E_{Q^\gets}[\text{Tr}(\mathbf{\Sigma}_{T})] \le \E_{Q^\gets}[\E[\|\rvy_0\|^2|\rvy_t]] \le \E_{\rvy \sim q_0}\left[\|\rvy\|^2\right]$.

Similarly, consider the time interval $(T-1,T-\delta]$,  $\eta_r \sigma^{-4}_{T-t_r} \le \frac{25}{(T-t_r)^2} \eta_r \le \frac{25}{(T-t_r)} \eta$, then
\begin{align}
\label{eq:e2_2}
    &\sum^{N-1}_{k=M} \eta_r \sigma^{-4}_{T-t_r}\Big(\E_{Q^\gets}[\text{Tr}(\mathbf{\Sigma}_{T-t_r})] - \E_{Q^\gets}[\text{Tr}(\mathbf{\Sigma}_{T-t_{r+1}})]\Big) \\
    & \le  \frac{25}{16}\eta\sum^{N-1}_{k=M} \frac{1}{T-t_r}\Big(\E_{Q^\gets}[\text{Tr}(\mathbf{\Sigma}_{T-t_r})] - \E_{Q^\gets}[\text{Tr}(\mathbf{\Sigma}_{T-t_{r+1}})]\Big)\notag\\
    & \le \frac{25}{16} \eta \E_{Q^\gets}[\text{Tr}(\mathbf{\Sigma}_{1})] + \frac{25}{16} \eta \sum^{N-1}_{k=M+1} \frac{\eta_{r-1}}{(T-t_r)(T-t_{k-1})}\E_{Q^\gets}[\text{Tr}(\mathbf{\Sigma}_{T-t_r})]\notag\\
    &\le \frac{25}{16} \eta \E_{Q^\gets}[\text{Tr}(\mathbf{\Sigma}_{1})] + \frac{25}{16} \eta^2 \sum^{N-1}_{k=M+1} \frac{1}{T-t_{k-1}}\E_{Q^\gets}[\text{Tr}(\mathbf{\Sigma}_{T-t_r})]
\end{align}
For the interval  $t \in (T-1,T-\delta]$, we have
\begin{align}
\label{eq:e2_2_1}
    \E_{Q^\gets}[\text{Tr}(\mathbf{\Sigma}_{T-t})] &=\E_{Q^\gets}[\text{Tr}(\text{Cov}(e^{T-t}\rvy_{T-t}-\rvy_0|\rvy_{T-t}))] \notag\\
    & = e^{2(T-t)}\E_{Q^\gets}[\text{Tr}(\text{Cov}(\rvy_{T-t}-e^{t-T}\rvy_0|\rvy_{T-t}))]\notag\\
    & \le e^{2(T-t)} \E_{Q^\gets}[\E[\|\rvy_{T-t}-e^{t-T}\rvy_0\|^2|\rvy_{T-t}]]\notag\\
    & \le d (e^{2(T-t)} - 1)\notag\\
    & \le 2de^2(T-t)
\end{align}
Then, plugging Eq.~\ref{eq:e2_2_1} into Eq.~\ref{eq:e2_2}, we have
\begin{align*}
     &\sum^{N-1}_{k=M} \eta_r \sigma^{-4}_{T-t_r}\Big(\E_{Q^\gets}[\text{Tr}(\mathbf{\Sigma}_{T-t_r})] - \E_{Q^\gets}[\text{Tr}(\mathbf{\Sigma}_{T-t_{r+1}})]\Big) \\
     &\le \frac{25}{16}\eta \E_{p_{*,k+1|[1:k]}(\cdot|\vx_{[1:k]})}\left[\|\rvy\|^2\right]+ \frac{25de^2}{8} \eta^2 \sum^{N-1}_{k=M+1} \frac{T-t_r}{T-t_{k-1}}\\
     &\lesssim \eta \E_{p_{*,k+1|[1:k]}(\cdot|\vx_{[1:k]})}\left[\|\rvy\|^2\right]+ d N\eta^2
\end{align*}
Therefore, we obtain
\begin{align*}
    &\sum^{N-1}_{r=0} \int^{t_{r+1}}_{t_r} \int^t_{t_r} E^{(2)}_{t_r,r} \der r \der t \lesssim \eta \E_{p_{*,k+1|[1:k]}(\cdot|\vx_{[1:k]})}\left[\|\rvy\|^2\right]+ d N\eta^2
\end{align*}
\end{proof}

\section{Lower bound for vanilla diffusion models}
\begin{proof}[Proof of~\ref{lem:lower_bound}]
Denote $d_{\vx} = d_1+d_2+\dots+d_k$, and $d_{\rvy} = d_{k+1}$, and denote $d=d_{\vx}+d_{\rvy}$.
Choose $p_*(\rvy,\vx)$ to be the density function of $\mathcal{N}\left(0, \left(\begin{array}{cc}
   \mathrm{I}_{d_{\rvy}\times d_{\rvy}}  & \frac{\epsilon}{d M} \bar{\mathrm{I}}_{d_{\rvy}\times d_{\vx}} \\
   \frac{\epsilon}{d M} \bar{\mathrm{I}}_{d_{\vx} \times d_{\rvy}}  & 2 \left( \frac{\epsilon^2}{d^2 M^2} + \frac{\epsilon^3}{d^2 M^2} \right) \mathrm{I}_{d_{\vx}\times d_{\vx}}
\end{array}\right)\right)$, and $\hat{p}_*(\rvy,\vx)$ the density function of $\mathcal{N}\left(0, \left(\begin{array}{cc}
   \mathrm{I}_{d_{\rvy}\times d_{\rvy}}  & \frac{\epsilon}{d M} \bar{\mathrm{I}}_{d_{\rvy}\times d_{\vx}} \\
   \frac{\epsilon}{d M} \bar{\mathrm{I}}_{d_{\vx} \times d_{\rvy}}  & 2 \frac{\epsilon^2}{d^2 M^2} \mathrm{I}_{d_{\vx}\times d_{\vx}}
\end{array}\right)\right)$.

$\hat{p}_*(\rvy|\vx)$ corresponds to $\mathcal{N}\left( \frac{d M}{2 \epsilon} \bar{\mathrm{I}}_{d_{\rvy}\times d_{\vx}} \vx, \frac{1}{2}\mathrm{I}_{d_{\rvy}\times d_{\rvy}} \right)$, and ${p}_*(\rvy|\vx)$ corresponds to $\mathcal{N}\left( \frac{d M}{2 \epsilon} \frac{1}{1+\epsilon} \bar{\mathrm{I}}_{d_{\rvy}\times d_{\vx}} \vx, \frac{1}{2} \left( 1 + \frac{\epsilon}{1+\epsilon} \right) \mathrm{I}_{d_{\rvy}\times d_{\rvy}} \right)$.
Hence $\KL{{p}_*(\rvy|\vx)}{\hat{p}_*(\rvy|\vx)} > \left( \frac{M}{1/d+\epsilon/d} \right)^2 \left\| \bar{\mathrm{I}}_{d_{\rvy}\times d_{\vx}} \vx \right\|^2
= \left( \frac{M}{1/d+\epsilon/d} \right)^2 \left\| \vx_{(1:d_{k+1})} \right\|^2$.

On the other hand, 
\begin{align*}
&\left(
\begin{array}{cc}
   \mathrm{I}_{d_{\rvy}\times d_{\rvy}}  & \frac{\epsilon}{d M} \bar{\mathrm{I}}_{d_{\rvy}\times d_{\vx}} \\
   \frac{\epsilon}{d M} \bar{\mathrm{I}}_{d_{\vx} \times d_{\rvy}}  & 2 \frac{\epsilon^2}{d^2 M^2} \mathrm{I}_{d_{\vx}\times d_{\vx}}
\end{array}\right)^{-1} \\
&= \frac{d^2 M^2}{\epsilon^2}
\left(\begin{array}{cc}
   2 \frac{\epsilon^2}{d^2 M^2} \mathrm{I}_{d_{\rvy}\times d_{\rvy}}  & - \frac{\epsilon}{d M} \bar{\mathrm{I}}_{d_{\rvy}\times d_{\vx}} \\
   - \frac{\epsilon}{d M} \bar{\mathrm{I}}_{d_{\vx} \times d_{\rvy}}  & \frac{1}{2} \mathrm{I}_{d_{\vx}\times d_{\vx}} + \frac{1}{2} \bar{\mathrm{I}}_{d_{\vx}\times d_{\rvy}} \bar{\mathrm{I}}_{d_{\rvy}\times d_{\vx}}
\end{array}\right)
\end{align*}
Hence $\KL{p_*(\rvy,\vx)}{\hat{p}_*(\rvy,\vx)}
\leq \epsilon(d_x + d_y) = \epsilon d$.

Choosing $\varepsilon = \epsilon d$, and noting that $d\geq 2$ and that $\varepsilon \leq 1/2$ finishes the proof.
\end{proof}

%% file: 0_contents/0Xappendix/0XA_auxiliary_lemmas.tex
\section{Auxiliary Lemmas}
\label{sec:app_aux_lem}

\begin{lemma}[Variant of Lemma 10 in~\cite{cheng2018convergence}]
    \label{lem:strongly_lsi}
    Suppose $-\log p_*$ is $m$-strongly convex function, for any distribution with density function $p$, we have
    \begin{equation*}
        \KL{p}{p_*}\le \frac{1}{2m}\int p(\vx)\left\|\grad\log \frac{p(\vx)}{p_*(\vx)}\right\|^2\der\vx.
    \end{equation*}
    By choosing $p(\vx)=g^2(\vx)p_*(\vx)/\mathbb{E}_{p_*}\left[g^2(\rvx)\right]$ for the test function $g\colon \R^d\rightarrow \R$ and  $\mathbb{E}_{p_*}\left[g^2(\rvx)\right]<\infty$, we have
    \begin{equation*}
        \mathbb{E}_{p_*}\left[g^2\log g^2\right] - \mathbb{E}_{p_*}\left[g^2\right]\log \mathbb{E}_{p_*}\left[g^2\right]\le \frac{2}{m} \mathbb{E}_{p_*}\left[\left\|\grad g\right\|^2\right],
    \end{equation*}
    which implies $p_*$ satisfies $m$-log-Sobolev inequality.
\end{lemma}

\begin{lemma}[Lemma 11 in~\cite{vempala2019rapid}]
    \label{lem:exp_score_bound}
    Suppose a density function $p\propto\exp(-f)$ and $\|\grad^2 f\|\le L$. 
    Then, it has
    \begin{equation*}
        \E_{\rvx\sim p}\left[\left\|\grad f(\rvx)\right\|^2\right]\le Ld
    \end{equation*}
    where $d$ is the dimension number of $\vx$.
\end{lemma}

\begin{lemma}[Lemma B.1 in~\cite{huang2024reverse}]
    \label{lem:tv_chain_rule}
    Consider four random variables, $\rvx, \rvz, \tilde{\rvx}, \tilde{\rvz}$, whose underlying distributions are denoted as $p_x, p_z, q_x, q_z$.
    Suppose $p_{x,z}$ and $q_{x,z}$ denotes the densities of joint distributions of $(\rvx,\rvz)$ and $(\tilde{\rvx},\tilde{\rvz})$, which we write in terms of the conditionals and marginals as
    \begin{equation*}
        \begin{aligned}
        &p_{x,z}(\vx,\vz) = p_{x|z}(\vx|\vz)\cdot p_z(\vz)=p_{z|x}(\vz|\vx)\cdot p_{x}(\vx)\\
        &q_{x,z}(\vx,\vz)=q_{x|z}(\vx|\vz)\cdot q_z(\vz) = q_{z|x}(\vz|\vx)\cdot q_x(\vx).
        \end{aligned}
    \end{equation*}
    then we have
    \begin{equation*}
        \begin{aligned}
            \TVD{p_{x,z}}{q_{x,z}} \le  \min & \left\{ \TVD{p_z}{q_z} + \E_{\rvz\sim p_z}\left[\TVD{p_{x|z}(\cdot|\rvz)}{q_{x|z}(\cdot|\rvz)}\right],\right.\\
            &\quad  \left.\TVD{p_x}{q_x}+\E_{\rvx \sim p_x}\left[\TVD{p_{z|x}(\cdot|\rvx)}{q_{z|x}(\cdot|\rvx)}\right]\right\}.
        \end{aligned}
    \end{equation*}
    Besides, we have
    \begin{equation*}
        \TVD{p_x}{q_x}\le \TVD{p_{x,z}}{q_{x,z}}.
    \end{equation*}
\end{lemma}

\begin{lemma}[Lemma B.4 in~\cite{huang2024reverse}]
    \label{lem:kl_chain_rule}
    Consider four random variables, $\rvx, \rvz, \tilde{\rvx}, \tilde{\rvz}$, whose underlying distributions are denoted as $p_x, p_z, q_x, q_z$.
    Suppose $p_{x,z}$ and $q_{x,z}$ denotes the densities of joint distributions of $(\rvx,\rvz)$ and $(\tilde{\rvx},\tilde{\rvz})$, which we write in terms of the conditionals and marginals as
    \begin{equation*}
        \begin{aligned}
        &p_{x,z}(\vx,\vz) = p_{x|z}(\vx|\vz)\cdot p_z(\vz)=p_{z|x}(\vz|\vx)\cdot p_{x}(\vx)\\
        &q_{x,z}(\vx,\vz)=q_{x|z}(\vx|\vz)\cdot q_z(\vz) = q_{z|x}(\vz|\vx)\cdot q_x(\vx).
        \end{aligned}
    \end{equation*}
    then we have
    \begin{equation*}
        \begin{aligned}
            \KL{p_{x,z}}{q_{x,z}} = & \KL{p_z}{q_z} + \E_{\rvz\sim p_z}\left[\KL{p_{x|z}(\cdot|\rvz)}{q_{x|z}(\cdot|\rvz)}\right]\\
            = & \KL{p_x}{q_x}+\E_{\rvx \sim p_x}\left[\KL{p_{z|x}(\cdot|\rvx)}{q_{z|x}(\cdot|\rvx)}\right]
        \end{aligned}
    \end{equation*}
    where the latter equation implies
    \begin{equation*}
        \KL{p_x}{q_x}\le \KL{p_{x,z}}{q_{x,z}}.
    \end{equation*}
\end{lemma}

\begin{lemma}[Lemma C.1 in~\cite{chen2023improved}]
\label{lemma:c.1}
Consider the following two It\^o processes
\begin{equation}
\begin{aligned}
    d X_t &= F_1(X_t, t) \, \der t + g(t) \, dW_t,  & X_0 &= a, \\
    d Y_t &= F_2(Y_t, t) \, \der t + g(t) \, dW_t, & Y_0 &= a,
\end{aligned}
\end{equation}
where $F_1, F_2, g$ are continuous functions and may depend on $a$. We assume the uniqueness and regularity condition:
\begin{itemize}
    \item The two SDEs have unique solutions.
    \item $X_t, Y_t$ admit densities $p_t, q_t \in C^2(\mathbb{R}^d)$ for $t>0$.
\end{itemize}

Define the relative Fisher information between $p_t$ and $q_t$ by
\begin{equation}
    J(p_t || q_t) = \int p_t(x) \left\| \nabla \log \frac{p_t(x)}{q_t(x)} \right\|^2 \, dx.
\end{equation}

Then for any $t>0$, the evolution of $\KL{p_t}{q_t}$ is given by
\begin{equation}
    \frac{\partial}{\partial t} \KL{p_t}{q_t} = - \frac{g(t)^2}{2} J(p_t || q_t) + \mathbb{E} \left[ \left\langle F_1(X_t, t) - F_2(X_t, t), \nabla \log \frac{p_t(X_t)}{q_t(X_t)}  \right\rangle \right].
\end{equation}
\end{lemma}

\begin{lemma}[Lemma C.2 in~\cite{chen2023improved}]
\label{lemma:c.2}
For $0 \leq k \leq N - 1$, consider the reverse SDE starting from $\tilde{x}_{t'_k} = a$
\begin{equation}
\label{eq:sde_1}
    d \tilde{x}_t = \left[ \frac{1}{2} \tilde{x}_t + \nabla \log \tilde{p}_t(\tilde{x}_t)  \right] \der t + dW_t, \qquad \tilde{x}_{t'_k} = a
\end{equation}
\text{and its discrete approximation:}
\begin{equation}
\label{eq:sde_2}
    d \hat{y}_t = \left[ \frac{1}{2} \hat{y}_t + s(a, t-t'_k) \right] \der t + dW_t, \qquad \hat{y}_{t'_k} = a
\end{equation}
for time $t \in (t'_k, t'_{k+1}]$. Let $\tilde{p}_{t|t'_k}$ be the density of $\tilde{x}_t$ given $\tilde{x}_{t'_k}$ and $\hat{q}_{t|t'_k}$ be density of $\hat{y}_t$ given $\hat{y}_{t'_k}$. Then we have
\begin{enumerate}
    \item For any $a \in \mathbb{R}^d$, the two processes satisfy the uniqueness and regularity condition stated in Lemma \ref{lemma:c.1}, that is, \ref{eq:sde_1} and \ref{eq:sde_2} have unique solution and $\tilde{p}_{t|t'_k}(\cdot|a), \hat{q}_{t|t'_k}(\cdot|a) \in C^2(\mathbb{R}^d)$ for $t>t'_k$.
    \item For a.e. $a \in \mathbb{R}^d$ (with respect to the Lebesgue measure), we have
    \begin{equation}
    \lim_{t \to {t'_k}^+}\KL{\tilde{p}_{t|t'_k}(\cdot|a)}{\hat{q}_{t|t'_k}(\cdot|a)} = 0.
    \end{equation}
\end{enumerate}
\end{lemma}
\begin{lemma}[Lemma C.9 in~\cite{chen2023improved}]
\label{lemma:c.9} 
Suppose that Assumption 3 holds. If $\sigma_t^2 \leq \frac{\alpha_t}{2L}$, we have $\nabla \log p_t$ is $2L \alpha_t^{-1}$-Lipschitz on $\mathbb{R}^d$.
\end{lemma}
\begin{lemma}[Lemma C.6 in~\cite{chen2023improved}]
\label{lemma:c.6}
    For any $0 \leq t \leq s \leq T$, the forward process \ref{sde:ideal_condi_forward} satisfies,
 \begin{align}
     \mathbb{E} \left\| \nabla \log q_t(x_t) - \nabla \log q_s(x_s) \right\|^2 &\leq 4 \mathbb{E} \left\| \nabla \log q_t(x_t) - \nabla \log q_t(\alpha_{t,s}^{-1} x_s) \right\|^2 + 2 \mathbb{E} \left\| \nabla \log q_t(x_t) \right\|^2 \left( 1 - \alpha_{t,s}^{-1} \right)^2.
 \end{align}
\end{lemma}
\begin{lemma}[Lemma 1 in~\cite{benton2024nearly}]
\label{lemma:4.5}
For all $t >0$, $\frac{\sigma^3_t}{2\Dot{\sigma_t}}\frac{\der}{\der t} \E[\mathbf{\Sigma}_t] = \E[\mathbf{\Sigma}^2_t]$.

\end{lemma}

%% file: 0_contents/0Xappendix/0X4_training_details.tex
\section{More Details on Experiments}
\label{sec:app_more_training_details}

\subsection{Synthetic Tasks and Evaluation}
\paragraph{Synthetic Data and Training Details.}For each task, we synthesize 10000 samples with image size $32 \times 32$ and use them to train both AR Diffusion and DDPM. For AR Diffusion, we adopt a simplified version without VAE, training directly in pixel space to preserve inter-feature relationships. To ensure a fair comparison, AR Diffusion and DDPM are configured with consistent total model parameters and training epochs. We use U-Net as the denoising neural network for DDPM and MLP for AR Diffusion, with a learning rate of $3e-4$ and approximately $23M$ model parameters. Training is conducted on a single NVIDIA A800 GPU for $400$ epochs (Task 1) and $1000$ epochs (Task 2), respectively.
\begin{figure}[]
\centering
    \hfill
    \subfigure[Task 1]{\label{fig:task1}\includegraphics[width=0.4\textwidth]{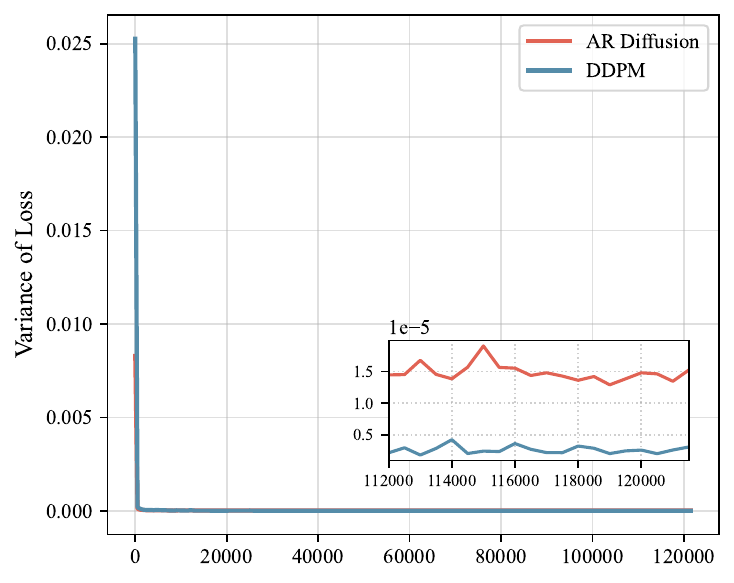}}
    \hfill
    \subfigure[Task 2]{\label{fig:task2}\includegraphics[width=0.4\textwidth]{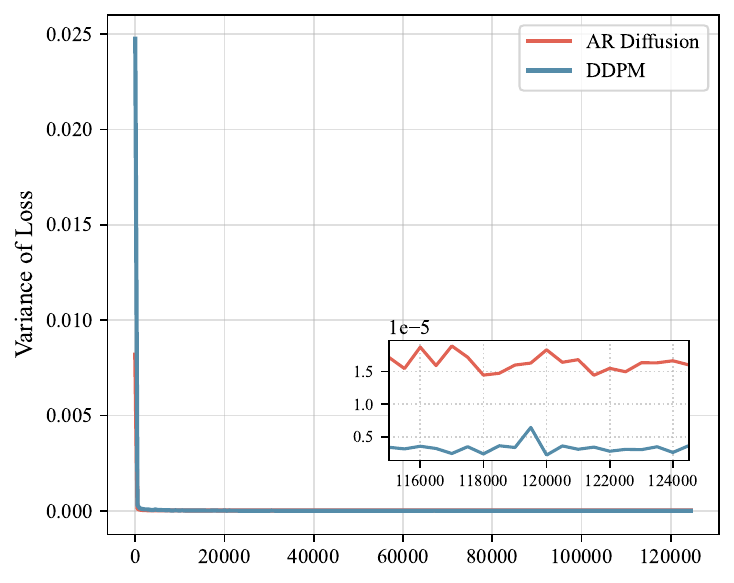}}
    \hfill
\vspace{-0.1in}
\caption{The training loss variance with a sliding window of $500$. We observed that at the end of training, AR Diffusion and DDPM exhibit training stability with loss variation less than $1e-5$.}
\vspace{-0.15in}
\label{fig:variance}
\end{figure}
\paragraph{Inference Phase.}To compare the performance of AR Diffusion and DDPM during the inference stage, we directly quantify whether the images generated by the two models satisfy the predefined feature dependencies in \Cref{fig:task_1_dis} and \Cref{fig:task_2_dis}.  For each task, we synthesize 3000 images, and we can directly obtain generated image masks using predefined colors (e.g., yellow tones or gray tones for the sun and shadow in Task 1). Using these masks, we can directly extract the variables of interest, such as the geometric features we focus on, including the sun altitude \( l_1 \) and shadow length \( l_2 \).  For Task 1, the target ratio is \( R = 1 \), which means the extracted geometric features need to satisfy \( \frac{l_2 h_1}{l_1 h_2} = 1 \). In contrast, for Task 2, the target ratio is \( R = \frac{l_2}{l_1} = 1.5 \).  Additionally, to ensure the robustness of the experimental results shown in \Cref{fig:task1_tr}, \Cref{fig:task1_ar}, \Cref{fig:task1_ddpm} and \Cref{fig:task2_tr}, \Cref{fig:task2_ar}, \Cref{fig:task2_ddpm} , we only consider sample points with a ratio within the 5\% to 95\% range for visualization. More details on data synthesis and feature extraction can be found in previous work \cite{han2025can}.

\paragraph{Training Phase.}\Cref{fig:task1_loss} and \Cref{fig:task2_loss} further quantify the training performance of AR Diffusion and DDPM using training loss as a metric. It is worth noting that in the theoretical analysis of this paper, we use the score matching loss, whereas in actual model training, we optimize MSE between the denoising neural network and Gaussian noise.  These two objectives are equivalent in optimizing the neural network parameters but differ by a constant factor in value. Specifically, following \citet{chen2023improved}, their relationship is as follows:
\begin{align}
&\underbrace{\mathbb{E}_{p_t} \| s_\theta (x,t) - \nabla \log p_t (x) \|^2}_{\text{theorical target}} \\
&= \mathbb{E}_{p_t} \| s_\theta (x,t) \|^2 + \mathbb{E}_{p_t} \| \nabla \log p_t (x) \|^2 - 2 \mathbb{E}_{p_t} \langle s_\theta (x,t), \nabla \log p_t (x) \rangle \notag\\
&= \mathbb{E}_{p_t} \| s_\theta (x,t) \|^2 + \mathbb{E}_{p_t} \| \nabla \log p_t (x) \|^2 + 2 \mathbb{E}_{p_t} \nabla \cdot s_\theta (x,t) \notag\\
&= \mathbb{E}_{p_t} \| s_\theta (x,t) \|^2 + \mathbb{E}_{p_t} \| \nabla \log p_t (x) \|^2 + 2 \mathbb{E}_{p_0 (x_0)} \mathbb{E}_{p_{t|0} (x_t | x_0)} \nabla \cdot s_\theta (x,t) \notag\\
&= \mathbb{E}_{p_t} \| s_\theta (x,t) \|^2 + \mathbb{E}_{p_t} \| \nabla \log p_t (x) \|^2 - 2 \mathbb{E}_{p_0 (x_0)} \mathbb{E}_{p_{t|0} (x_t | x_0)} \langle \nabla \log p_{t|0} (x_t | x_0), s_\theta (x,t) \rangle \notag\\
&= \mathbb{E}_{p_t} \| s_\theta (x,t) \|^2 + \mathbb{E}_{p_t} \| \nabla \log p_t (x) \|^2 - 2 \mathbb{E}_{p_0 (x_0)} \mathbb{E}_{p_{t|0} (x_t | x_0)} \left\langle \frac{x_t - \alpha_t x_0}{\sigma_t^2}, s_\theta (x,t) \right\rangle \notag\\
&= \mathbb{E} \left\| s_\theta (x,t) - \frac{x_t - \alpha_t x_0}{\sigma_t^2} \right\|^2 + \mathbb{E}_{p_t} \| \nabla \log p_t (x) \|^2 - \frac{d}{\sigma_t^2} \notag\\
&= \underbrace{\mathbb{E} \left\| s_\theta (x,t) - \boldsymbol{\epsilon}\right\|^2}_{\text{training 
 target}} + C \geq {\mathbb{E} \left\| s_\theta (x,t) - \boldsymbol{\epsilon}\right\|^2} - \|C\|,
 \label{eq:loss constant}
\end{align}
where $\boldsymbol{\epsilon}$ is the Gaussian noise and the constant \( C \) in \eqref{eq:loss constant} is independent of the model parameters but is closely related to the oracle score. we then representate \eqref{eq:loss constant} as
\begin{align}
\label{eq:score_loss_ar}
    \mathcal{L}_{\mathrm{AR}} -  \|C_{\mathrm{AR}}\| \leq \boldsymbol{\epsilon}^2_{\mathrm{score}-\mathrm{AR}}
\end{align}
\begin{align}
\label{eq:score_loss_ddpm}
    \mathcal{L}_{\mathrm{DDPM}}-\| C_{\mathrm{DDPM}}\| \leq \boldsymbol{\epsilon}^2_{\mathrm{score}-\mathrm{DDPM}} 
\end{align}
\eqref{eq:score_loss_ar} and \eqref{eq:score_loss_ddpm} provide a theoretical lower bound on the comparison of score estimation errors between AR Diffusion and DDPM. Under the best-case scenario—where AR Diffusion and DDPM achieve nearly optimal training—we can isolate the fundamental differences between the two models and understand their theoretical performance gap in the most favorable conditions. Naturally, to accurately link the empirical training loss \( \mathcal{L} \) with the theoretical score estimation error \( \boldsymbol{\epsilon}^2 \), we need to estimate the corresponding constant \( C \) for both AR Diffusion and DDPM separately.

Considering the total number of training steps as \( S \), our core idea is to identify a sufficiently large training phase where the training loss stabilizes, i.e., to find the maximum training step with a window of length \( e \) such that the variance of the steps satisfies  
\[
\max_h \operatorname{Var}(S_h, S_{h+e}) \leq \eta
\]  
where \( \eta \) is an extremely small value close to \( 0 \). This indicates that the training loss in this phase is primarily dominated by a constant \( C \), allowing us to approximate \( C \) using the training loss in this phase.  Next, we consider the latter half of the entire training process, specifically the training interval \( \left[ \frac{S}{2}, S \right] \). Compared to the first half \( \left[ 0, \frac{S}{2} \right) \), this phase is more stable and avoids extreme values, ensuring a more reasonable comparison.  Finally, we apply \eqref{eq:score_loss_ar} and \eqref{eq:score_loss_ddpm} to estimate the score estimation error for AR Diffusion and DDPM, respectively.

\cref{fig:variance} illustrates the maximum stable training step \( h \) for Task 1 and Task 2, respectively, with a window length of $e=500$ steps, along with the corresponding variance in this phase. We observe that towards the end of training, the changes in training loss are extremely small, with a variance on the order of \( 1e-5 \). Therefore, consider the small $\eta = 1e-4$, we can approximate the constant by directly using the training loss at the final step and subsequently estimate the score estimation error.  After obtaining the score estimation errors for AR Diffusion and DDPM separately, we compute their difference $\boldsymbol{\epsilon}^2_{\mathrm{score}-\mathrm{DDPM}} - \boldsymbol{\epsilon}^2_{\mathrm{score}-\mathrm{AR}} $ and visualize the different in \Cref{fig:task1_loss} and \Cref{fig:task2_loss}.

\begin{figure*}[]
\centering
    \hfill
    \subfigure[Training Data]{\label{fig:task1_tr_app}\includegraphics[width=0.24\textwidth]{figures/task1_tr.pdf}}
    \hfill
    \subfigure[Inference: AR Diffusion]{\label{fig:task1_ar_unet_app}\includegraphics[width=0.24\textwidth]{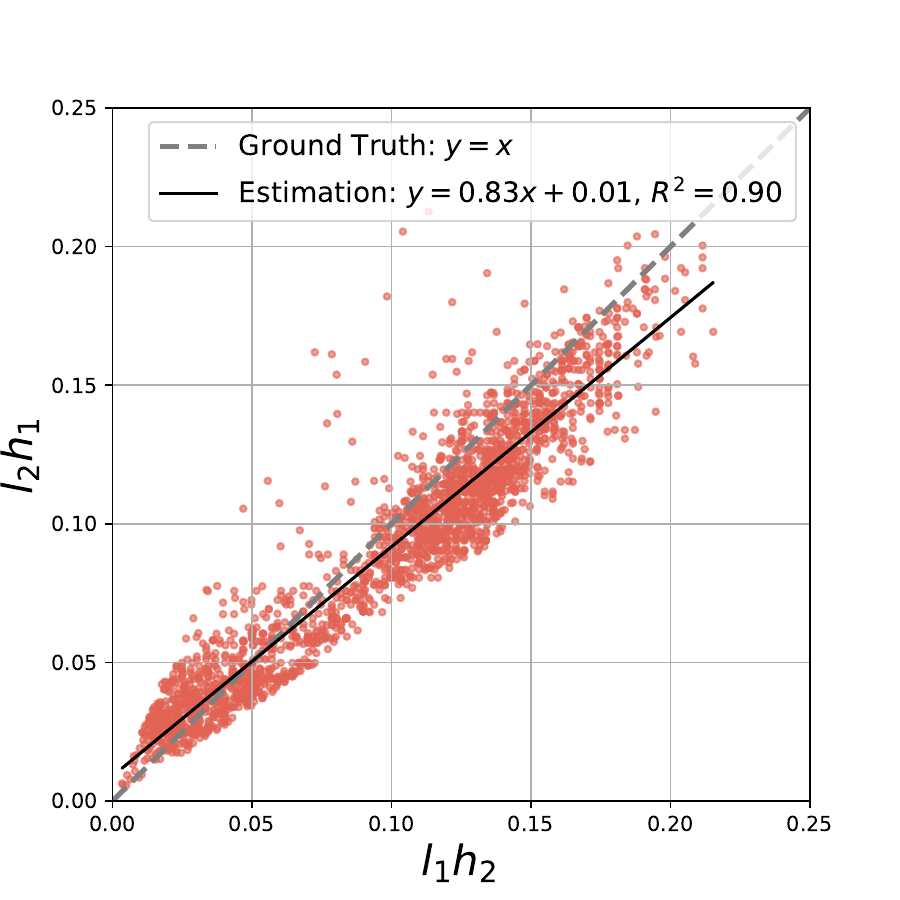}}
    \hfill
    \subfigure[Inference: DDPM]{\label{fig:task1_ddpm_app}\includegraphics[width=0.24\textwidth]{figures/task1_ddpm.pdf}}
    \hfill
    \subfigure[Training: Diffusion Loss]{\label{fig:task1_loss_unet}\includegraphics[width=0.22\textwidth]{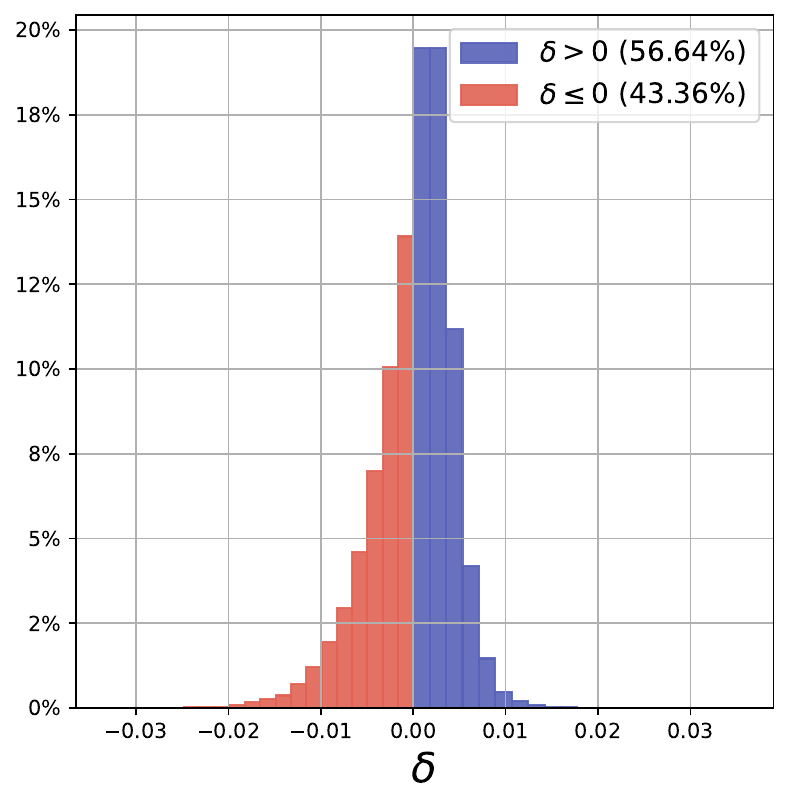}}
    \hfill
\vspace{-0.15in}
\caption{\textbf{Comparison of AR Diffusion with U-Net and DDPM with U-Net Performance on Sun-Shadow Setting.} \Cref{fig:task1_ar_unet_app} and \Cref{fig:task1_ddpm_app} illustrate the performance of AR Diffusion with U-Net and DDPM with U-Net during the inference phase, showing that AR Diffusion better captures inter-feature dependencies with a higher \( R^2 \). \Cref{fig:task1_loss_unet} presents the difference in training loss between DDPM and AR Diffusion, denoted as \( \delta \). For most training steps, AR Diffusion's training loss is lower than that of DDPM, with \( \delta > 0 \).}
\vspace{-0.15in}
\label{fig:task_1_dis_unet}
\end{figure*}

% The conclusions based on the U-Net experiment align with those in \Cref{fig:task_1_dis}

\begin{figure*}[]
\centering
    \hfill
    \subfigure[Training Data]{\label{fig:task1_tr_app}\includegraphics[width=0.24\textwidth]{figures/task1_tr.pdf}}
    \hfill
    \subfigure[Inference: AR Diffusion]{\label{fig:task1_ar_app}\includegraphics[width=0.24\textwidth]{figures/task1_ar.pdf}}
    \hfill
    \subfigure[Inference: DDPM]{\label{fig:task1_ddpm_mlp}\includegraphics[width=0.24\textwidth]{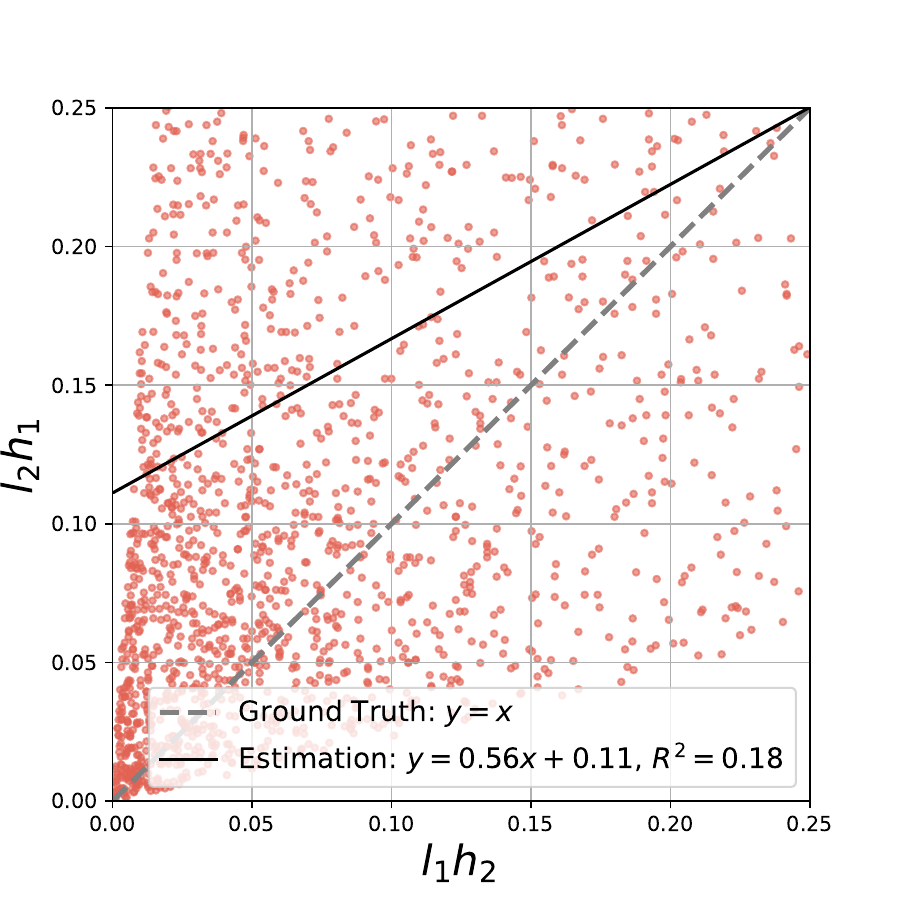}}
    \hfill
    \subfigure[Training: Diffusion Loss]{\label{fig:task1_loss_mlp}\includegraphics[width=0.22\textwidth]{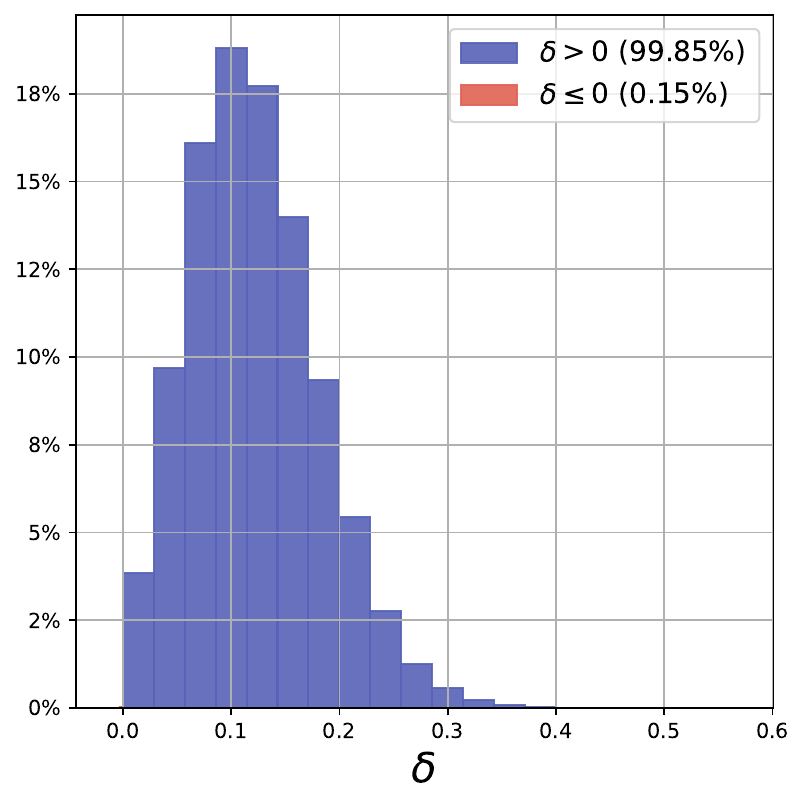}}
    \hfill
\vspace{-0.15in}
\caption{\textbf{Comparison of AR Diffusion with MLP and DDPM with MLP Performance on Sun-Shadow Setting.} \Cref{fig:task1_ar_app} and \Cref{fig:task1_ddpm_mlp} illustrate the performance of AR Diffusion with MLP and DDPM with MLP during the inference phase, showing that AR Diffusion better captures inter-feature dependencies with a higher \( R^2 \), while DDPM with MLP exhibits a significant quality degradation. \Cref{fig:task1_loss_mlp} presents the significant difference in training loss between DDPM and AR Diffusion, denoted as \( \delta \). For nearly all training steps, AR Diffusion's training loss is lower than that of DDPM, with \( \delta > 0 \).}
\vspace{-0.15in}
\label{fig:task_1_dis_mlp}
\end{figure*}

\subsection{More Architectures and More Data}
\label{app:More Architectures and More Data}
% Complex setting (dependency)

% \begin{itemize}
%     \item Updated Task 1 (done)
% \end{itemize}

\paragraph{More Architectures.}In the main text, we consider the AR with MLP backbone and DDPM with U-Net backbone, which are classic architectures for AR generation and DDPM generation. In this section, while keeping other factors such as the number of parameters, training epochs, and learning rate consistent, we extend our experiments to AR with U-Net backbone and DDPM with MLP backbone, which allows for a comparison of AR and DDPM under the same backbone. \Cref{fig:task_1_dis_unet} shows the experimental results for AR with U-Net and DDPM with U-Net on Task 1. We observe that, consistent with the results in \Cref{fig:task_1_dis}, AR outperforms DDPM in both inference and training phases on Task 1. \Cref{fig:task_1_dis_mlp} compares the performance of AR with MLP and DDPM with MLP. We observe that simplifying the backbone leads to a significant performance drop in DDPM, but the results remain consistent with those in \Cref{fig:task_1_dis}, where AR generates samples that better satisfy dependencies, and the training loss is notably lower than that of DDPM with MLP.

\begin{figure*}[]
\centering
    \hfill
    \subfigure[Training Data]{\label{fig:eo_tr}\includegraphics[width=0.23\textwidth]{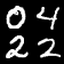}}
    \hfill
    \subfigure[Inference Phase]{\label{fig:eo_ar_ddpm_infer}\includegraphics[width=0.26\textwidth]{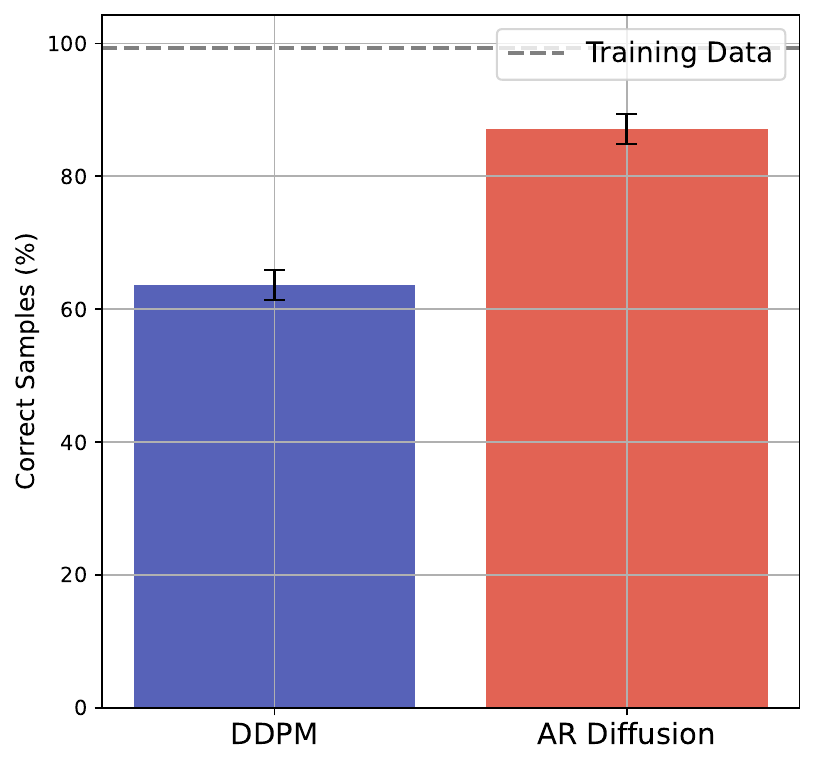}}
    \hfill
    \subfigure[Training Phase]{\label{fig:eo_ar_ddpm_loss}\includegraphics[width=0.24\textwidth]{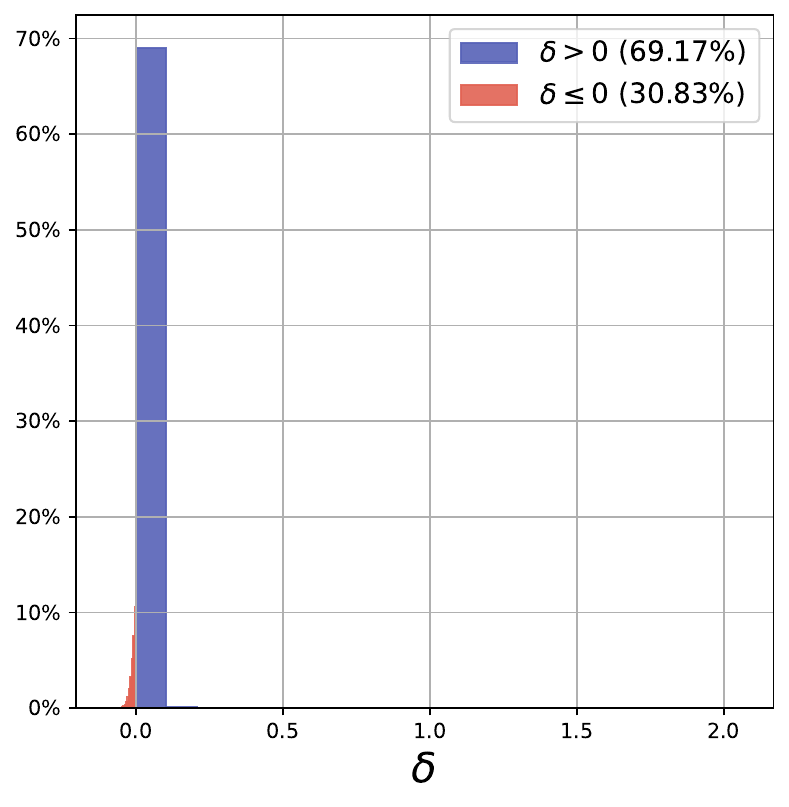}}
    \hfill
\vspace{-0.15in}
\caption{\textbf{Comparison of AR Diffusion with U-Net and DDPM with U-Net Performance on Odd and Even Setting.} \Cref{fig:eo_tr} shows the odd-even setting based on MNIST, where each training sample contains four sub-images with consistent parity. \Cref{fig:eo_ar_ddpm_infer} calculates the generated samples that satisfy the predefined dependencies, and AR Diffusion significantly outperforms DDPM. \Cref{fig:eo_ar_ddpm_loss} shows the training loss of AR Diffusion and DDPM, where DDPM's training loss is higher than AR Diffusion's for nearly 70\% of the training steps.}
\vspace{-0.15in}
\label{fig:eo_ar_ddpm_unet}
\end{figure*}

\begin{figure*}[]
\centering
    \hfill
    \subfigure[Training Data]{\label{fig:seq_tr}\includegraphics[width=0.23\textwidth]{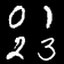}}
    \hfill
    \subfigure[Inference Phase]{\label{fig:seq_ar_ddpm_infer}\includegraphics[width=0.26\textwidth]{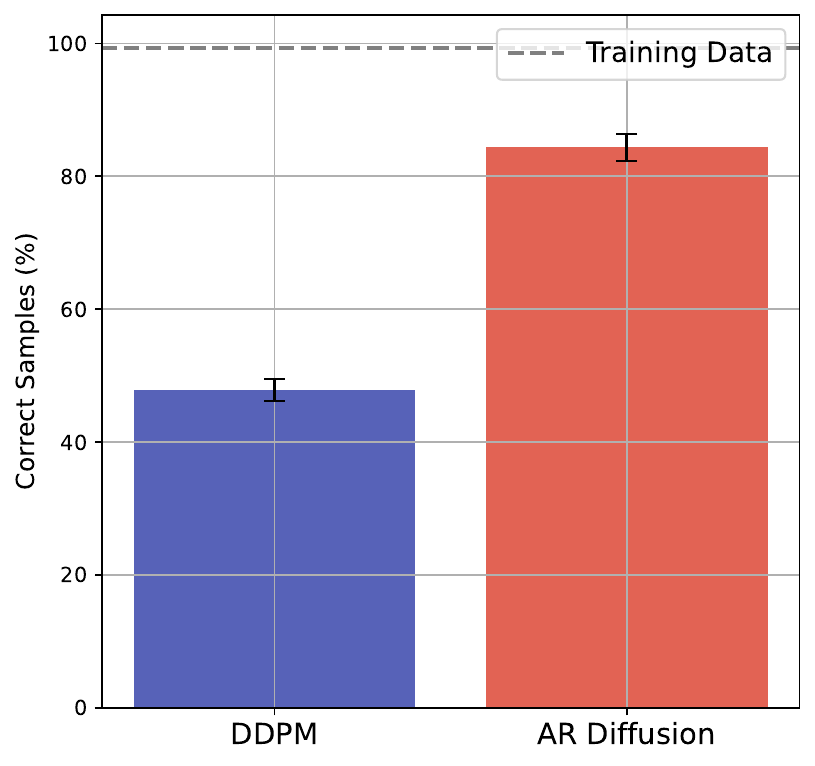}}
    \hfill
    \subfigure[Training Phase]{\label{fig:seq_ar_ddpm_loss}\includegraphics[width=0.24\textwidth]{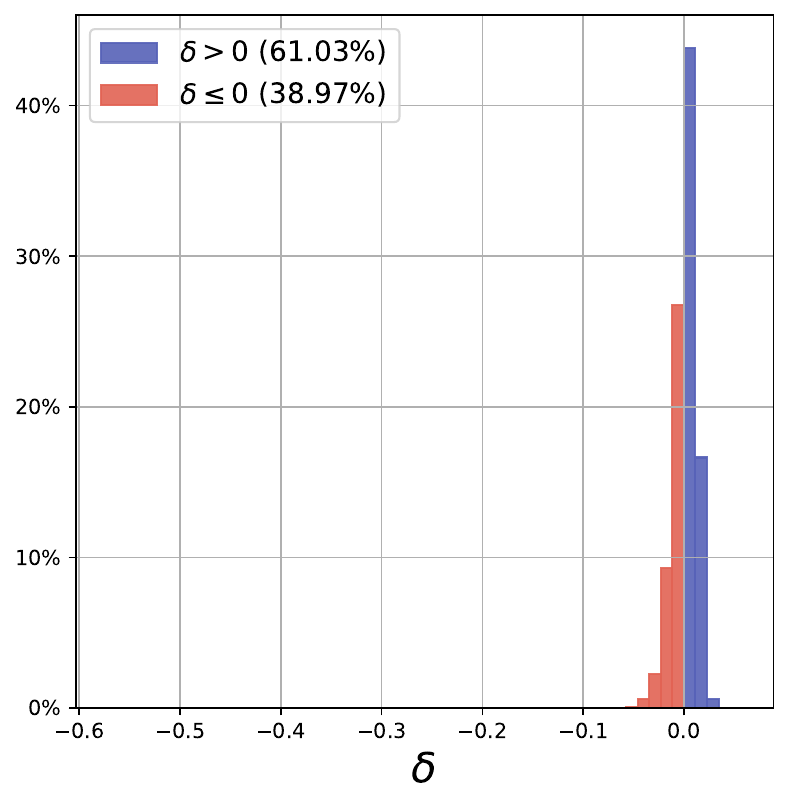}}
    \hfill
\vspace{-0.15in}
\caption{\textbf{Comparison of AR Diffusion with U-Net and DDPM with U-Net Performance on Arithmetic Sequence Setting.}\Cref{fig:seq_tr} shows the arithmetic sequence setting based on MNIST, where each training sample contains four sub-images that satisfy the arithmetic sequence. \Cref{fig:seq_ar_ddpm_infer} shows that over 80\% of the samples generated by AR Diffusion satisfy the predefined dependencies, while DDPM only has less than 50\% of its generated samples meeting the dependencies. \Cref{fig:seq_ar_ddpm_loss} shows the training loss of AR Diffusion and DDPM, where for 60\% of the training steps, DDPM's training loss is higher than that of AR Diffusion.}
\vspace{-0.15in}
\label{fig:seq_ar_ddpm_unet}
\end{figure*}

\paragraph{Real-world Data.}In this section, we conduct experiments on the real MNIST data. Specifically, we construct training data by concatenating MNIST digits that exhibit inter-feature dependencies. We consider two settings: (1) Odd and Even: The concatenated MNIST digits must all be either odd or even, for example, $0, 4, 2, 2$, as illustrated in \Cref{fig:eo_tr}; (2) Arithmetic Sequence: The concatenated 4 digits, from top to bottom and left to right, must form an arithmetic sequence, for example, $0, 1, 2, 3$, as illustrated in \Cref{fig:seq_tr}. For each setting, we synthesize $3000$ training samples and train AR with U-Net and DDPM with U-Net for $2000$ epochs. 

In the inference phase, we first pretrain a CNN on the MNIST dataset, achieving a test accuracy of $99.23\%$. For the generated data, we first divide the image into four sub-images along the vertical and horizontal midlines. Then, we use the pretrained CNN to predict the labels of the sub-images. We evaluate whether the sub-image labels satisfy the predefined dependencies, such as whether they are all odd/even or form an arithmetic sequence. If the predefined dependencies are satisfied, the sample is labeled as a correct sample. For both DDPM and AR Diffusion, we generate 300 samples at a time, repeated three times. We calculate the mean and error bars of the proportion of correct samples as the metric. \Cref{fig:eo_ar_ddpm_unet} and \Cref{fig:seq_ar_ddpm_unet} show the same conclusions as  \Cref{fig:task_1_dis} under both settings, namely that AR can better capture inter-feature dependencies in both the inference and training phases.

% \begin{itemize}
%     \item Mnist-Mnist-Seq
% \begin{itemize}
%     \item AR with MLP (done)
%         \item DDPM with U-Net (done)
% \end{itemize}
%     \item Mnist-Mnist-EO
%     \begin{itemize}
%         \item AR with MLP (done)
%         \item DDPM with U-Net (done)
%     \end{itemize}
% \end{itemize}

\subsection{Ablation Study: Parallel Order}
\label{sec:Ablation: Parallel order}
\begin{wrapfigure}{r}{0.32\textwidth}
\vspace{-.5in}
\begin{center}
    \includegraphics[width=0.36\textwidth]{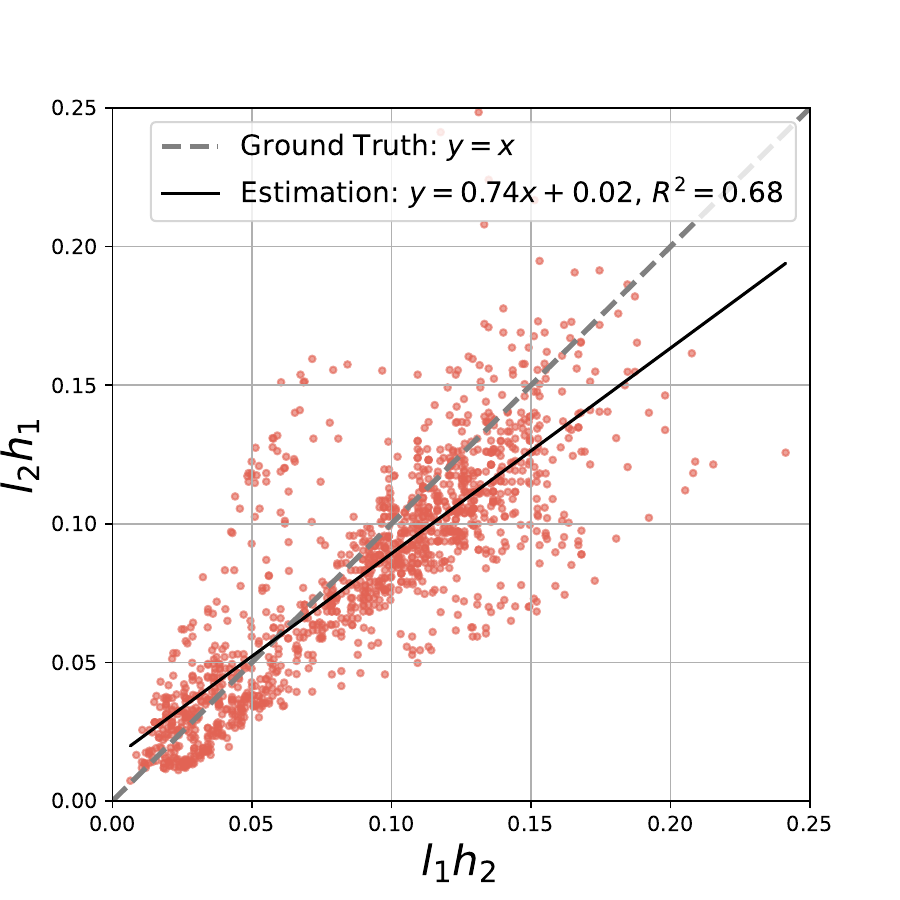}
\end{center}
\vspace{-0.2in}
\caption{AR Diffusion with Parallel Order.} 
% \vspace{-0.05in}
\label{fig:task1_ar_parallel} 
\end{wrapfigure}
In this section, we conduct an ablation study on Task 1 by modifying only the learning order of AR Diffusion, changing it from the raster scan order to the parallel order illustrated in \cref{fig:task1_ar_parallel}.  
It is important to emphasize that in the synthesized Task 1, the raster scan order aligns with the dependency order between features—i.e., the model always learns the sun first, followed by the shadow. In contrast, the parallel order completely decouples the sun and shadow, disrupting their relationship. \Cref{fig:task1_ar_parallel} shows the inference performance of AR Diffusion trained with the parallel order. We observe that, compared to AR Diffusion with the raster scan order shown in \Cref{fig:task1_ar}, simply changing the order while keeping all other factors the same leads to a significant drop in generated image quality, with \( R^2 \) decreasing from $0.92$ to $0.68$.  This ablation study highlights that AR Diffusion with an appropriate order can facilitate the learning of feature dependencies.